\documentclass[english]{article}

\usepackage{algorithmic}
\usepackage{algorithm}
\usepackage{amsfonts, amsmath, amssymb, amsthm}
\usepackage{array}
\usepackage{arydshln}
\usepackage{bm}
\usepackage{cite}
\usepackage{color}
\usepackage{comment}
\usepackage{dsfont}
\usepackage{enumitem}
\usepackage{float}
\usepackage[T1]{fontenc}
\usepackage{geometry}
\usepackage{graphicx,multirow}
\usepackage{grffile}
\usepackage{hyperref}\hypersetup{colorlinks,linkcolor=red,anchorcolor=blue,citecolor=blue}
\usepackage[latin9]{inputenc}
\usepackage{letltxmacro}
\usepackage{mathrsfs}
\usepackage{mathtools,natbib}
\usepackage{multirow,xspace,booktabs}
\usepackage{nicefrac}
\usepackage{tablefootnote}
\usepackage{verbatim}
\usepackage{subcaption}

\newcommand{\GD}{{\sf GD}\xspace}
\newcommand{\ScaledGD}{{\sf ScaledGD}\xspace}
\newcommand{\myalg}{{\sf{ScaledGD($\lambda$)}}\xspace}

\newcommand{\pnorm}[2]{\lVert #1\rVert_{#2}}
\newcommand{\bigpnorm}[2]{\big\lVert#1\big\rVert_{#2}}
\newcommand{\biggpnorm}[2]{\bigg\lVert#1\bigg\rVert_{#2}}

\newcommand{\opA}{\mathcal{A}}
\newcommand{\opAA}{\mathcal{A^*A}}

\DeclareMathOperator{\id}{\mathcal{I}}

\newcommand{\uinorm}[1]{{\vert\kern-0.25ex\vert\kern-0.25ex\vert #1 
    \vert\kern-0.25ex\vert\kern-0.25ex\vert}}

\newcommand{\defeq}{\coloneqq}

\newcommand{\reals}{\mathbb{R}}

\newcommand{\T}{\top}

\newcommand{\Prob}{\mathbb{P}}
\newcommand{\Event}{\mathcal{E}}

\newcommand{\SVDU}{U}
\newcommand{\SVDSigma}{\Sigma}
\newcommand{\SVDV}{V}
\newcommand{\singval}{\sigma}

\newcommand{\myX}{X}
\newcommand{\Signal}{S}
\newcommand{\Noise}{N}

\newcommand{\Mtruth}{M_\star}
\newcommand{\Utruth}{\SVDU_\star}
\newcommand{\SigmaTruth}{\SVDSigma_\star}
\newcommand{\Xtruth}{\myX_\star}
\newcommand{\Uperp}[1]{\SVDU_{{#1,\perp}}}
\newcommand{\UperpTruth}{\Uperp{\star}}
\newcommand{\Uspec}{\widehat{\SVDU}}
\newcommand{\Vperp}[1]{\SVDV_{{#1,\perp}}}

\newcommand{\Ess}[1]{\widetilde{#1}}

\newcommand{\EssX}{\Ess{\myX}}
\newcommand{\EssS}{\Ess{\Signal}}
\newcommand{\Misalign}{\Ess{\Noise}}
\newcommand{\Overpar}{\Ess{O}}
\newcommand{\Err}[1]{{E}^{#1}}

\newcommand{\smax}{\singval_{\max}}
\newcommand{\smin}{\singval_{\min}}

\newcommand{\Tcvg}{T_{\min}}
\newcommand{\Ccvg}{C_{\min}}
\newcommand{\Tmax}{T_{\max}}
\newcommand{\Cmax}{C_{\max}}

\DeclareMathOperator{\fro}{\mathsf{F}}

\DeclareMathOperator{\rank}{\mathrm{rank}}

\DeclareMathOperator{\tr}{\mathrm{tr}}

\allowdisplaybreaks
\makeatletter 
\let\hat\widehat
\let\epsilon\varepsilon

\theoremstyle{plain}\newtheorem{lemma}{\textbf{Lemma}} 
\newtheorem{proposition}{\textbf{Proposition}}
\newtheorem{theorem}{\textbf{Theorem}}
\newtheorem{corollary}{\textbf{Corollary}} 
\newtheorem{assumption}{\textbf{Assumption}}

\theoremstyle{definition}\newtheorem{definition}{\textbf{Definition}}
 
\theoremstyle{remark}\newtheorem{remark}{\textbf{Remark}}

\definecolor{cm}{RGB}{250,0,200}
\definecolor{ys}{RGB}{100,0,200}
\definecolor{yc}{RGB}{255,0,0}
\definecolor{xingyu}{RGB}{0,150,0}

\let\hat\widehat

\begin{document}
\title{The Power of Preconditioning in Overparameterized \\ Low-Rank Matrix Sensing}

 \author
 {
 	Xingyu Xu\thanks{Carnegie Mellon University; Email:
 		\texttt{\{xingyuxu,yandis\}@andrew.cmu.edu}.} \\
     CMU
 		\and
	Yandi Shen\footnotemark[1]  \\
    CMU
		\and
 	Yuejie Chi\thanks{Yale University; Email:
 		\texttt{yuejie.chi@yale.edu.}}  \\
    Yale
			\and	
 	Cong Ma\thanks{University of Chicago; Email:
 		\texttt{congm@uchicago.edu}.}  \\
UChicago
 }

\date{February 2023; Revised December 2025}

\setcounter{tocdepth}{2}
\maketitle

\abstract{
We propose $\mathsf{ScaledGD(\lambda)}$, a preconditioned gradient descent method to tackle the low-rank matrix sensing problem when the true rank is unknown, and when the matrix is possibly ill-conditioned. Using overparameterized factor representations, $\mathsf{ScaledGD(\lambda)}$ starts from a small random initialization, and proceeds by gradient descent with
a specific form of {\em damped} preconditioning to combat bad curvatures induced by overparameterization and ill-conditioning. $\mathsf{ScaledGD(\lambda)}$ is remarkably robust to ill-conditioning compared to vanilla gradient descent ($\mathsf{GD}$) even with overparameterization. Specifically, we show that, under the restricted isometry property (RIP) of the sensing operator, $\mathsf{ScaledGD(\lambda)}$ converges to the true low-rank matrix at a constant linear rate after a small number of iterations that scales only {\em logarithmically} with respect to the condition number and the problem dimension. This significantly improves over the convergence rate of vanilla $\mathsf{GD}$ which suffers from a polynomial dependency on the condition number. 
Furthermore, we show that in the presence of measurement noise, $\mathsf{ScaledGD(\lambda)}$ converges to the minimax optimal error up to a multiplicative factor of the condition number at the same rate as in the noiseless setting, which is the first nearly minimax-optimal overparameterized gradient method for low-rank matrix sensing scaling with the true rank rather than the (possibly much larger) overparameterized rank. Our results also extend to the setting when the matrix is only approximately low-rank under the Gaussian design. 
Our work provides evidence on the power of preconditioning in accelerating the convergence without hurting generalization in overparameterized learning. 
}


\tableofcontents{}

\section{Introduction}

Low-rank matrix recovery plays an essential role in modern machine learning 
and signal processing. 
To fix ideas, let us consider estimating a rank-$r_\star$ positive semidefinite matrix $\Mtruth \in \mathbb{R}^{n\times n}$ based on a few linear measurements $y \coloneqq \mathcal{A} (M_\star)$, where $\mathcal{A} : \mathbb{R}^{n \times n} \to \mathbb{R}^{m}$ models the measurement process. 
Significant research efforts have been devoted to tackling low-rank matrix recovery in a statistically and computationally efficient manner in recent years. 
Perhaps the most well-known method is convex relaxation~\citep{candes2011tight,recht2010guaranteed,davenport2016overview}, which seeks the matrix with lowest nuclear norm to fit the observed measurements:
\begin{align*}
    \min_{M \succeq 0} \quad \|M\|_{*} \qquad \text{s.t.} \quad y = \mathcal{A} (M). 
\end{align*}
While statistically optimal, convex relaxation is prohibitive in terms of both computation and memory as it directly operates in the ambient matrix domain, i.e., $\mathbb{R}^{n \times n}$. 
To address this challenge, nonconvex approaches based on low-rank factorization have been proposed \citep{burer2005LRSDP}:
\begin{align}\label{eq:intro-nonconvex}
  \min_{X\in\reals^{n\times r}} \quad  \frac{1}{4}\big\|\opA(XX^\T)- y\big\|_2^2,
\end{align}
where $r$ is a user-specified rank parameter. Despite nonconvexity, when the rank is correctly specified, i.e., when $r = r_\star$, the problem~\eqref{eq:intro-nonconvex} admits computationally efficient solvers~\citep{chi2019nonconvex}, e.g., gradient descent (\GD) with spectral initialization or with small random initialization.    
However, three main challenges remain when applying the factorization-based nonconvex approach (\ref{eq:intro-nonconvex}) in practice.
\begin{itemize}
    \item {\bf Unknown rank}. First, the true rank $r_\star$ is often unknown, which makes it infeasible to set $r = r_\star$. One necessarily needs to consider an overparameterized setting in which $r$ is set conservatively, i.e., one sets $r \geq r_\star$ or even $r =n$.
    \item {\bf Poor conditioning}. Second, the ground truth matrix $\Mtruth$ may be ill-conditioned, which is commonly encountered in practice. Existing approaches such as gradient descent are still computationally expensive in such settings as the number of iterations necessary for convergence increases with the condition number.  
    \item {\bf Robustness to noise and approximate low-rankness}. Last but not least, it is desirable that the performance is robust when the measurement $y$ is contaminated by noise and when $\Mtruth$ is approximately low-rank. 
\end{itemize}

\noindent In light of these two challenges, the main goal of this work is to address the following question: 
\begin{center}
 {\em Can one develop an efficient and robust method for solving ill-conditioned matrix recovery in the overparameterized setting?}
\end{center}  

\nopagebreak

\begin{table*}[t]
\resizebox{\textwidth}{!}{  
\centering
\begin{tabular}{c|c|c|c|c} 
\toprule
  parameterization   &  reference & algorithm &  init. & iteration complexity      \tabularnewline  \midrule
 \multirow{3}{*}{$r>r_{\star}$}      &  \citet{stoger2021small}  \vphantom{$\frac{1^{7^{7^{7}}}}{1^{7^{7^{7}}}}$} & \GD & random & $ \kappa^8 + \kappa^6\log(\kappa n/\varepsilon) $    \tabularnewline 
\cline{2-5}  
    &  \citet{zhang2021preconditioned} \vphantom{$\frac{1^{7^{7^{7}}}}{1^{7^{7^{7}}}}$}   & $\mathsf{PrecGD}$ & spectral &  $ \log(1/\varepsilon)$  \tabularnewline 
    \cline{2-5}  
    & {\bf  Theorem~\ref{thm:main} } \vphantom{$\frac{1^{7^{7^{7}}}}{1^{7^{7^{7}}}}$}   & \myalg & random &  $\log\kappa \cdot \log (\kappa n) + \log(1/\varepsilon)$  \tabularnewline \midrule
 \multirow{3}{*}{$r= r_{\star}$}    &  \citet{tong2021accelerating}  \vphantom{$\frac{1^{7^{7^{7}}}}{1^{7^{7^{7}}}}$}     & \ScaledGD & spectral & $ \log(1/\varepsilon)$    \tabularnewline
   \cline{2-5}
 &    \citet{stoger2021small}  \vphantom{$\frac{1^{7^{7^{7}}}}{1^{7^{7^{7}}}}$}   & \GD & random & $ \kappa^8  \log (\kappa n)  + \kappa^2\log(1/\varepsilon) $   \tabularnewline
   \cline{2-5}
 &  {\bf  Theorem~\ref{thm:zero} } \vphantom{$\frac{1^{7^{7^{7}}}}{1^{7^{7^{7}}}}$}    & \myalg & random & $\log\kappa \cdot \log (\kappa n)  + \log(1/\varepsilon)$     \tabularnewline  \bottomrule 
\end{tabular}
}
\caption{Comparison of iteration complexity with existing algorithms for low-rank matrix sensing under Gaussian designs.  
  Here, $n$ is the matrix dimension, $r_{\star}$ is the true rank, $r$ is the overparameterized rank, and $\kappa$ is the condition number of the problem instance (see Section~\ref{sec:formulation} for a formal problem formulation). It is important to note that in the overparameterized setting ($r>r_{\star}$), the sample complexity of \citet{zhang2021preconditioned} scales polynomially with the overparameterized rank $r$, while that of \citet{stoger2021small}  and ours only scale polynomially with  the true rank $r_{\star}$.
}
\label{table:comparison}
\end{table*}

\subsection{Our contributions: a preview}

The main contribution of the current paper is to answer the question affirmatively by developing a \emph{preconditioned} gradient descent method (\myalg) that converges to the (possibly ill-conditioned) low-rank matrix in a fast and global manner, even with overparamterized rank $r \geq r_\star$.   

\begin{theorem}[Informal]
Under overparameterization $r\geq r_{\star}$ and mild statistical assumptions, \myalg---starting from a sufficiently small random initialization with a sample complexity depending polynomially with  the true rank $r_{\star}$ ---achieves a relative $\varepsilon$-accuracy, i.e., $\|\myX_T\myX_T^\T-\Mtruth\|_{\fro}\le \epsilon\|\Mtruth\|$,
with no more than an order of 
$$\log\kappa \cdot \log (\kappa n) +  \log (1 / \varepsilon)  $$ iterations, where $\kappa$ is the condition number of the problem.  
Moreover, in the presence of per-entry Gaussian measurement noise $\mathcal{N}(0, \sigma^2)$, \myalg converges to the nearly minimax-optimal error 
$$   \|X_T X_T^\T - \Mtruth\|_{\fro} \lesssim \kappa^4 \sigma \sqrt{n r_{\star}}$$ 
with the same rate as above.
\end{theorem} 
The above theorem suggests that from a small random initialization, \myalg converges at a constant linear rate---independent of the condition number---after a small logarithmic number of iterations. Overall, the iteration complexity is nearly independent of the condition number and the problem dimension, making it extremely suitable for solving large-scale and ill-conditioned problems. To the best of our knowledge, \myalg is the first provably minimax-optimal overparameterized gradient method for low-rank matrix sensing, where both the sample complexity and the error bound depend on the true rank $r_{\star}$. In contrast, prior error bounds for nonconvex gradient methods \cite{zhang2024fast,zhuo2021computational} scale with the overparameterized rank $r$, which can be significantly larger. Our results also extend to the setting when the matrix $\Mtruth$ is only approximately low-rank under the Gaussian design, which is new. See Table~\ref{table:comparison} for a summary of comparisons with prior art in the noiseless setting.

 Our algorithm \myalg is closely related to scaled gradient descent (\ScaledGD) \citep{tong2021accelerating}, a recently proposed preconditioned gradient descent method that achieves a $\kappa$-independent convergence rate under spectral initialization and exact parameterization. We modify the preconditioner design by introducing a fixed damping term, which prevents the preconditioner itself from being ill-conditioned due to overparameterization; the modified preconditioner preserves the low computational overhead when the overparameterization is moderate. In the exact parameterization setting, our result extends \ScaledGD beyond local convergence by characterizing the number of iterations it takes to enter the local basin of attraction from a small random initialization. 

Moreover, our results shed light on the power of preconditioning in accelerating the optimization process over vanilla \GD while still guaranteeing generalization in overparameterized learning models \citep{amari2020does}. Remarkably, despite the existence of an infinite number of global minima in the landscape of \eqref{eq:intro-nonconvex} that do not generalize, i.e., not corresponding to the ground truth, starting from a small random initialization, \GD \citep{li2018algorithmic,stoger2021small} is known to converge to a generalizable solution without explicit regularization. However, \GD takes $O(\kappa^8 + \kappa^6\log(\kappa n/\varepsilon))$ iterations to reach $\varepsilon$-accuracy, which is unacceptable even for moderate condition numbers. On the other hand, while common wisdom suggests that preconditioning accelerates convergence, it is yet unclear if it still converges to a generalizable global minimum. Our work answers this question in the affirmative for overparameterized low-rank matrix sensing, where \myalg significantly accelerates the convergence against the poor condition number---both in the initial phase and in the local phase---without hurting generalization, which is corroborated in Figure~\ref{fig:intro}. 

\begin{figure}[t]
    \centering
    \includegraphics[width=0.5\textwidth]{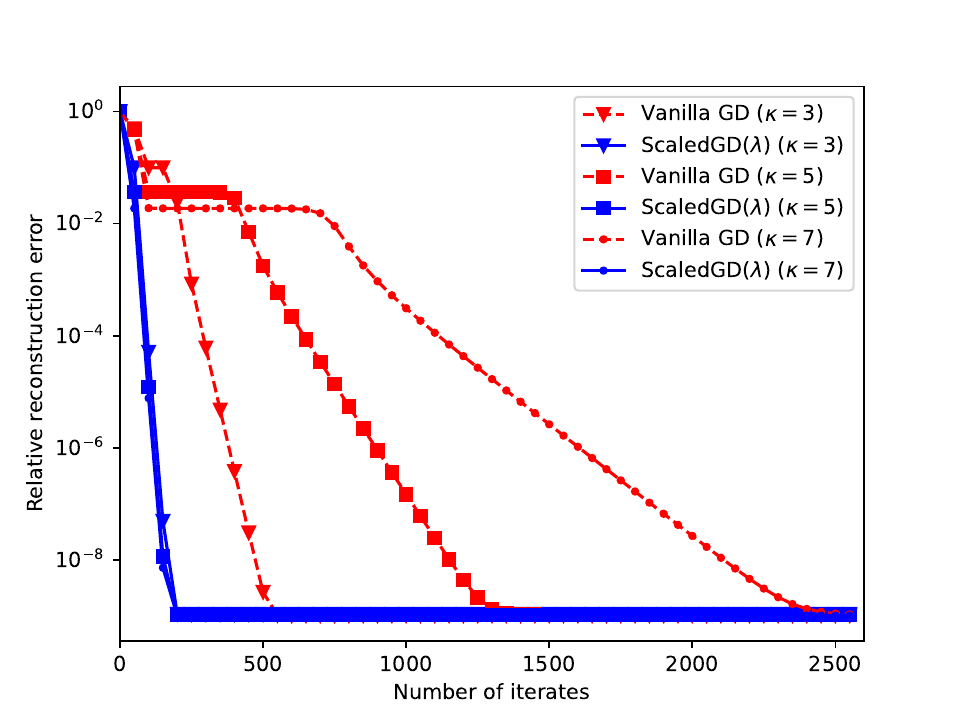}
    \caption{Comparison between \myalg and \GD. The learning rate of \GD has been fine-tuned to achieve fastest convergence for each $\kappa$, while that of \myalg is fixed to $0.3$. The initialization scale $\alpha$ in each case has been fine-tuned so that the final accuracy is $10^{-9}$. The details of the experiment are deferred to Section~\ref{sec:numerical}.  }
    \label{fig:intro} 
\end{figure}

\subsection{Related work}

Significant efforts have been devoted to understanding nonconvex optimization for low-rank matrix estimation in recent years, see \cite{chi2019nonconvex} and \cite{chen2018harnessing} for recent overviews. By reparameterizing the low-rank matrix into a product of factor matrices, also known as the Burer-Monteiro factorization \citep{burer2005LRSDP}, the focus point has been examining if the factor matrices can be recovered---up to invertible transformations---faithfully using simple iterative algorithms in a provably efficient manner. However, the majority of prior efforts suffer from the limitations that they assume an exact parameterization where the rank of the ground truth is given or estimated somewhat reliably, and rely on a carefully constructed initialization (e.g., using the spectral method \citep{chen2021spectral}) in order to guarantee global convergence in a polynomial time. The analyses adopted in the exact parameterization case fail to generalize when overparameterization presents, and drastically new approaches are called for.

\paragraph{Overparameterization in low-rank matrix sensing.} \citet{li2018algorithmic} made a theoretical breakthrough that showed that gradient descent converges globally to any prescribed accuracy even in the presence of full overparameterization ($r=n$), with a small random initialization, where their analyses were subsequently adapted and extended in \citet{stoger2021small} and \citet{zhuo2021computational}. \citet{ding2021rank} investigated robust low-rank matrix recovery with overparameterization from a spectral initialization, and  \citet{ma2022global} examined the same problem from a small random initialization with noisy measurements. \citet{zhang2021preconditioned,zhang2022preconditioned} developed a preconditioned gradient descent method for overparameterized low-rank matrix sensing, where an adaptive damping parameter is introduced in  \ScaledGD. A variant with global convergence guarantee is studied in \citet{zhang2022preconditioned}, which requires adding perturbation at the initial stage to first converge to a second-order stationary point before switching to a fast local convergence. Last but not least, a number of other notable works that study overparameterized low-rank models include, but are not limited to,
\citet{soltanolkotabi2018theoretical,geyer2020low,oymak2019overparameterized,zhang2021sharp,zhang2022improved}.

\paragraph{Global convergence from random initialization without overparameterization.} Despite nonconvexity, it has been established recently that several structured learning models admit global convergence via simple iterative methods even when initialized randomly even without overparameterization. For example, \citet{chen2019gradient} showed that phase retrieval converges globally from a random initialization using a near-minimal number of samples through a delicate leave-one-out analysis. In addition, the efficiency of randomly initialized \GD is established for 
complete dictionary learning \citep{gilboa2019efficient,bai2018subgradient}, multi-channel sparse blind deconvolution \citep{qu2019nonconvex,shi2021manifold}, asymmetric low-rank matrix factorization \citep{ye2021global}, and rank-one matrix completion \citep{kim2022rank}. Moving beyond \GD, \citet{lee2022randomly}
 showed that randomly initialized alternating least-squares converges globally for rank-one matrix sensing, whereas \citet{chandrasekher2022alternating} developed sharp recovery guarantees of alternating minimization for generalized rank-one matrix sensing with sample-splitting and random initialization.

\paragraph{Algorithmic or implicit regularization.} Our work is related to the phenomenon of algorithmic or implicit regularization \citep{gunasekar2017implicit}, where the trajectory of simple iterative algorithms follows a path that maintains desirable properties without explicit regularization. Along this line, \citet{ma2017implicit,chen2019nonconvex,li2021nonconvex} highlighted the implicit regularization of \GD  
for several statistical estimation tasks, \citet{ma2021beyond} showed that \GD automatically balances the factor matrices in asymmetric low-rank matrix sensing, where \citet{jiang2022algorithmic} analyzed the algorithmic regularization in overparameterized asymmetric matrix factorization in a model-free setting.


\section{Problem formulation}
\label{sec:formulation}

Section~\ref{sec:models} introduces the problem of low-rank matrix sensing, and 
Section~\ref{sec:alg} provides background on the proposed \myalg algorithm developed 
for the possibly overparameterized case.

\subsection{Model and assumptions}
\label{sec:models}

Suppose that the ground truth $\Mtruth \in \mathbb{R}^{n\times n}$ is a positive-semidefinite (PSD) matrix  of rank $r_\star \ll n$, whose (compact) eigendecomposition is given by
$$\Mtruth=\Utruth\SigmaTruth^2\Utruth^\T. $$
Here, the columns of $\Utruth\in\mathbb{R}^{n\times r_\star}$ specify the set of eigenvectors, and $\SigmaTruth \in \mathbb{R}^{r_\star \times r_\star}$ is a diagonal matrix where the diagonal entries are ordered in a non-increasing fashion. Setting $\Xtruth \defeq \Utruth\SigmaTruth\in\mathbb{R}^{n\times r_\star}$, we can rewrite $\Mtruth$ as
\begin{equation}\label{eq:true_factorization}  
\Mtruth =\Xtruth\Xtruth^{\top}. 
\end{equation}
We call $\Xtruth$ the ground truth low-rank factor matrix, whose condition number $\kappa$ is defined as
\begin{equation} \label{eq:def_kappa}
  \kappa \coloneqq \frac{\smax(\Xtruth)}{\smin(\Xtruth)}.
\end{equation}
Here we recall that $\smax(\Xtruth)$ and $\smin(\Xtruth)$ are the largest and the  smallest singular values of $\Xtruth$, respectively.

Instead of having access to $\Mtruth$ directly, we wish to recover $\Mtruth$ from a set of random linear measurements $\opA(\Mtruth)$, 
where $\opA: \operatorname{Sym}_2(\reals^n)\to\mathbb{R}^m$ is a linear map 
from the space of $n\times n$ symmetric  matrices to $\mathbb{R}^m$, namely
\begin{equation}\label{eq:measurements}
y = \opA(\Mtruth),
\end{equation}
or equivalently,
\begin{equation*}
\quad y_i = \langle A_i, \Mtruth \rangle, \qquad 1\leq i\leq m.
\end{equation*}
We are interested in recovering $\Mtruth$ based on the measurements $y$ and the sensing operator $\opA$ in a provably  efficient manner, even when the true rank $r_\star$ is unknown.

\subsection{\myalg for overparameterized low-rank matrix sensing}
\label{sec:alg}
 
Inspired by the factorized representation~\eqref{eq:true_factorization}, we aim to recover the low-rank matrix $\Mtruth$ 
by solving the following optimization problem \citep{burer2005LRSDP}:
\begin{equation}\label{eqn:b-m}
  \min_{X\in\reals^{n\times r}} \quad f(X) \coloneqq \frac{1}{4} \big\|\opA(XX^\T)- y\big\|_2^2,
\end{equation}
where $r$ is a predetermined rank parameter, possibly different from $r_{\star}$. It is evident that for any rotation matrix $O\in\mathcal{O}_r$, it holds that $f(X) = f(XO)$, leading to an infinite number of global minima of the loss function~$f$.

\paragraph{A prelude: exact parameterization.} When $r$ is set to be the true rank $r_{\star}$ of $\Mtruth$, \citet{tong2021accelerating} set forth a provable algorithmic approach called scaled gradient descent (\ScaledGD)---gradient descent with a specific form of preconditioning---that adopts the following update rule
\begin{align}\label{eq:scaledGD_original}
 \mathsf{ScaledGD:}\qquad  \myX_{t+1} & =\myX_{t}-\eta  \nabla f(\myX_{t})(\myX_t^\T \myX_t )^{-1}  \\
 &  =\myX_{t}-\eta \opAA(\myX_t\myX_t^\T-\Mtruth)\myX_t  (\myX_t^\T \myX_t )^{-1} . \nonumber
\end{align} 
Here, $\myX_t$ is the $t$-th iterate, $ \nabla f(X_{t})$ is the gradient of $f$ at $X= X_t$, and $\eta>0$ is the learning rate. Moreover, $\opA^*: \mathbb{R}^m \mapsto \operatorname{Sym}_2(\mathbb R^n)$ is the adjoint operator of $\opA$, that is $\opA^*(y) = \sum_{i=1}^m y_i A_i$ for $y\in \mathbb{R}^m$. 

At the expense of light computational overhead, \ScaledGD is remarkably robust to ill-conditioning compared with vanilla gradient descent (\GD). 
It is established in \citet{tong2021accelerating} that \ScaledGD, when starting from spectral initialization, converges linearly at a constant rate---{\em independent} of the condition number $\kappa$ of $\Xtruth$ (cf.~\eqref{eq:def_kappa}); in contrast, the iteration complexity of \GD \citep{tu2015low,zheng2015convergent} scales on the order of $\kappa^2$ from the same initialization, therefore \GD becomes exceedingly slow when the problem instance is even moderately ill-conditioned, a scenario that is quite commonly encountered in practice. 

 \paragraph{\myalg: overparametrization under unknown rank. }
In this paper, we are interested in the so-called overparameterization regime, where $ r_\star \leq r\leq n$. From an operational perspective, the true rank $r_{\star}$ is related to model order, e.g., the number of sources or targets in a scene of interest, which is often unavailable and makes it necessary to consider the misspecified setting. Unfortunately, in the presence of overparameterization, the original \ScaledGD algorithm is no longer appropriate, as the preconditioner $(\myX_t^\T \myX_t )^{-1}$ might become numerically unstable to calculate. Therefore, we propose a new variant of \ScaledGD by adjusting the preconditioner as
\begin{align}\label{eqn:update}
 \mathsf{ScaledGD(\lambda):}\qquad \myX_{t+1} & =\myX_{t}-\eta \nabla f(\myX_{t})(\myX_t^\T \myX_t+\lambda I)^{-1},  \\
  & =\myX_{t}-\eta \opAA(\myX_t\myX_t^\T-\Mtruth)\myX_t  (\myX_t^\T \myX_t+\lambda I)^{-1}, \nonumber
\end{align}
where $\lambda>0$ is a {\em fixed} damping parameter. The new algorithm is dubbed as \myalg, and it recovers the original \ScaledGD when $\lambda = 0$. Similar to \ScaledGD, a key property of \myalg is that the iterates $\{X_t\}$ are equivariant with respect to the parameterization of the factor matrix. Specifically, taking a rotationally equivalent factor $X_t O$ with an arbitrary $O\in\mathcal{O}_r$, and feeding it into the update rule \eqref{eqn:update}, the next iterate 
$$ \myX_{t}O -\eta\opAA(\myX_t\myX_t^\T-\Mtruth)\myX_t O( O^{\T}\myX_t^\T \myX_t O+\lambda I)^{-1}  =  \myX_{t+1} O $$
is rotated simultaneously by the same rotation matrix $O$. In other words, the recovered matrix sequence $M_t = X_tX_t^{\T}$ is invariant with respect to the parameterization of the factor matrix.

\begin{remark}
We note that a related variant of \ScaledGD, called $\mathsf{PrecGD}$, has been proposed recently in \citet{zhang2021preconditioned,zhang2022preconditioned} for the overparameterized setting, which follows the update rule
\begin{align}\label{eqn:precGD}
 \mathsf{PrecGD:}\qquad \myX_{t+1}=\myX_{t}-\eta\opAA(\myX_t\myX_t^\T-\Mtruth)\myX_t(\myX_t^\T \myX_t+\lambda_t I)^{-1},
\end{align}
where the damping parameters $\lambda_t =\sqrt{f(X_t)}$ are selected in an {\em iteration-varying} manner. In contrast, \myalg assumes a fixed damping parameter $\lambda$ throughout the iterations. We defer more detailed comparisons with $\mathsf{PrecGD}$ in Section~\ref{sec:main}. 
\end{remark}

\section{Main results}
 \label{sec:main}

Before formally presenting our theorems, let us introduce several key assumptions that will be in effect throughout this paper.

\paragraph{Restricted Isometry Property.}
A key property of the operator $\opA(\cdot)$ is the celebrated Restricted Isometry Property (RIP) \citep{recht2010guaranteed}, which says that the operator $\opA(\cdot)$ approximately preserves the distances between low-rank matrices. The formal definition is given as follows.

\begin{definition}[Restricted Isometry Property] \label{def:RIP}
The linear map $\opA(\cdot)$ is said to obey rank-$r$ RIP with a constant $\delta_r\in [0, 1)$,
if for all matrices $M\in\operatorname{Sym}_2(\mathbb R^n)$ of rank at most $r$, it holds that
\begin{equation}\label{eqn:rip}
  (1-\delta_r)\|M\|_{\fro}^2\le \big\|\opA(M) \big\|_2^2\le (1+\delta_r)\|M\|_{\fro}^2.
\end{equation}
The Restricted Isometry Constant (RIC) is defined to be 
the smallest positive $\delta_r$ such that \eqref{eqn:rip} holds.
\end{definition}
The RIP is a standard assumption in low-rank matrix sensing, which has been verified to hold with high probability for a wide variety of measurement operators. The following lemma establishes the RIP for the Gaussian design.

\begin{lemma}\label{lem:gaussian_design}\citep[Lemma 1]{stoger2024non}
If the sensing operator $\opA(\cdot)$ follows the {\em Gaussian design}, i.e., the entries of $\{A_i\}_{i=1}^m$ are independent up to symmetry with diagonal elements sampled from $\mathcal{N}(0,1/m)$ and off-diagonal elements from $\mathcal{N}(0,1/2m)$, then with high probability, $\opA(\cdot)$ satisfies rank-$r$ RIP with constant $\delta_r$, as long as $m \geq Cnr/\delta_r^2$ for some sufficiently large universal constant $C > 0$.  
\end{lemma}

We make the following assumption about the  operator $\opA(\cdot)$. 

\begin{assumption} \label{assumption:opA}
The operator $\opA(\cdot)$ satisfies the rank-$(r_\star+1)$ RIP 
with $\delta_{r_\star+1} \eqqcolon \delta$. Furthermore, there exist a sufficiently small constant $c_\delta>0$ and a sufficiently large constant $C_\delta > 0$ such that
\begin{equation}     \label{eqn:delta-cond}
    \delta \le c_\delta r_\star^{-1/2}\kappa^{- C_\delta }.
\end{equation}    
\end{assumption}

\paragraph{Small random initialization.} Similar to \citet{li2018algorithmic,stoger2021small}, 
we set the initialization $\myX_0$ to be a small random matrix, i.e.,
\begin{equation}
  \myX_0=\alpha G,
\end{equation}
where $G \in \mathbb{R}^{n\times r}$ is some matrix considered to be normalized 
and $\alpha>0$ controls the magnitude of the initialization.
To simplify exposition, we take $G$ to be a standard random Gaussian matrix, 
that is, $G$ is a random matrix with i.i.d. entries distributed as $\mathcal{N}(0,1/n)$. 

\paragraph{Choice of parameters.} Last but not least, the parameters of \myalg are selected according to the following assumption.

\begin{assumption} \label{assumption:param}
There exist some universal constants 
$c_\eta, c_\lambda ,  C_\alpha>0$ 
such that $(\eta, \lambda, \alpha)$ in \myalg
satisfy the following conditions:
\begin{subequations} \label{eq:param_conditions}
  \begin{align}
\mathsf{(learning~rate)} & \qquad\qquad\qquad  \eta  \le c_\eta, 
    \label{eqn:eta-cond}\\
 \mathsf{(damping~parameter)} & \qquad    \frac{1}{100}c_\lambda   \kappa^{-4}  \smin^2(\Xtruth)  \le\lambda \le c_\lambda\smin^2(\Xtruth), 
    \label{eqn:lambda-cond}\\
   \mathsf{(initialization~size)} & \qquad     \log\frac{\|\Xtruth\|}{\alpha}  \ge \frac{C_\alpha}{  \max(\eta, \kappa^{-2}) }\log(2\kappa)\cdot\log(2\kappa n). 
    \label{eqn:alpha-cond}
  \end{align}
\end{subequations}
\end{assumption}

We are now in place to present the main theorems.

\subsection{The overparameterization setting}

We begin with our main theorem, which characterizes the performance of \myalg with overparameterization.
\begin{theorem}\label{thm:main}
Suppose Assumptions~\ref{assumption:opA} and \ref{assumption:param} hold.
With high probability (with respect to the realization of the random initialization $G$),
there exists a universal constant $\Ccvg > 0$ such that
for some $T\le\Tcvg \coloneqq \frac{\Ccvg}{\eta}\log\frac{\|\Xtruth\|}{\alpha}$, we have 
\[
  \|\myX_T \myX_T^\T-\Mtruth\|_{\fro}\le\alpha^{1/3}\|\Xtruth\|^{5/3}.
\]
In particular, for any prescribed accuracy target $\epsilon\in(0,1)$, 
by choosing a sufficiently small $\alpha$ fulfilling both~\eqref{eqn:alpha-cond} 
and $\alpha\le\epsilon^3\|\Xtruth\|$, we have
$\|\myX_T\myX_T^\T-\Mtruth\|_{\fro}\le \epsilon\|\Mtruth\|$.
\end{theorem}

A few remarks are in order.

\paragraph{Iteration complexity.} 
Theorem~\ref{thm:main} shows that by choosing an appropriate $\alpha$, \myalg finds an $\epsilon$-accurate solution, i.e., $\|\myX_t\myX_t^\T-\Mtruth\|_{\fro}\le \epsilon\|\Mtruth\|$, in no more than an order of
$$ \log\kappa \cdot \log (\kappa n) + \log(1/\epsilon)$$ iterations. Roughly speaking, this asserts that \myalg converges at a constant linear rate after an initial phase of approximately $O(\log\kappa \cdot \log (\kappa n) )$ iterations.
Most notably, the iteration complexity is nearly independent of the condition number $\kappa$, with a small overhead only through the poly-logarithmic additive term $O(\log\kappa \cdot \log (\kappa n))$. In contrast, \GD requires $O(\kappa^8 + \kappa^6\log(\kappa n/\varepsilon) )$ iterations to converge from a small random initialization to $\epsilon$-accuracy; see~\citet{stoger2021small, li2018algorithmic}. Thus, the convergence of \GD is much slower than \myalg even for mildly ill-conditioned matrices. 
 

\paragraph{Sample complexity.} The sample complexity of \myalg hinges upon the Assumption~\ref{assumption:opA}. When the sensing operator $\opA(\cdot)$ follows the Gaussian design, this assumption is fulfilled as long as $m\gtrsim nr_\star^2\cdot \mathsf{poly}(\kappa)$. Notably, our sample complexity depends only on the true rank $r_\star$, but not on the overparameterized rank $r$ --- a crucial feature in order to provide meaningful guarantees when the overparameterized rank $r$ is close to the full dimension $n$. The dependency on $\kappa$ in the sample complexity, on the other end, is believed to be an artifact of the proof, as empirically shown in some related settings (see e.g., Figure 4 of \cite{chen2019noisy}). Rigorously proving this, however, remains an open problem in nonconvex low-rank estimation \citep{chi2019nonconvex}.

\paragraph{Comparison with \cite{zhang2021preconditioned,zhang2022preconditioned}.} 
As mentioned earlier, our proposed algorithm \myalg is similar to $\mathsf{PrecGD}$ proposed in \cite{zhang2021preconditioned} that adopts an iteration-varying damping parameter in \ScaledGD \cite{tong2021accelerating}, with several important distinctions. In terms of theoretical guarantees, \cite{zhang2021preconditioned} only provides the local convergence for $\mathsf{PrecGD}$ assuming an initialization close to the ground truth; in contrast, we provide global convergence guarantees where a small random initialization is used. More critically, the sample complexity of $\mathsf{PrecGD}$ \cite{zhang2021preconditioned} depends on the overparameterized rank $r$, while ours only depends on the true rank $r_{\star}$. While \cite{zhang2022preconditioned} also studied variants of $\mathsf{PrecGD}$ with global convergence guarantees, they require additional operations such as gradient perturbations and switching between different algorithmic stages, which are harder to implement in practice. Furthermore, their convergence rate is much more pessimistic than ours. Our theory suggests that additional perturbation is  unnecessary to ensure the global convergence of \myalg, as \myalg automatically adapts to different curvatures of the optimization landscape throughout the entire trajectory.

\subsection{The exact parameterization setting}\label{subsec:exact_para}
We now single out the exact parametrization case, i.e., when $r= r_\star$. In this case, our theory suggests that \myalg converges to the ground truth even from a random initialization with a fixed scale $\alpha>0$. 

\begin{theorem}\label{thm:zero}
  Assume that $r=r_\star$. Suppose Assumptions~\ref{assumption:opA} and \ref{assumption:param} hold.
With high probability (with respect to the realization of the random initialization $G$), there exist some universal constants $\Ccvg > 0$ and $c > 0$ such that
for some $T\le\Tcvg=\frac{\Ccvg}{\eta}\log(\|\Xtruth\|/\alpha)$, we have
  for any $t\ge T$  
  \[\|\myX_t\myX_t^\T-\Mtruth\|_{\fro}
  \le (1-c\eta)^{t-T}\|\Mtruth\|.\]
\end{theorem}

Theorem~\ref{thm:zero} shows that with some fixed initialization scale $\alpha$, 
\myalg takes at most an order of
$$ \log\kappa \cdot \log (\kappa n) + \log(1/\epsilon)$$
iterations 
to converge to $\epsilon$-accuracy for any $\epsilon>0$ in the exact parameterization case. 
Compared with \ScaledGD \citep{tong2021accelerating} which takes $O(\log(1/\epsilon))$ iterations to converge from a spectral initialization, 
we only pay a logarithmic order $O(\log\kappa \cdot \log (\kappa n) )$ of additional iterations to converge from a random initialization. In addition, once the algorithms enter the local regime, both \myalg and \ScaledGD behave similarly and converge at a fast constant linear rate, suggesting the effect of damping is locally negligible. Furthermore, compared with \GD \citep{stoger2021small} which requires $O(\kappa^8\log(\kappa n)+\kappa^2 \log(1/\epsilon))$ iterations 
to achieve $\epsilon$-accuracy, our theory again highlights the benefit of \myalg in boosting the global convergence even for mildly ill-conditioned matrices.

\subsection{The noisy setting}\label{subsec:noisy}

We next consider the case where the measurements are contaminated by noise $\xi=(\xi_i)_{i=1}^m$, that is 
\begin{equation}
\label{eqn:noisy-obsv}
y = \opA(\Mtruth) + \xi, 
\quad\text{or more concretely}\quad  y_i = \langle A_i, \Mtruth\rangle + \xi_i, \quad 1\le i\le m.
\end{equation}
Instantiating \eqref{eqn:update} with the noisy measurements, the update rule of \myalg can be written as
\begin{equation}
\label{eqn:noisy-update}
  X_{t+1} = X_t - \eta  \big(\opAA(X_t X_t^\T) - \opA^*(y)\big) X_t (X_t^\T X_t + \lambda I)^{-1}.
\end{equation}

For simplicity, we make the following mild assumption on the noise. 
\begin{assumption} \label{assumption:noise}
We assume that $\xi_i$'s are independent with $\opA(\cdot)$, and are i.i.d.~Gaussian, i.e.,
\[\xi_i\stackrel{\text{i.i.d.}}{\sim}\mathcal N(0, \sigma^2), \quad 1\le i\le m. \]
\end{assumption}

Our theory demonstrates that \myalg achieves the minimax-optimal error in this noisy setting as long as the noise is not too large.
\begin{theorem}
\label{thm:noisy}
Assume that $\sigma\sqrt{n}\le c_\sigma\kappa^{-C_\sigma}\|\Mtruth\|$ for some sufficiently small universal constant $c_\sigma>0$ and some sufficiently large universal constant $C_\sigma>0$. Then the following holds with high probability (with respect to the realization of the random initialization $G$ and the noise $\xi$). Suppose Assumptions~\ref{assumption:opA}, \ref{assumption:param} and \ref{assumption:noise} hold. Given a prescribed accuracy target $\epsilon\in(0,1)$, suppose further that $\alpha\le\epsilon^3\|\Xtruth\|$. There exist  universal constants $\Ccvg > 0$, $C_{\ref{thm:noisy}}>0$,  such that for some $T\le\Tcvg \coloneqq \frac{\Ccvg}{\eta}\log\frac{\|\Xtruth\|}{\alpha}$, we have
\begin{align*}
  \|X_T X_T^\T - \Mtruth\| &\le \max\left(\varepsilon\|\Mtruth\|,~ C_{\ref{thm:noisy}} \kappa^4 \sigma\sqrt{n}\right),
  \\
  \|X_T X_T^\T - \Mtruth\|_{\fro} &\le \max\left(\varepsilon\|\Mtruth\|,~C_{\ref{thm:noisy}} \kappa^4 \sigma\sqrt{n r_\star}\right).
\end{align*}
\end{theorem}

A few remarks are in order.

\paragraph{Minimax optimality.} Theorem~\ref{thm:noisy} suggests that as long as the noise level is not too large, by setting the optimization error $\varepsilon$ sufficiently small, i.e., $\varepsilon \|\Mtruth\| \asymp \kappa^4 \sigma \sqrt{n}$,   \myalg finds a solution that satisfies
\begin{align}
  \|X_T X_T^\T - \Mtruth\| \lesssim \kappa^4 \sigma \sqrt{n}, \qquad   \|X_T X_T^\T - \Mtruth\|_{\fro} \lesssim \kappa^4 \sigma \sqrt{n r_{\star}}
\end{align}
in no more than 
$ \log\kappa \cdot \log (\kappa n) + \log \left( \frac{ \|\Mtruth\|} { \kappa^4 \sigma \sqrt{n}} \right)$ iterations, the number of which again only depends logarithmically on the problem parameters. 
When $\kappa$ is upper bounded by a constant, our result is minimax optimal, in the sense that the final error matches the minimax lower bound in the classical work of \cite{candes2011tight}, which we recall here for completeness: for any estimator $\hat M(y)$ based on the measurement $y$ defined in~\eqref{eqn:noisy-obsv}, for any $r_\star\le n$, there always exists some $\Mtruth\in\reals^{n\times n}$ of rank $r_\star$ such that
\[
  \|\hat M(y) - \Mtruth\|\gtrsim\sigma\sqrt{n} ,
  \qquad
  \|\hat M(y) - \Mtruth\|_{\fro}\gtrsim\sigma\sqrt{nr_\star} 
\]
with probability at least $0.99$ (with respect to the realization of the noise $\xi$). 
To the best of our knowledge, Theorem~\ref{thm:noisy} is the first result to establish the  minimax optimality  (up to multiplicative factors of $\kappa$) of overparameterized gradient methods in the context of low-rank matrix sensing. We remark that similar sub-optimality with respect to $\kappa$ is also observed in \citet{chen2019noisy}.

\paragraph{Consistency.} It is often desirable that the estimator is (asymptotically) consistent, i.e., the estimation error converges to $0$ as the number of samples $m\to\infty$. 
To see that Theorem~\ref{thm:noisy} indicates \myalg indeed produces a consistent estimator, let us consider again the Gaussian design.
In this case, $\langle A_i, \Mtruth\rangle$ is on the order of $\|\Mtruth\|/\sqrt{m}$, thus the signal-to-noise ratio can be measured by $\mathsf{SNR}\coloneqq (\|\Mtruth\|/\sqrt{m})^2 / \sigma^2  = \|\Mtruth\|^2/(m\sigma^2)$. 
With this notation, Theorem~\ref{thm:noisy} asserts that the final error is $O\big(\mathsf{SNR}^{-1/2} \sqrt{\frac nm} \|\Mtruth\|\big)$ in operator norm and $O\big(\mathsf{SNR}^{-1/2}\sqrt{\frac{nr_\star}{m}}\|\Mtruth\|\big)$ in Frobenius norm, both of which converge to $0$ at a rate of $\sqrt{\frac{nr_\star}{m}}$ as $m\to\infty$ when $\mathsf{SNR}$ is fixed. 

\subsection{The approximately low-rank setting}

Last but not least, we examine a more general model of $\Mtruth$, which does not need to be exactly low-rank, but only approximately low-rank. Instead of recovering $\Mtruth$ exactly, one seeks to find a low-rank approximation to $\Mtruth$ from its linear measurements.

To set up, let $\Mtruth\in\reals^{n\times n}$ be a general PSD ground truth matrix, where its spectral decomposition is given by $\Mtruth=\sum_{i=1}^n 
\sigma_i u_i u_i^\T$, with 
$$\sigma_1\ge \sigma_2\ge\cdots\ge\sigma_n.$$ 
For any given $r\le n$, 
let $M_{r}$ be the best rank-$r$ approximation of $\Mtruth$ and $M_r'$ be the residual, i.e.,
\begin{equation}
\label{eqn:approx-low-rank}
\Mtruth = \underbrace{\sum_{i=1}^{r} \sigma_i u_i u_i^\T}_{\eqqcolon M_{r}} + \underbrace{\sum_{i=r+1}^{n} \sigma_i u_i u_i^\T}_{\eqqcolon M_r'}.
\end{equation}
If $\hat M_r$ is a rank-$r$ approximation to $\Mtruth$, the approximation error can be measured by $\|\hat M_r - \Mtruth\|_{\fro}$. It is well-known that the best rank-$r$ approximation in this sense is exactly $M_r$, and the optimal error is thus $\|M_r'\|_{\fro}$. By picking a larger $r$, one has a smaller approximation error $\|M_r'\|_{\fro}$, but a higher memory footprint for the low-rank approximation $M_r$ whose condition number also grows with $r$.

For simplicity, we consider the Gaussian design (cf.~Lemma~\ref{lem:gaussian_design}) in this subsection, which is less general than the RIP. The following theorem demonstrates that, as long as the sample size satisfies $m\gtrsim nr_\star^2\cdot\mathsf{poly}(\kappa)$, \myalg automatically adapts to the available sample size and produces a near-optimal rank-$r_{\star}$ approximation to $\Mtruth$ in spite of overparameterization.
 
\begin{theorem}
\label{thm:approx}
Assume that $\Mtruth$ is given in \eqref{eqn:approx-low-rank} and the sensing operator $\opA$ follows the Gaussian design with $m\ge Cnr_\star^2\kappa^{C}$, where $\kappa = \sigma_1/\sigma_{r_{\star}}$ is the condition number of $M_{r_\star}$. In addition, assume $\|M_{r_{\star}}'\|\le c_\sigma\kappa^{-C_\sigma}\|\Mtruth\|$ and $\|M_{r_{\star}}'\|_{\fro} \le c_\sigma\kappa^{-C_\sigma}\sqrt{\frac mn}\|\Mtruth\|$. Then the following holds with high probability (with respect to the realization of the random initialization $G$ and the sensing operator $\opA$). Suppose  Assumption \ref{assumption:param} holds for $M_{r_\star} = X_{\star} X_{\star}^{\top}$. Given a prescribed accuracy target $\epsilon\in(0,1)$, suppose further that $\alpha\le\epsilon^3\|\Xtruth\|$. there exist universal constants $\Ccvg > 0$, $C_{\ref{thm:approx}}>0$,  such that for some $T\le\Tcvg \coloneqq \frac{\Ccvg}{\eta}\log\frac{\|\Xtruth\|}{\alpha}$, we have
\begin{align*}
  \|X_T X_T^\T - \Mtruth\|_{\fro} &\le \max\left(
    \varepsilon\|\Mtruth\|, 
    ~C_{\ref{thm:approx}} \kappa^4 \|M_{r_\star}'\|_{\fro} \right).
\end{align*}
Here, $C>0, C_\sigma>0$ are some sufficiently large universal constants, and $c_\sigma>0$ is some sufficiently small universal constant.
\end{theorem}
\begin{remark}
Theorem~\ref{thm:approx} also holds in the matrix factorization setting, i.e., when $\opA$ is the identity operator.
\end{remark}

Theorem~\ref{thm:approx} suggests that as long as $\Mtruth$ is well approximated by a low-rank matrix, by setting the optimization error $\varepsilon$ sufficiently small, i.e., $\varepsilon \|\Mtruth\| \asymp \kappa^4  \|M_{r_\star}'\|_{\fro} $,   \myalg finds a solution that satisfies
\begin{align}
   \|X_T X_T^\T - \Mtruth\|_{\fro} \lesssim \kappa^4  \|M_{r_\star}'\|_{\fro} 
\end{align}
in no more than 
$ \log\kappa \cdot \log (\kappa n) + \log\left( \frac{ \|\Mtruth\|} { \kappa^4  \|M_{r_\star}'\|_{\fro}  } \right)$ iterations, which again only depend on the problem parameters logarithmically. This suggests that  if the residual $M_{r_\star}'$ is small, \myalg produces an approximate solution to the best rank-$r_{\star}$ approximation problem with near-optimal error, up to a multiplicative factor depending only on $\kappa$, without knowing the rank $r_{\star}$ a priori. To our best knowledge, this is the first near-optimal theoretical guarantee for approximate low-rank matrix sensing using gradient-based methods.


\section{Analysis}
\label{sec:analysis}
In this section, we present the main steps for proving Theorem~\ref{thm:main} and Theorem~\ref{thm:zero}. The proofs of Theorem~\ref{thm:noisy} and Theorem~\ref{thm:approx} will follow the same ideas with minor modification. The detailed proofs are collected in the appendix. All of our statements will be conditioned 
on the following high probability event regarding the initialization matrix $G$: 
\begin{equation}\label{asp:init}
  \Event=\{\|G\|\le C_G\}\cap\{\smin (\Uspec^\T G) \geq (2n)^{-C_G}\}, 
\end{equation}
where $\Uspec \in \mathbb{R}^{n \times r_\star}$  is an orthonormal basis of the eigenspace 
associated with the $r_\star$ largest eigenvalues of $\opAA(\Mtruth)$, and
 $C_G > 0$ is some sufficiently large universal constant. 
It is a standard result in random matrix theory 
that $\Event$ happens with high probability, as verified by the following lemma.
\begin{lemma}\label{lem:high-prob-event}
  With respect to the randomness in $G$, the event $\Event$ happens with probability at least $1-(cn)^{-C_G(r-r_\star+1)/2}-2\exp(-cn)$, 
  where $c > 0$ is some universal constant.
\end{lemma}
\begin{proof}
See Appendix~\ref{sec:pre-random-matrix}.
\end{proof}

\subsection{Preliminaries: decomposition of the iterates}
Before embarking on the main proof, we present a useful decomposition (cf.~\eqref{eqn:x-decomp-final}) of the iterate $\myX_t$ 
into a signal term, a misalignment error term, and an overparametrization error term. 
Choose some matrix $\UperpTruth \in \mathbb{R}^{n \times (n - r_\star)}$ such that $[\Utruth, \UperpTruth]$ is orthonormal. 
Then we can define 
$$\Signal_t \defeq \Utruth^\T \myX_t \in \mathbb{R}^{r_{\star}\times r},\quad \mbox{and} \quad  
\Noise_t \defeq \UperpTruth^\T \myX_t \in \mathbb{R}^{(n-r_{\star})\times r}.$$
Let  the SVD of $\Signal_t$ be
$$ \Signal_t =  \SVDU_t \SVDSigma_t \SVDV_t^\T,$$ 
where 
$\SVDU_t\in\mathbb{R}^{r_{\star}\times r_{\star}}$, $\SVDSigma_t \in\mathbb{R}^{r_{\star}\times r_{\star}}$, and $\SVDV_t\in \mathbb{R}^{r\times r_\star}$. Similar to $\UperpTruth$, we define the orthogonal complement of $\SVDV_t$ as 
$\Vperp{t}\in \mathbb{R}^{r\times (r-r_\star)}$. 
When $r=r_\star$ we simply set $\Vperp{t}=0$.

We are now ready to present the main decomposition of $\myX_t$, which we use 
repeatedly in later analysis. This decomposition is inspired by \cite{stoger2021small}. A similar decomposition also appeared in \cite{ma2022global}.

\begin{proposition}\label{prop:decomposition}
The following decomposition holds:
\begin{equation}\label{eqn:x-decomp-final}
  \myX_t 
  =\underbrace{  \Utruth\EssS_t \SVDV_t^\T  }_{\mathsf{signal}}+ 
  \underbrace{  \UperpTruth\Misalign_t \SVDV_t^\T  }_{\mathsf{misalignment}}  + \hspace{-0.05in}
   \underbrace{ \UperpTruth\Overpar_t \Vperp{t}^\T }_{\mathsf{overparametrization}},
\end{equation}
where
\begin{equation}\label{eq:imp_not}
\EssS_t \defeq \Signal_t \SVDV_t \in \reals^{r_\star \times r_\star},\quad \Misalign_t \defeq \Noise_t \SVDV_t \in \reals^{(n-r_\star) \times r_\star} ,\quad
\mbox{and} \quad \Overpar_t \defeq \Noise_t \Vperp{t} \in \reals^{(n-r_\star) \times (r-r_\star)}.
\end{equation}
\end{proposition}
\begin{proof}
See Appendix~\ref{sec:prop-decomposition}.
\end{proof}

\noindent Several remarks on the decomposition are in order.
\begin{itemize}
\item First, since $\Vperp{t}$ spans the obsolete subspace arising from overparameterization, 
$\Overpar_t$ naturally represents the error incurred by overparameterization; in particular, 
in the well-specified case (i.e., $r = r_\star$), one has zero overparameterization error, i.e., $\Overpar_t = 0$. 

\item Second, apart from the rotation matrix $\SVDV_t$, $\EssS_t$ documents the projection of the iterates $\myX_t$ onto the signal space $\Utruth$. 
Similarly, $\Misalign_t$ characterizes the misalignment of the iterates with the signal subspace $\Utruth$. It is easy to observe that in order for $\myX_t\myX_t^\T \approx M_\star$, one must have 
$\EssS_t \EssS_t^\top \approx \SigmaTruth^2$, and $\Misalign_t \approx 0$.  

\item Last but not least, the extra rotation induced by $\SVDV_t $ is extremely useful in making the signal/misalignment terms rationally invariant. To see this, suppose that we rotate the current iterate by $X_t\mapsto X_t Q$ with some rotational matrix $Q \in\mathcal{O}_r$, 
then $\Signal_t \mapsto \Signal_t Q$ but $\EssS_t$ remains unchanged, and similarly for $\Misalign_t$.

\end{itemize} 

\subsection{Proof roadmap}
     
Our analysis breaks into a few phases that characterize the dynamics of the key terms in the above decomposition, which we provide a roadmap to facilitate understanding. Denote 
\begin{equation*}
  \Cmax\defeq\begin{cases}
    4\Ccvg, & r>r_\star,\\
    \infty, & r=r_\star,
  \end{cases} \qquad
\mbox{and}  \qquad \Tmax\defeq\frac{\Cmax}{\eta}\log(\|\Xtruth\|/\alpha),
\end{equation*}
where  $\Tmax$ represents the largest index of the iterates that we maintain error control. The analysis boils down to the following phases, indicated by time points $t_1, t_2, t_3, t_4$ that  satisfy
\begin{align*} 
 t_1 \leq \Tcvg/16,  \quad  t_1 \leq t_2 \leq t_1 + \Tcvg / 16, \quad t_2 \leq t_3 \leq t_2 + \Tcvg / 16,  \quad t_3 \leq t_4 \leq t_3 + \Tcvg / 16. 
\end{align*}
\begin{itemize}
\item {\em Phase I: approximate power iterations.} In the initial phase, \myalg behaves similarly to \GD, which is shown in \cite{stoger2021small} to approximate the power method in the first few iterations up to $t_1$. After this phase, namely for $t \in [t_1, \Tmax]$, although the signal strength is still quite small, it begins to be aligned  with the ground truth with the overparameterization error kept relatively small.

\item {\em Phase II: exponential amplification of the signal.} In this phase, \myalg behaves somewhat as a mixture of \GD and \ScaledGD with a proper choice of the damping parameter $\lambda \asymp  \smin^2(\Xtruth)$, which ensures
  the signal strength first grows exponentially fast to reach a constant level no later than $t_2$, and then reaches the desired level no later than $t_3$, i.e., $\EssS_{t}\EssS_{t}^\T \approx \SigmaTruth^2$. 
\item {\em Phase III: local linear convergence.} At the last phase, \myalg behaves similarly to  \ScaledGD, which converges linearly at a rate independent of the condition number. Specifically, for $t\in [t_3, \Tmax]$, the reconstruction error $\|\myX_t\myX_t^\T - \Mtruth\|_{\fro}$ converges at a linear rate up to some small overparameterization error, until reaching the desired accuracy for any $t \in [t_4, \Tmax]$.  
\end{itemize}

\subsection{Phase I: approximate power iterations}
It has been observed in~\citet{stoger2021small} that when initialized at a 
small scaled random matrix, the first few iterations of \GD mimic the power iterations on 
the matrix $\opAA(\Mtruth)$. When it comes to \myalg, since the initialization size $\alpha$ 
is chosen to be much smaller than the damping parameter $\lambda$, the preconditioner $(\myX_t^\top \myX_t + \lambda I)^{-1}$ 
behaves like $(\lambda I)^{-1}$ in the beginning. This renders \myalg akin to gradient descent in the initial phase. 
As a result, we also expect the first few iterations of \myalg to be similar to the power iterations, i.e., 
\[
  \myX_t\approx\left(I+\frac{\eta}{\lambda}\opAA(\Mtruth)\right)^{t}\myX_0,\qquad \text{when }t \text{ is small}. 
\]
Such proximity between \myalg and power iterations can indeed be justified 
in the beginning period, which allows us to deduce the following nice properties {\em after} 
the initial iterates of \myalg.  
\begin{lemma}\label{lem:p1.5}
  Under the same setting as Theorem \ref{thm:main}, there exists an iteration number $t_1: t_1 \leq \Tcvg / 16$ such that
  \begin{equation}\label{eqn:p1.5-0}
    \smin(\EssS_{t_1}) \geq \alpha^2/\|\Xtruth\|, 
  \end{equation}
  and that, for any $t\in[t_1, \Tmax]$, $\EssS_t$ is invertible and one has 
  \begin{subequations}\label{subeq:condition-t-1}
    \begin{align}
      \|\Overpar_{t}\|&\le (C_{\ref{lem:p1.5}.b}\kappa n)^{-C_{\ref{lem:p1.5}.b}}\|\Xtruth\|
      \smin \big((\SigmaTruth^2+\lambda I)^{-1/2}\EssS_t \big), 
      \label{eqn:p1.5-1}\\
      \|\Overpar_{t}\|&\le \left(1+\frac{\eta}{12\Cmax\kappa}\right)^{t-t_1}
      \alpha^{5/6}\|\Xtruth\|^{1/6},
      \label{eqn:p1.5-2}\\
      \|\Misalign_{t}\EssS_{t}^{-1}\SigmaTruth\| 
      &\le c_{\ref{lem:p1.5}}\kappa^{- C_\delta /2}\|\Xtruth\|, 
      \label{eqn:p1.5-3}\\
      \|\EssS_{t}\|& \le C_{\ref{lem:p1.5}.a}\kappa^3 \|\Xtruth\|,
      \label{eqn:p1.5-4}
    \end{align} 
  \end{subequations}
   where $C_{\ref{lem:p1.5}.a}$, $C_{\ref{lem:p1.5}.b}$, $c_{\ref{lem:p1.5}}$ are some positive constants 
  satisfying $C_{\ref{lem:p1.5}.a}\lesssim c_\lambda^{-1/2}$, $c_{\ref{lem:p1.5}}\lesssim c_\delta/c_\lambda$, 
  and  $C_{\ref{lem:p1.5}.b}$ can be made arbitrarily large by increasing $C_\alpha$.
\end{lemma}
\begin{proof}
See Appendix~\ref{sec:proof-p1}. 
\end{proof}

\begin{remark}
Let us record two immediate consequences of \eqref{subeq:condition-t-1}, which sometimes are more convenient for later analysis. From \eqref{eqn:p1.5-1}, we may deduce
  \begin{align}
    \label{eqn:p1.5-1-var}
    \|\Overpar_t\|
    &\le (C_{\ref{lem:p1.5}.b}\kappa n)^{-C_{\ref{lem:p1.5}.b}}\|\Xtruth\|
    \smin(\SigmaTruth^2+\lambda I)^{-1/2}\smin(\EssS_t)
    \nonumber\\
    &\le \kappa(C_{\ref{lem:p1.5}.b}\kappa n)^{-C_{\ref{lem:p1.5}.b}}\smin(\EssS_t)
    \nonumber\\
    &\le (C_{\ref{lem:p1.5}.b}'\kappa n)^{-C_{\ref{lem:p1.5}.b}'}\smin(\EssS_t),
  \end{align}
  where $C_{\ref{lem:p1.5}.b}'=C_{\ref{lem:p1.5}.b}/2$, provided $C_{\ref{lem:p1.5}.b} > 4$. 
  It is clear that $C_{\ref{lem:p1.5}.b}'$ can also be made arbitrarily large by enlarging $C_\alpha$. 
  Similarly, from \eqref{eqn:p1.5-2}, we may deduce
  \begin{align}
    \|\Overpar_t\|\le\left(1+\frac{\eta}{12\Cmax\kappa}\right)^{t-t_1}
    \alpha^{5/6}\|\Xtruth\|^{1/6}
    &\le \left(1+\frac{\eta}{12\Cmax\kappa}\right)^{\frac{\Cmax}{\eta}\log(\|\Xtruth\|/\alpha)}
    \alpha^{5/6}\|\Xtruth\|^{1/6}
    \nonumber\\
&\le(\|\Xtruth\|/\alpha)^{1/12}\alpha^{5/6}\|\Xtruth\|^{1/6}
    =\alpha^{3/4}\|\Xtruth\|^{1/4}.
    \label{eqn:p1.5-2-var}
  \end{align}
\end{remark}

Lemma~\ref{lem:p1.5} ensures the iterates of \myalg maintain several desired properties after iteration $t_1$, as summarized in \eqref{subeq:condition-t-1}. In particular, for any $t\in[t_1, \Tmax]$: (i) the overparameterization error $\|\Overpar_{t}\|$ remains small relatively to the signal strength measured in terms of the scaled minimum singular value $\smin\big((\SigmaTruth^2+\lambda I)^{-1/2}\EssS_t\big)$, and remains bounded with respect to the size of the initialization $\alpha$ (cf.~\eqref{eqn:p1.5-1} and \eqref{eqn:p1.5-2} and their consequences \eqref{eqn:p1.5-1-var} and \eqref{eqn:p1.5-2-var}); 
(ii) the scaled misalignment-to-signal ratio remains bounded, suggesting the iterates remain aligned with the ground truth signal subspace $\Utruth$ (cf.~\eqref{eqn:p1.5-3}); (iii) the size of the signal component $\EssS_{t}$ remains bounded (cf.~\eqref{eqn:p1.5-4}). 
These properties play an important role in the follow-up analysis.

\begin{remark}
It is worth noting that, 
the scaled minimum singular value $\smin((\SigmaTruth^2+\lambda I)^{-1/2}\EssS_{t})$ plays a key role in our analysis, 
which is in sharp contrast to the use of the vanilla minimum singular value $\smin(\EssS_t)$ in the analysis of gradient descent \citep{stoger2021small}. 
  This new measure of signal strength is inspired by the scaled distance for \ScaledGD 
  introduced in \cite{tong2021accelerating, tong2022scaling}, which carefully takes the preconditioner design into consideration. Similarly, the metrics $\|\Misalign_{t}\EssS_{t}^{-1}\SigmaTruth\| $ in (\ref{eqn:p1.5-3}) and $  \big\|\SigmaTruth^{-1}(\EssS_{t+1}\EssS_{t+1}^\T-\SigmaTruth^2)\SigmaTruth^{-1} \big\|$ (to be seen momentarily) are also scaled for similar considerations to unveil the fast convergence (almost) independent of the condition number. 
  \end{remark}

\subsection{Phase II: exponential amplification of the signal}
By the end of Phase I, the signal strength is still quite small (cf.~\eqref{eqn:p1.5-0}), which is 
far from the desired level. Fortunately, the properties established in Lemma~\ref{lem:p1.5} allow us to establish an exponential amplification of the signal term $\EssS_{t}$ thereafter, which can be further divided into two stages.
\begin{enumerate}
\item In the first stage, the signal is boosted to a constant level, 
i.e., $\EssS_{t}\EssS_{t}^\T\succeq \frac{1}{10}\SigmaTruth^2$;
\item In the second stage, the signal grows further to the desired level, i.e., $\EssS_{t}\EssS_{t}^\T \approx \SigmaTruth^2$. 
\end{enumerate}

We start with the first stage, which again uses $\smin\big((\SigmaTruth^2+\lambda I)^{-1/2}\EssS_{t} \big)$ as a 
measure of signal strength in the following lemma. 
  
\begin{lemma}\label{lem:p2}
  For any $t$ such that \eqref{subeq:condition-t-1} holds, we have
  \begin{equation*}
    \smin\big((\SigmaTruth^2+\lambda I)^{-1/2}\EssS_{t+1} \big)
    \ge (1-2\eta)\smin\big((\SigmaTruth^2+\lambda I)^{-1/2}\EssS_{t}\big).
  \end{equation*}
  Moreover, if $\smin\big((\SigmaTruth^2+\lambda I)^{-1/2}\EssS_t\big)\le 1/3$, then
  \begin{equation*}
    \smin\big((\SigmaTruth^2+\lambda I)^{-1/2}\EssS_{t+1}\big)
    \ge \left(1+\frac18\eta\right)\smin \big((\SigmaTruth^2+\lambda I)^{-1/2}\EssS_{t} \big).
  \end{equation*}
\end{lemma}
\begin{proof}
See Appendix~\ref{pf:lem:p2}.
\end{proof}

The second half of Lemma~\ref{lem:p2} uncovers the exponential growth of  the signal strength  $\smin\big((\SigmaTruth^2+\lambda I)^{-1/2}\EssS_{t} \big)$
until a constant level after several iterations, which resembles the exponential growth of the signal strength in \GD \citep{stoger2021small}. This is formally established in the following corollary.

\begin{corollary}\label{cor:p2}
  There exists an iteration number $t_2: t_1 \leq t_2 \leq t_1 + \Tcvg / 16$ such that for all 
  $t \in[t_2,\Tmax]$, we have
  \begin{equation}\label{eqn:cor-p2}
    \EssS_{t}\EssS_{t}^\T\succeq \frac{1}{10}\SigmaTruth^2.
  \end{equation}
\end{corollary}

\begin{proof}
 See Appendix~\ref{sec:proof-cor-1}.
 \end{proof}

  We next aim to show that $\EssS_t\EssS_t^\T \approx \SigmaTruth^2$ after the signal strength is above the constant level. To this end, the behavior of \myalg becomes closer to that of \ScaledGD, and it turns out to be easier to work with $\big\|\SigmaTruth^{-1}(\EssS_{t}\EssS_{t}^\T-\SigmaTruth^2)\SigmaTruth^{-1} \big\|$ 
as a measure of the scaled recovery error of the signal component. We establish the approximate exponential shrinkage of this measure in the following lemma.  

\begin{lemma}\label{lem:p2.5}
  For all $t\in[t_2,\Tmax]$ with $t_2$ given in Corollary \ref{cor:p2}, one has 
  \begin{equation}
    \big\|\SigmaTruth^{-1}(\EssS_{t+1}\EssS_{t+1}^\T-\SigmaTruth^2)\SigmaTruth^{-1} \big\|
    \le \left(1-\eta\right)\big\|\SigmaTruth^{-1}(\EssS_{t}\EssS_{t}^\T-\SigmaTruth^2)\SigmaTruth^{-1} \big\|
    +\frac{1}{100}\eta.
  \end{equation}
\end{lemma}

\begin{proof} 
See Appendix~\ref{pf:lem:p2.5}.
\end{proof}

With the help of Lemma~\ref{lem:p2.5}, it is straightforward to establish the desired approximate 
recovery guarantee of the signal component, i.e., $\EssS_{t}\EssS_{t}^\T \approx \SigmaTruth^2$.
\begin{corollary}\label{cor:p2.5}
  There exists an iteration number $t_3: t_2 \leq t_3 \leq t_2 + \Tcvg / 16$ such that  
  for any $t \in[ t_3, \Tmax]$,  one has
  \begin{equation}\label{eqn:cor:p2.5}
    \frac{9}{10}\SigmaTruth^2\preceq\EssS_{t}\EssS_{t}^\T\preceq\frac{11}{10}\SigmaTruth^2.
  \end{equation}
\end{corollary}

\begin{proof} 
See Appendix~\ref{sec:pf:cor:p2.5}.
\end{proof}

\subsection{Phase III: local convergence}
Corollary~\ref{cor:p2.5} tells us that after iteration $t_3$, we enter a local region 
in which $\EssS_t \EssS_t^\top $ is close to the ground truth $\SigmaTruth^2$.
In this local region, the behavior of \myalg becomes closer to that of \ScaledGD analyzed in \citet{tong2021accelerating}. We turn attention to the 
reconstruction error $\|\myX_t\myX_t^\T - \Mtruth\|_{\fro}$ that measures the generalization performance, and show it converges at a linear rate independent of the condition number up to some small overparameterization error. 
\begin{lemma}\label{lem:p3}
There exists some universal constant $c_{\ref{lem:p3}}>0$ such that
for any $t: t_3\le t\le\Tmax$, we have
\begin{equation}\label{eqn:lem:p3-primal}
  \|\myX_{t} \myX_{t}^\T-\Mtruth\|_{\fro}
  \le (1-c_{\ref{lem:p3}}\eta)^{t-t_3}\sqrt{r_\star}\|\Mtruth\|
  +8c_{\ref{lem:p3}}^{-1}\|\Mtruth\|
  \max_{t_3\le\tau\le t}\left(\frac{\|\Overpar_{\tau}\|}{\|\Xtruth\|}\right)^{1/2}.
\end{equation}
In particular, there exists an iteration number 
$t_4: t_3 \leq t_4 \leq t_3 + \Tcvg / 16$ such that 
for any $t\in[t_4, \Tmax]$, we have
\begin{equation}\label{eqn:lem:p3-final}
  \|\myX_t\myX_t^\T - \Mtruth\|_{\fro}
  \le\alpha^{1/3}\|\Xtruth\|^{5/3}\le\epsilon\|\Mtruth\|.
\end{equation}
Here, $\epsilon$ and $\alpha$ are as stated in Theorem \ref{thm:main}.
\end{lemma}

\begin{proof} 
See Appendix~\ref{sec:pf:lem:p3}.
\end{proof}

\subsection{Proofs of main theorems}
Now we are ready to collect the results in the preceding sections to prove our main 
results, i.e., Theorem~\ref{thm:main} and Theorem~\ref{thm:zero}. The proofs of Theorem~\ref{thm:noisy} and Theorem~\ref{thm:approx} follows from similar ideas but with additional technicality, thus are postponed to Appendix~\ref{sec:proof-noisy}.

We start with proving Theorem~\ref{thm:main}. 
By Lemma~\ref{lem:p1.5}, Corollary~\ref{cor:p2}, Corollary~\ref{cor:p2.5} 
and Lemma~\ref{lem:p3}, the final $t_4$ given by Lemma~\ref{lem:p3}
is no more than $4\times\Tcvg/16\le\Tcvg/2$, 
thus~\eqref{eqn:lem:p3-final} holds for all $t\in[\Tcvg/2,\Tmax]$, 
in particular, for some $T\le \Tcvg$, as claimed.

Now we consider Theorem~\ref{thm:zero}. 
In case that $r=r_\star$, 
it follows from definition that $\Overpar_t =0$ vanishes for all $t$.
It follows from Lemma~\ref{lem:p3}, in particular from \eqref{eqn:lem:p3-primal}, that
\begin{equation*}
  \|\myX_t\myX_t^\T - \Mtruth\|_{\fro}\le (1-c_{\ref{lem:p3}}\eta)^{t-t_3}\sqrt{r_\star}\|\Mtruth\|,
\end{equation*}
for any $t\ge t_3$ (recall that $\Tmax=\infty$ by definition when $r=r_\star$).
Note that $(1-c_{\ref{lem:p3}}\eta)^t\sqrt{r_\star}\le(1-c_{\ref{lem:p3}}\eta)^{t-T+t_3}$  
if $T-t_3\ge 4\log(r_\star)/(c_{\ref{lem:p3}}\eta)$ given that $\eta\le c_\eta$ is sufficiently small. 
Thus for any $t\ge T$ we have
\begin{equation*}
  \|\myX_t\myX_t^\T - \Mtruth\|_{\fro}\le (1-c_{\ref{lem:p3}}\eta)^{t-T}\|\Mtruth\|.
\end{equation*}
It is clear that one may choose such $T$ which also satisfies
$T\le t_3+8/(c_{\ref{lem:p3}}\eta)\le t_3+\Tcvg/16$. 
We have already shown in the proof of Theorem~\ref{thm:main} that
$t_3\le 4\times\Tcvg/16\le\Tcvg/4$, thus $T\le\Tcvg$ as desired.

\paragraph{Early stopping.} In the overparameterized setting, our theory guarantees the reconstruction error to be small until some iteration $\Tmax$. This is consistent with the phenomenon known as \emph{early stopping} in prior works of learning with overparameterized models \citep{stoger2021small, li2018algorithmic}.  
Given the form of \eqref{eqn:p1.5-2}, one may wonder if the early stopping needs to be precisely controlled, if $\|\widetilde{O}_t\|$ could grow excessively.
Fortunately, this is not the case, as the following proposition -- proved in Appendix~\ref{sec:pf:lem:p3} -- demonstrates.
 
 \begin{proposition}
\label{prop:Ot}
Under the same setting as Theorem~\ref{thm:main}, we have
\[
\|\Overpar_{t}\| \le \alpha^{7/10} \|X_\star\|^{3/10}, 
\quad \forall t \le \left(\frac{\|X_\star\|}{\alpha}\right)^{3/10}.
\]
\end{proposition}
As we pick a very small $\alpha$, this means one does not need to do early stopping for all practical purposes.

\section{Numerical experiments}
\label{sec:numerical}

In this section, we conduct numerical experiments to demonstrate the 
efficacy of \myalg for solving 
overparameterized low-rank matrix sensing. 
We set the ground truth matrix $\Xtruth = \Utruth\SigmaTruth\in\reals^{n\times r_\star}$ 
where $\Utruth\in\reals^{n\times r_\star}$ is a random orthogonal matrix 
and $\SigmaTruth\in\reals^{r_\star \times r_\star}$ is a diagonal matrix whose condition number is set to be $\kappa$. 
We set $n=150$ and $r_\star=3$, and use random Gaussian measurements with $m=10n r_\star$. 
The overparameterization rank $r$ is set to be $5$ unless otherwise specified.

Throughout our experiments, to choose $\lambda$, we estimate $\sigma_{\min}(X_\star)$ using a simple rule of thumb. Let $\hat\sigma_1 \ge \hat\sigma_2 \ge \cdots \ge \hat\sigma_n$ be the singular values of $\mathcal{A}^*(y)$. Let $i_0$ be the smallest number such that
\[
\sum_{i \le i_0} \hat\sigma_i \ge 0.95 \sum_{i\le n}\hat \sigma_i.
\] 
Then we estimate $\hat\sigma_{\min}^2(X_\star) = \hat\sigma_{i_0}$. This heuristic also applies to noisy or approximately low-rank matrices, thanks to our Theorem~\ref{thm:noisy} and Theorem~\ref{thm:approx}. In practice, the $0.95$ threshold can be tuned towards a desired accuracy level.

\paragraph{Comparison with overparameterized \GD. }
We run \myalg and \GD with random initialization and compare their convergence speeds under different condition numbers $\kappa$ of the ground truth $\Xtruth$; 
the result is depicted in Figure~\ref{fig:intro}. 
Even for a moderate range of $\kappa$, \GD slows down significantly while the convergence speed of \myalg remains almost the same with a almost negligible initial phase, which is consistent with our theory. The advantage of \myalg enlarges as $\kappa$ increase, and is already more than 10x times faster than \GD when $\kappa=7$.

\paragraph{Effect of initialization size.} 
We study the effect of the initialization scale $\alpha$ on the reconstruction accuracy of \myalg.

We fix the learning rate $\eta$ to be a constant and vary the initialization scale. We run \myalg until it converges.\footnote{More precisely, in accordance with our theory which requires early stopping, we stop the algorithm once we detected that the training error no longer decreases significantly for a long time (e.g., $100$ iterations).} The resulting reconstruction errors and their corresponding initialization scales are plotted in Figure~\ref{fig:alpha}. 
It can be inferred that the reconstruction error increases with respect to $\alpha$,
which is consistent with our theory. 

\begin{figure}[ht]
    \centering
    \includegraphics[width=0.5\textwidth]{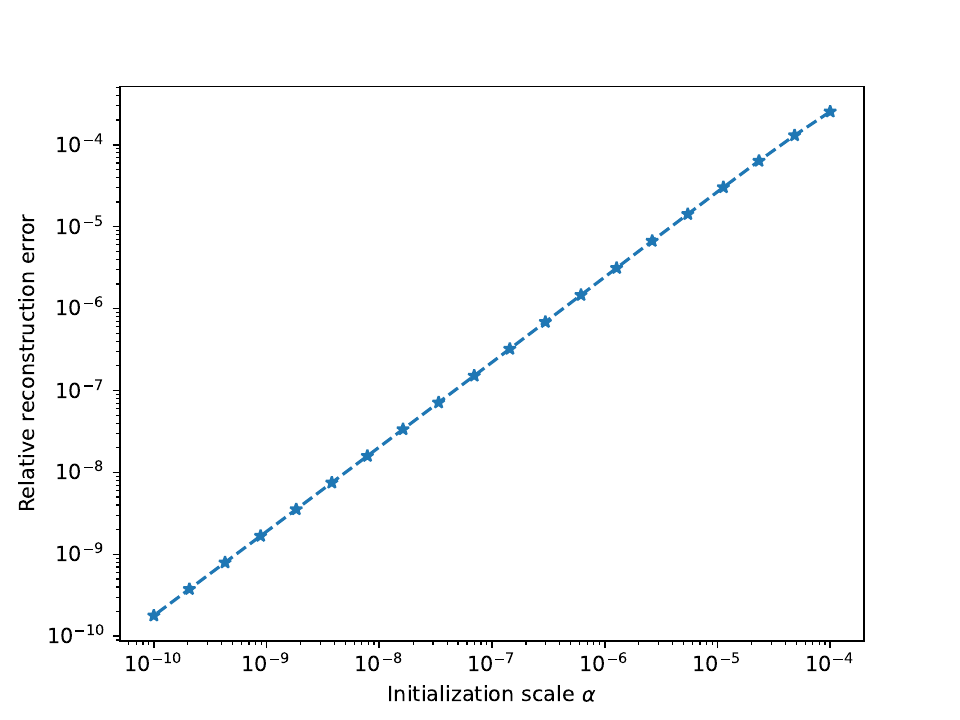}
    \caption{Relative reconstruction error versus initialization scale $\alpha$. The slope of the dashed line is  approximately $1$. }
    \label{fig:alpha}
\end{figure}

\paragraph{Comparison with \citet{zhang2021preconditioned}.}
We compare \myalg with the algorithm $\mathsf{PrecGD}$ proposed in \citet{zhang2021preconditioned}, which also has a $\kappa$-independent convergence rate assuming a sufficiently good initialization using spectral initialization. However, $\mathsf{PrecGD}$ requires RIP of rank $r$, thus demanding $O(nr^2)$ many samples instead of $O(nr_\star^2)$ as in \GD and \myalg. This can be troublesome for larger $r$. 
To demonstrate this point, we run \myalg and $\mathsf{PrecGD}$ with different overparameterization rank $r$ while fixing all other parameters. The results are shown in Figure~\ref{fig:r}. It can be seen that the convergence rate of $\mathsf{PrecGD}$ and \myalg are almost the same when the rank is exactly specified ($r=r_\star=3$), though \myalg requires a few more iterations for the initial phases\footnote{Usually this has no significant implication on the computational cost: the amount of computations required in the initial phases for \myalg is approximately the same as that required by the spectral initialization for $\mathsf{PrecGD}$.}. 
When $r$ goes higher, \myalg is almost unaffected, while $\mathsf{PrecGD}$ suffers from a significant drop in the convergence rate and even breaks down with a moderate overparameterization $r=20$.
\begin{figure}[t]
    \centering
    \includegraphics[height=2.75in,width=0.5\textwidth]{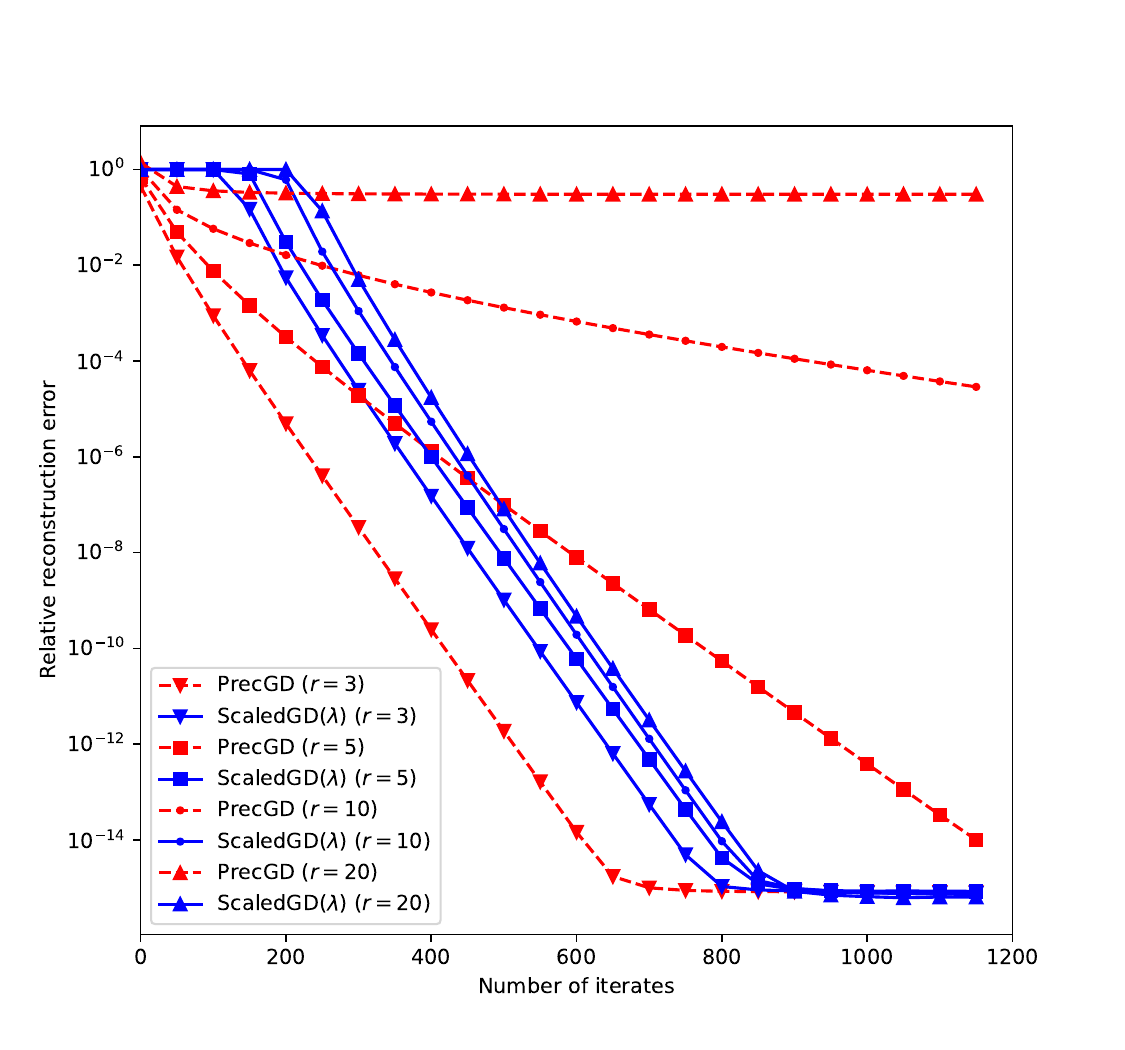}
    \caption{Relative reconstruction error versus the number of iterates with different overparameterization rank $r$ for \myalg and $\mathsf{PrecGD}$.}
    \label{fig:r} 
\end{figure}

\paragraph{Noisy setting.}
Though our theoretical results here are formulated in the noiseless setting, empirical evidence indicates our algorithm \myalg also works in the noisy setting. Modifying the equation \eqref{eq:measurements} for noiseless measurements, we assume the noisy measurements $y_i=\langle A_i, M\rangle + \xi_i$ where $\xi_i\sim\mathcal{N}(0,\sigma^2)$ are i.i.d.~Gaussian noises. The minimax lower bound for the reconstruction error $\|X_t X_t^\top - \Mtruth\|_{\fro}$ is denoted by $\mathcal{E}_{\mathsf{stat}}=\sigma\sqrt{nr_\star}$ \citep{candes2011tight}. 
We compare the reconstruction error of \myalg with $\mathcal{E}_{\mathsf{stat}}$ under different noise levels $\sigma$. 
The results are shown in Figure~\ref{fig:noisy}. 
It can be seen that the final error of \myalg matches the minimax optimal error $\mathcal{E}_{\mathsf{stat}}$ within a small multiplicative factor for all noise levels. 
\begin{figure}[ht]
    \centering
    \includegraphics[width=0.5\textwidth]{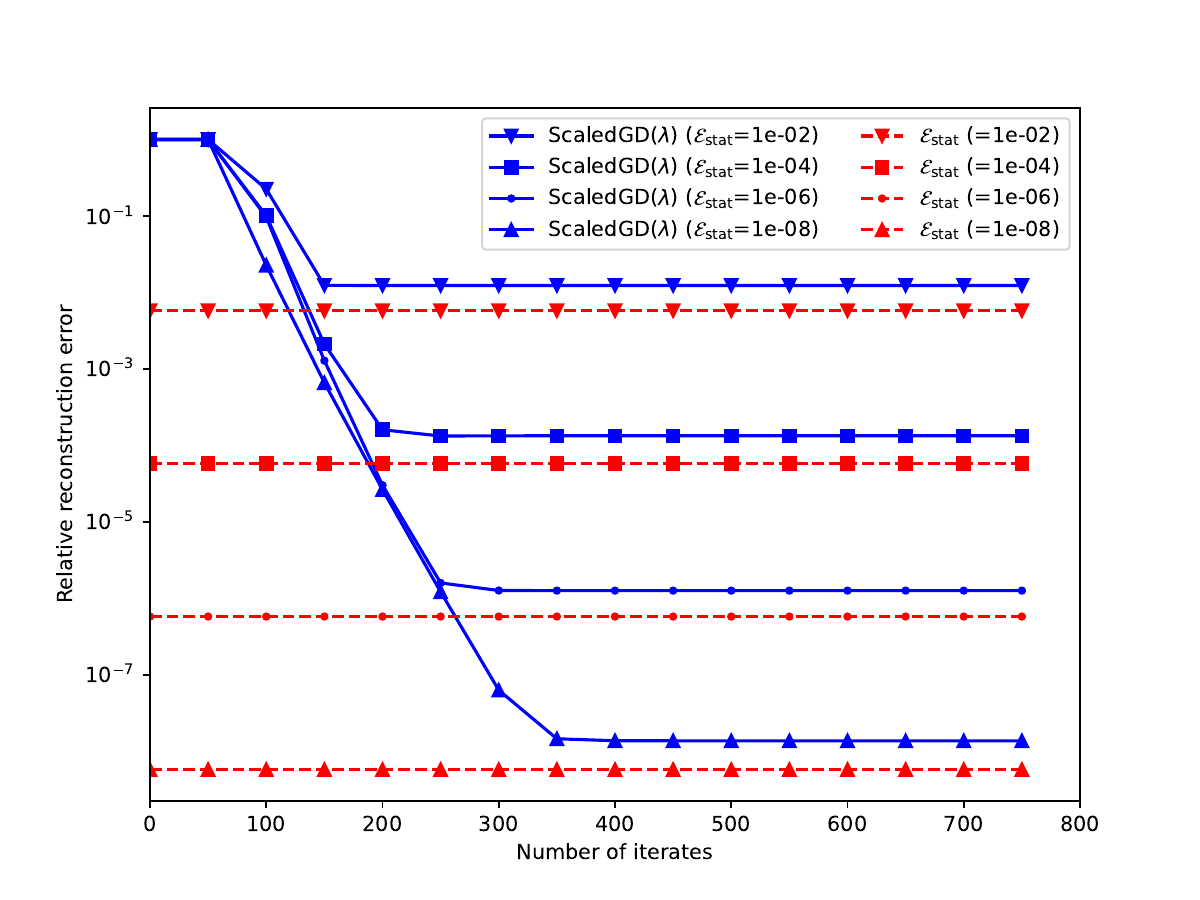}
    \caption{The relative reconstruction error of \myalg versus the number of iterates for \myalg in the noisy setting, where it is observed that the final error of \myalg approaches the minimax error.}
    \label{fig:noisy} 
\end{figure}


\section{Discussions}

This paper demonstrates that an appropriately preconditioned gradient descent method, called \myalg, guarantees an accelerated convergence to the ground truth low-rank matrix in overparameterized low-rank matrix sensing, when initialized from a sufficiently small random initialization. Furthermore, in the case of exact parameterization, our analysis guarantees the fast global convergence of \myalg from a small random initialization. Collectively, this complements and represents a major step forward from prior analyses of \ScaledGD \citep{tong2021accelerating} by allowing overparametrization and small random initialization for noisy and approximately low-rank settings. This works opens up a few exciting future directions that are worth further exploring.

\begin{itemize}

\item {\em Asymmetric case.} Our current analysis is confined to the recovery of low-rank positive semidefinite matrices, with only one factor matrix to be recovered. It remains to generalize this analysis to the recovery of general low-rank matrices with overparameterization.

\item {\em Robust setting.} Many applications encounter corrupted measurements that call for robust recovery algorithms that optimize nonsmooth functions such as the least absolute deviation loss. One such example is the scaled subgradient method \citep{tong2021low}, which is the nonsmooth counterpart of \ScaledGD robust to ill-conditioning, and it'll be interesting to study its performance under overparameterization. 
  
\item {\em Other overparameterized learning models.} Our work provides evidence on the power of preconditioning in accelerating the convergence without hurting generalization in overparameterized low-rank matrix sensing, which is one kind of overparameterized learning models. It will be greatly desirable to extend the insights developed herein to other overparameterized learning models, for example low-rank matrix optimization \citep{boumal2016smooth}, tensors \citep{tong2022scaling,dong2022fast}, and neural networks \citep{wang2021deep}.
  
\end{itemize}

We believe the analysis framework put forth in this paper can be extended to analyze these general issues, by leveraging similar error decompositions and tailoring the treatment to the corresponding measurement or data models, see an overview \cite{ma2024provably} and some recent works \cite{giampouras2024guarantees,diaz2025preconditioned} along this line after the initial version of this paper.

\section*{Acknowledgements}
 
The work of X. Xu and Y.~Chi is supported in part by Office of Naval Research under N00014-19-1-2404, by Air Force Office of Scientific Research under award number FA9550-25-1-0060, and by National Science Foundation under CCF-1901199, DMS-2134080  and ECCS-2126634. The work of C.~Ma is partially supported National Science Foundation via grant DMS-2311127 and DMS CAREER Award 2443867. Any opinions, findings, and conclusions or recommendations expressed in this material are those of the author(s) and do not necessarily reflect the views of the United States Air Force.

\bibliographystyle{apalike} 
\bibliography{bibfileNonconvexScaledGD,bibfileTensor,bibfileOverparam}



\appendix


\section{Preliminaries}
This section collects several preliminary results that are useful in later proofs.
In general, for a matrix $A$, we will denote by $\SVDU_A$ 
the first factor in its compact SVD $A=\SVDU_A\SVDSigma_A \SVDV_A^\T$, unless otherwise specified.
\subsection{Proof of Lemma \ref{lem:high-prob-event}}
\label{sec:pre-random-matrix}

It is a standard result in random matrix theory \citep{vershynin2010nonasym, rudelson2009smallest} that 
an $M\times N$ ($M\ge N$) random matrix  $G_0$ with i.i.d.~standard Gaussian entries satisfies
\begin{subequations}
\begin{align}
\Prob\left(\|G_0\|\le 4\big(\sqrt M + \sqrt N\big)\right)
&\ge 1 - \exp(-M/C),  \label{eq:rmt_1}
\\
\Prob\left(\smin(G_0)\ge \epsilon\big(\sqrt M - \sqrt{N-1}\big)\right)
&\ge 1 - (C\epsilon)^{M-N+1} - \exp(-M/C), \label{eq:rmt_2}
\end{align}
\end{subequations} 
for some universal constant $C>0$ and for any $\epsilon>0$. 
Applying \eqref{eq:rmt_1} to the random matrix $\sqrt{n}G$ 
which is an $n\times r$ random matrix with i.i.d.~standard Gaussian entries, 
we have 
\[\|G\|\le 4(\sqrt{n}+\sqrt{r})/\sqrt{n}\le 8 \]
with probability at least $1-\exp(-n/C)$. 

Turning to the bound on $\smin^{-1}(\Uspec^\T G)$, 
observe that $\sqrt n\Uspec^\T G$ is a $r_\star\times r$ 
random matrix with i.i.d.~standard Gaussian entries, 
thus applying \eqref{eq:rmt_2} to $\sqrt n\Uspec^\T G$ with 
$\epsilon=(2n)^{-C_G+1}$ yields 
$$\smin^{-1}(\Uspec^\T G)\le (2n)^{C_G-1}(\sqrt{r} - \sqrt{r_\star-1})^{-1}
\le (2n)^{C_G-1}(2\sqrt r) \le (2n)^{C_G}$$ 
with probability at least $1-(2n/C)^{-(C_G-1)(r-r_\star+1)}-\exp(-n/C)$. Here, the second inequality follows from 
$$ \frac{1}{\sqrt{r} - \sqrt{r_\star-1}} \leq \frac{1}{\sqrt{r} - \sqrt{r -1}} =\sqrt{r} + \sqrt{r -1}< 2\sqrt{r}. $$
Combining the above two bounds directly implies the desired probability bound 
if we choose $c=1/C$ 
and choose a large $C_G$ such that $C_G\ge 8$ and $C_G-1\ge C_G/2$.

\subsection{Proof of Proposition \ref{prop:decomposition}}
\label{sec:prop-decomposition}
 Using the definitions of $\Signal_t$ and $\Noise_t$, we have
\begin{align*}
X_t = (\Utruth\Utruth^{\top} + \UperpTruth\UperpTruth^{\top}) X_t & =  \Utruth\Signal_t + \UperpTruth \Noise_t\\
&= \Utruth\EssS_t\SVDV_t^\T + \UperpTruth \Noise_t(\SVDV_t\SVDV_t^\T + \Vperp{t}\Vperp{t}^\T)\\
&= \Utruth\EssS_t \SVDV_t^\T + 
  \UperpTruth\Misalign_t \SVDV_t^\T   + \UperpTruth\Overpar_t \Vperp{t}^\T,
\end{align*} 
where in the second line, we used the relation $\EssS_t = \Signal_t\SVDV_t =\SVDU_t \SVDSigma_t \SVDV_t^\T\SVDV_t =\SVDU_t \SVDSigma_t  $ and thus 
\begin{equation}\label{eq:signal_to_esssignal}
\Signal_t = \EssS_t\SVDV_t^\T.
\end{equation}

\subsection{Consequences of RIP}

The first result is a standard consequence of RIP, see, for example~\citet[Lemma 7.3]{stoger2021small}.

\begin{lemma}\label{lem:fro_rip}
Suppose that the linear map $\opA: \operatorname{Sym}_2(\reals^n)\to\reals^m$ satisfies Assumption~\ref{assumption:opA}.
Then we have
\begin{equation*}
  \|(\id-\opAA)(Z)\|\le\delta \|Z\|_{\fro}
\end{equation*}
for any $Z\in\operatorname{Sym}_2(\reals^n)$ with rank at most $r_\star$. Consequently, with $\hat\lambda_1 \ge \cdots \ge \hat\lambda_n$ denoting the eigenvalues of $\opAA(\Mtruth)$, it holds that
\[
|\hat\lambda_i - \sigma_i^2(X_\star)| \le \delta \sqrt{r_\star} \|X_\star\|^2.
\]
\end{lemma}

We need another straightforward consequence of RIP, given by the following lemma.

\begin{lemma}\label{lem:spectral_rip}
Under the same setting as Lemma~\ref{lem:fro_rip}, we have  
\begin{equation*}
  \|(\id-\opAA)(Z)\|\le 2\delta \sqrt{(r\lor r_\star)/r_\star}\|Z\|_{\fro}
  \le\frac{2(r\lor r_\star)\delta }{\sqrt{r_\star}}\|Z\|
\end{equation*}
for any $Z\in\operatorname{Sym}_2(\reals^n)$ with rank at most $r$.
\end{lemma}

\begin{proof}
Without loss of generality we may assume $r\ge r_\star$, 
thus $r \lor r_\star=r$. 
We claim that it is possible to decompose $Z=\sum_{i\le\lceil r/r_\star\rceil}Z_i$ 
where $Z_i\in\operatorname{Sym}_2(\reals^n)$, $\rank(Z_i)\le r_\star$ and $Z_iZ_j=0$ if $i\ne j$. 
To see why this is the case, notice the spectral decomposition of $Z$ gives $r$ rank-one components 
that are mutually orthogonal, thus we may divide them into $\lceil r/r_\star\rceil$ subgroups
indexed by $i=1,\ldots,\lceil r/r_\star\rceil$,
such that each subgroup contains at most $r_\star$ components. 
Let $Z_i$ be the sum of the components in the subgroup $i$, 
it is easy to check that $Z_i$ has the desired property.

The property of the decomposition yields
\begin{equation}\label{eqn:pf-spectral-rip-decomp}
  \|Z\|_{\fro}^2=\tr(Z^2)=\sum_{i,j\le\lceil r/r_\star\rceil}\tr(Z_i Z_j)
  =\sum_{i\le\lceil r/r_\star\rceil}\|Z_i\|_{\fro}^2.
\end{equation}
But for each $Z_i$, Lemma~\ref{lem:fro_rip} implies
\begin{equation*}
  \|(\id-\opAA)(Z_i)\|\le \delta \|Z_i\|_{\fro}.
\end{equation*}
Summing up for $i\le\lceil r/r_\star\rceil$ yields
\begin{equation*}
  \|(\id-\opAA)(Z)\|\le\sum_{i\le\lceil r/r_\star\rceil}\|(\id-\opAA)(Z_i)\|
  \le \delta \sum_{i\le \lceil r/r_\star\rceil}\|Z_i\|_{\fro}
  \le \delta \sqrt{\lceil r/r_\star\rceil}\,\|Z\|_{\fro},
\end{equation*}
where the last inequality follows from \eqref{eqn:pf-spectral-rip-decomp} 
and from Cauchy-Schwarz inequality.

The first inequality in Lemma~\ref{lem:spectral_rip} follows from the above inequality 
by noting that $\lceil r/r_\star\rceil\le 2r/r_\star$ 
given $r\ge r_\star$ which was assumed in the beginning of the proof. 
The second inequality in Lemma~\ref{lem:spectral_rip} follows from
$\|Z\|_{\fro}\le\sqrt{r}\|Z\|$.
\end{proof}


\subsection{Matrix perturbation results}

The next few results are all on matrix perturbations. We first present a perturbation result on matrix inverse. 
\begin{lemma}\label{prop:inv}
  Assume that $A, B$ are square matrices of the same dimension, and that $A$ is invertible. 
  If $\|A^{-1}B\|\le 1/2$, then
  \begin{equation*}
    (A+B)^{-1}=A^{-1}+A^{-1}BQA^{-1},\quad\text{for some }\|Q\|\le 2.
  \end{equation*}
  Similarly, if $\|BA^{-1}\|\le 1/2$, then we have
  \begin{equation*}
    (A+B)^{-1}=A^{-1}+A^{-1}QBA^{-1},\quad\text{for some }\|Q\|\le 2.
  \end{equation*}
  In particular, if $\|B\|\le \smin(A)/2$, then both of the above equations hold.
\end{lemma}
\begin{proof}
  The claims follow from the identity 
  \[(A+B)^{-1}
    =A^{-1}-A^{-1}B(I+A^{-1}B)^{-1}A^{-1}
    =A^{-1}-A^{-1}(I+BA^{-1})^{-1}BA^{-1}.
  \]
For the first claim when $\|A^{-1}B\|\le 1/2$, we set $Q:= -(I+A^{-1}B)^{-1}$, which satisfies $\|Q \| = \|(I+A^{-1}B)^{-1}\|\le \frac{1}{1-\|A^{-1}B\|} \le 2$. The second claim follows similarly. Finally, we note that when $\|B\|\le \smin(A)/2$, it holds
$$ \|A^{-1}B\| \leq  \frac{1}{\smin(A)}\|B\| \leq \frac{1}{2} \qquad \mbox{and}\qquad  \|BA^{-1}\| \leq \| B\| \frac{1}{\smin(A)} \leq \frac{1}{2}, $$
thus completing the proof.
\end{proof}

Next, we focus on the minimum singular value of certain matrix of form $I+AB$.
\begin{lemma}\label{prop:positive-AB}
  If $A$, $B$ are positive definite matrices of the same size, we have
  \begin{equation*}
    \smin(I+AB)\ge\kappa^{-1/2}(A), \quad \text{where } \kappa(A)\defeq\frac{\|A\|}{\smin(A)}.
  \end{equation*}
\end{lemma}
\begin{proof}
  Writing $I+AB=A^{1/2}(I+A^{1/2}BA^{1/2})A^{-1/2}$, 
  we obtain
  \[\smin(I+AB)\ge\smin(A^{1/2})\smin(A^{-1/2})
  \smin(I+A^{1/2}BA^{1/2}).
  \]
  The proof is completed by noting that  $\smin(A^{1/2})=\smin^{1/2}(A)$, $\smin(A^{-1/2})=\|A\|^{-1/2}$, 
  and that $\smin(I+A^{1/2}BA^{1/2})\ge 1$ 
  since $A^{1/2}BA^{1/2}$ is positive semidefinite.
\end{proof}

The last result still concerns the minimum singular value of a matrix of interest. 
\begin{lemma}\label{lem:la-aux}
  There exists a universal constant $c_{\ref{lem:la-aux}}>0$ such that if
   $\Lambda$ is a positive definite matrix obeying $\|\Lambda\|\le c_{\ref{lem:la-aux}}$ 
  and $\smin(Y)\le 1/3$, then for any $\eta\le c_{\ref{lem:la-aux}}$ we have
  \begin{equation}
    \smin\Big(\!\left( (1-\eta) I +\eta(YY^\T+\Lambda)^{-1}\right)Y\Big)
    \ge\left(1+\frac{\eta}6\right)\smin(Y).
  \end{equation}
\end{lemma}
\begin{proof}
Denote $Z=YY^\T$ and let $\SVDU\SVDSigma\SVDU^\T=Z+\Lambda$ 
be the spectral decomposition of $Z+\Lambda$. 
By a coordinate transform one may assume $Z+\Lambda=\Sigma$.
It suffices to show
\begin{equation}\label{eqn:apd:smin-goal}
  \lambda_{\min}\Big(\!\left( (1-\eta)I+\eta\Sigma^{-1} \right) Z 
    \left( (1-\eta)I+\eta\Sigma^{-1} \right)\!\Big)
  \ge\left(1+\frac16\eta\right)^2\lambda_{\min}(Z).
\end{equation}

For simplicity we denote $\zeta=\lambda_{\min}(Z)$, 
which is by assumption smaller than $1/9$. 
Fix $K = 1/4$ so that $K\ge 2\zeta+ 4c_{\ref{lem:la-aux}}$ by choosing $c_{\ref{lem:la-aux}}$ to be small enough. 
By permuting coordinates we may further assume that 
the diagonal matrix $\Sigma$ is of the following form:
\begin{equation}\label{eqn:apd:Sigma-decomp}
  \SVDSigma=\begin{bmatrix}
    \SVDSigma_{\le K} & \\
    & \SVDSigma_{>K}
  \end{bmatrix},
\end{equation}
where $\SVDSigma_{\le K}$, $\SVDSigma_{> K}$ are diagonal matrices 
such that $\lambda_{\max}(\SVDSigma_{\le K})\le K$ 
and $\lambda_{\min}(\SVDSigma_{>K})>K$. It suffices to consider the case where $\SVDSigma_{>K}$ is not degenerate, because otherwise $\lambda_{\max}(\Sigma) \leq K \leq 1/2$, and the desired (\ref{eqn:apd:smin-goal}) follows as
\begin{align*}
\lambda_{\min}\Big(\!\left( (1-\eta)I+\eta\Sigma^{-1} \right) Z 
    \left( (1-\eta)I+\eta\Sigma^{-1} \right)\!\Big) \geq \big(1 - \eta + \eta\lambda_{\max}^{-1}(\Sigma)\big)^2\lambda_{\min}(Z) \geq (1+\eta)^2\lambda_{\min}(Z).
\end{align*}
For the rest of the proof, we assume the block corresponding to $\SVDSigma_{>K}$ is not degenerate.

Divide $Z$ into blocks of the same shape as \eqref{eqn:apd:Sigma-decomp}:
\begin{equation}
  Z=\begin{bmatrix}
    Z_0 & A \\
    A^\T & Z_1
  \end{bmatrix}.
\end{equation}
The purpose of such division is to facilitate 
computation of minimum eigenvalues by Schur's complement lemma.
For preparation, we make a few simple observations. 
Since $Z=\Sigma-\Lambda$, we see that $A$ being an off-diagonal submatrix of $Z$ satisfies 
$\|A\|\le\|\Lambda\|\le  c_{\ref{lem:la-aux}} $, 
and similarly $\|Z_0-\SVDSigma_{\le K}\|\le  c_{\ref{lem:la-aux}} $, $\|Z_1-\SVDSigma_{>K}\|\le  c_{\ref{lem:la-aux}} $.
In particular, we have
\begin{equation}\label{eqn:la-aux-Z1}
\lambda_{\min}(Z_1)
\ge\lambda_{\min}(\SVDSigma_{>K})- c_{\ref{lem:la-aux}} >K- c_{\ref{lem:la-aux}} 
\ge 2\zeta + 3c_{\ref{lem:la-aux}} > \zeta,
\end{equation}
which implies $Z_1-\zeta I$ is positive definite and invertible.
Thus by Schur's complement lemma, $Z\succeq\zeta I$ is equivalent to
\begin{equation}
  \label{eqn:apd:smin-condition-original}
  Z_0-\zeta I-A(Z_1-\zeta I)^{-1}A^\T\succeq 0,
\end{equation}
which provides an analytic characterization for the minimum eigenvalue $\zeta$ of $Z$. 

The rest of the proof follows from the following steps: we will first show again by Schur's complement lemma  
that \eqref{eqn:apd:smin-goal} admits a similar analytic characterization. 
More precisely, denoting $\zeta'=(1+\frac{\eta}{6})^2\zeta$, 
$\SVDSigma_0=(1-\eta) I +\eta\SVDSigma_{\le K}^{-1}$ 
and $\SVDSigma_1= (1-\eta) I +\eta\SVDSigma_{> K}^{-1}$, 
then \eqref{eqn:apd:smin-goal} is equivalent to
\begin{equation}
  \label{eqn:apd:smin-condition-new}
  Z_0 - \zeta'\SVDSigma_0^{-2} - A(Z_1-\zeta'\SVDSigma_1^{-2})^{-1}A^\T\succeq 0.
\end{equation}
After proving they are equivalent, we will prove that~\eqref{eqn:apd:smin-condition-new} holds as long as the following sufficient condition holds
\begin{equation}
  \label{eqn:apd:smin-condition-new-tmp0}
  Z_0-(1+3\eta)^{-2}\zeta'I
  -A(Z_1-\zeta I)^{-1}A^\T
  -10\eta\zeta A(Z_1-\zeta I)^{-2}A^\T\succeq 0.
\end{equation}
In the last step, we establish the above sufficient condition to complete the proof.

\paragraph{Step 1: equivalence between~\eqref{eqn:apd:smin-goal} 
and~\eqref{eqn:apd:smin-condition-new}.} First notice that
\begin{equation}
  \left( (1-\eta) I  +\eta\Sigma^{-1} \right)Z \left( (1-\eta) I  +\eta\Sigma^{-1} \right)
  =\begin{bmatrix}
    \SVDSigma_0 Z_0\SVDSigma_0 & \SVDSigma_0 A\SVDSigma_1 \\
    \SVDSigma_1 A^\T\SVDSigma_0 & \SVDSigma_1 Z_1\SVDSigma_1
  \end{bmatrix}.
\end{equation}
In order to invoke Schur's complement lemma, we need to verify $\SVDSigma_1 Z_1\SVDSigma_1 - \zeta' I\succ 0$. 
Observe that by definition we have
\begin{equation}\label{eqn:apd:multiplier}
  \SVDSigma_0\succeq\big(1+(K^{-1}-1)\eta\big)I=(1+3\eta)I,\quad \SVDSigma_1\succeq(1-\eta)I.
\end{equation}
Hence
\[\SVDSigma_1 Z_1\SVDSigma_1 - \zeta' I
\succeq (1-\eta)^2 Z_1 - \left(1+\frac16\eta\right)^2\zeta I
\succ 2(1-\eta)^2\zeta I - \left(1+\frac16\eta\right)^2\zeta I
\succ 0,
\]
where in the second inequality we used $Z_1-2\zeta I\succ 0$ proved in \eqref{eqn:la-aux-Z1}, 
and in the last inequality we used $\eta\le c_\eta$ with $c_\eta$ sufficiently small.
This completes the verification that $\SVDSigma_1 Z_1\SVDSigma_1 - \zeta' I\succ 0$. 
Now, invoking Schur's complement lemma yields that~\eqref{eqn:apd:smin-goal} 
is equivalent to
\begin{equation*}
  \SVDSigma_0 Z_0\SVDSigma_0 - \zeta' I
  - \SVDSigma_0 A\SVDSigma_1(\SVDSigma_1 Z_1\SVDSigma_1-\zeta' I)^{-1}\SVDSigma_1 A^\T\SVDSigma_0
  \succeq 0,
\end{equation*}
which simplifies easily to~\eqref{eqn:apd:smin-condition-new}, as claimed.

\paragraph{Step 2: establishing~\eqref{eqn:apd:smin-condition-new-tmp0} as a sufficient condition for~\eqref{eqn:apd:smin-condition-new}.}  
By \eqref{eqn:apd:multiplier}, it follows that 
\begin{align}
  (Z_1-\zeta'\SVDSigma_1^{-2})^{-1}
  &\preceq (Z_1-(1-\eta)^{-2}\zeta'I)^{-1}
  \nonumber\\
  &= \Big(Z_1-\zeta I - \big((1-\eta)^{-2}\zeta'-\zeta\big) I\Big)^{-1},
  \label{eqn:Z1-perturb-bound}
\end{align}
where we used the well-known fact that $A\preceq B$ implies $B^{-1}\preceq A^{-1}$ for positive definite matrices $A$ and $B$ (cf. \cite[Proposition~V.1.6]{bhatia1997matrix}).
We aim to apply Lemma~\ref{prop:inv} to control the above term, by treating $((1-\eta)^{-2}\zeta'-\zeta) I$ as a perturbation term. For this purpose we need to verify
\begin{equation}\label{eq:lemma_inv_pre}
\left|(1-\eta)^{-2}\zeta'-\zeta\right|\le \frac12\lambda_{\min}(Z_1-\zeta I).
\end{equation} 
Given $\eta\le c_\eta$ with sufficiently small $c_\eta$, we have 
$(1-\eta)^{-2}\le 1+3\eta$, $(1+\frac16\eta)^2\le 1+\eta$, and 
$(1+3\eta)(1+\eta)\le 1+5\eta$, thus 
\[
  0\le (1-\eta)^{-2}\big(1+\frac16\eta\big)^2\zeta-\zeta 
  = (1-\eta)^{-2}\zeta'-\zeta
  \le (1+3\eta)(1+\eta)\zeta - \zeta
  \le 5\eta\zeta < \zeta/2,
\]
where the last inequality follows from $c_\eta\le 1/10$. On the other hand,
invoking \eqref{eqn:la-aux-Z1}, we obtain 
$$ \frac{1}{2} \zeta \leq \frac12\big(\lambda_{\min}(Z_1)-\zeta\big)
=\frac{1}{2}\lambda_{\min}(Z_1-\zeta I),$$ 
which verifies \eqref{eq:lemma_inv_pre}.
Thus we may apply Lemma~\ref{prop:inv} to show
\begin{align*}
  \left\|(Z_1-\zeta I)
  \left((Z_1-\zeta I)^{-1}-\left(Z_1-\zeta I - ((1-\eta)^{-2}\zeta'-\zeta) I\right)^{-1}\right)
  (Z_1-\zeta I)
  \right\|
  \le 2\left|(1-\eta)^{-2}\zeta'-\zeta\right|
  \le 10\eta\zeta,
\end{align*}
therefore
\begin{equation*}
  \left(Z_1-\zeta I - ((1-\eta)^{-2}\zeta'-\zeta) I\right)^{-1}
  \preceq (Z_1-\zeta I)^{-1}+10\eta\zeta(Z_1-\zeta I)^{-2}.
\end{equation*}
Together with \eqref{eqn:Z1-perturb-bound}, this implies
\begin{equation}\label{eqn:apd:schur-factor-approx}
  (Z_1-\zeta'\SVDSigma_1^{-2})^{-1}
  \preceq (Z_1-\zeta I)^{-1}+10\eta\zeta(Z_1-\zeta I)^{-2}. 
\end{equation}
Combining \eqref{eqn:apd:multiplier} and \eqref{eqn:apd:schur-factor-approx}, 
we see that a sufficient condition for \eqref{eqn:apd:smin-condition-new} to hold is \eqref{eqn:apd:smin-condition-new-tmp0}.

\paragraph{Step 3: establishing \eqref{eqn:apd:smin-condition-new-tmp0}.}
 It is clear that~\eqref{eqn:apd:smin-condition-new-tmp0} is implied by
\begin{equation}\label{eqn:smin-condition-tmp1}
  \zeta I-(1+3\eta)^{-2}\zeta'I
  -10\eta\zeta A(Z_1-\zeta I)^{-2}A^\T\succeq 0,
\end{equation}
by leveraging the relation
$Z_0\succeq \zeta I+A(Z_1-\zeta I)^{-1}A^\T$ from \eqref{eqn:apd:smin-condition-original}.

Hence, it boils down to prove \eqref{eqn:smin-condition-tmp1}. Recalling $\|A\|\le c_{\ref{lem:la-aux}}$, and from \eqref{eqn:la-aux-Z1}, 
we know 
$\lambda_{\min}(Z_1-\zeta I)\ge K-c_{\ref{lem:la-aux}}-\zeta\ge\zeta+3c_{\ref{lem:la-aux}}$. 
Thus 
\[\|A(Z_1-\zeta I)^{-2}A^\T\|\le\|A\|^2\|(Z_1-\zeta I)^{-2}\|
\le c_{\ref{lem:la-aux}}^2/(\zeta+3c_{\ref{lem:la-aux}})^2\le 1/9.
\] 
Therefore, to prove \eqref{eqn:smin-condition-tmp1} it suffices to show
\begin{equation}\label{eqn:smin-condition-final}
  \zeta - (1+3\eta)^{-2}\zeta'
  \ge \frac{10}9\eta\zeta.
\end{equation} 
It is easy to verify that the above inequality holds
for our choice $\zeta'=(1+\frac16\eta)^2\zeta$. 
In fact, given $\eta\le c_\eta$ for sufficiently small $c_\eta$,
we have $(1+3\eta)^{-2}\le 1-4\eta$, $(1+\frac16\eta)^2\le 1+\eta$. 
These together yield
\begin{equation*}
  \zeta - (1+3\eta)^{-2}\big(1+\frac16\eta\big)^2\zeta
  \ge \zeta - (1-4\eta)(1+\eta)\zeta
  = 3\eta\zeta + 4\eta^2\zeta
  \ge 3\eta\zeta\ge\frac{10}{9}\eta\zeta,
\end{equation*}
establishing \eqref{eqn:smin-condition-final} as desired. 
\end{proof}
\section{Decompositions of key terms}
\label{sec:approx}

In this section, we first present a useful bound of a key error quantity
 \begin{equation}\label{eq:defn-Delta}
  \Delta_t \defeq (\id-\opAA)(\myX_t\myX_t^\T-\Mtruth),
\end{equation}
where $\myX_t$ is the iterate of \myalg given in \eqref{eqn:update}.

\begin{lemma}\label{lem:Delta-bound}
Suppose $\opA(\cdot)$ satisfies Assumption~\ref{assumption:opA}.
For any $t\ge 0$ such that~\eqref{subeq:condition-t-1} holds, 
we have
\begin{equation}\label{eqn:Delta-bound}
  \|\Delta_t\|\le
  8\delta\left(\|\EssS_t\EssS_t^\T-\SigmaTruth^2\|_{\fro}
  +\|\EssS_t\|\|\Misalign_t\|_{\fro}
  +n\|\Overpar_t\|^2\right).
\end{equation}
In particular, 
there exists some constant $c_{\ref{lem:Delta-bound}}\lesssim c_\delta/c_\lambda$ such that
\begin{equation}
  \label{eqn:Delta-bound-coarse}
  \|\Delta_t\|\le 16(C_{\ref{lem:p1.5}.a}+1)^2c_\delta\kappa^{-2C_\delta/3}\|\Xtruth\|^2
  \le c_{\ref{lem:Delta-bound}}\kappa^{-2C_\delta/3}\|\Xtruth\|^2.
\end{equation}
\end{lemma}

\begin{proof}
The decomposition \eqref{eqn:x-decomp-final} in Proposition~\ref{prop:decomposition} yields 
\begin{equation*}
  \myX_t\myX_t^\T = \Utruth\EssS_t\EssS_t^\T\Utruth^\T 
  + \Utruth\EssS_t\Misalign_t^\T\UperpTruth^\T
  + \UperpTruth\Misalign_t\EssS_t^\T\Utruth^\T
  + \UperpTruth\Misalign_t\Misalign_t^\T\UperpTruth^\T
  + \UperpTruth\Overpar_t\Overpar_t^\T\UperpTruth^\T.
\end{equation*}
Since $\Mtruth=\Utruth\SigmaTruth^2\Utruth^\T$, we have
\begin{equation}\label{eqn:loss-decomp-primal}
  \myX_t\myX_t^\T-\Mtruth = \underbrace{\Utruth(\EssS_t\EssS_t^\T-\SigmaTruth^2)\Utruth^\T}_{=:T_1}
  + \underbrace{\Utruth\EssS_t\Misalign_t^\T\UperpTruth^\T
  + \UperpTruth\Misalign_t\EssS_t^\T\Utruth^\T}_{=:T_2}
  + \underbrace{\UperpTruth\Misalign_t\Misalign_t^\T\UperpTruth^\T}_{=:T_3}
  + \underbrace{\UperpTruth\Overpar_t\Overpar_t^\T\UperpTruth^\T}_{=:T_4}.
\end{equation}
Note that $\Utruth\in\reals^{n\times r_\star}$ is of rank $r_\star$,
thus $T_1$ has rank at most $r_\star$ and $T_2$ has rank at most $2r_\star$. Similarly, since $\Misalign_t=\Noise_t\SVDV_t$ while $\SVDV_t\in\reals^{r\times r_\star}$ is of rank $r_\star$, $T_3$ has rank at most $r_\star$.
It is also trivial that $T_4$ as an $n\times n$ matrix has rank at most $n$.
Invoking Lemma~\ref{lem:spectral_rip}, we obtain
\begin{align*}
  \|(\id-\opAA)(T_1)\|
  &\le 2\delta\|\Utruth(\EssS_t\EssS_t^\T-\SigmaTruth^2)\Utruth^\T\|_{\fro}
  \le 2\delta\|\EssS_t\EssS_t^\T-\SigmaTruth^2\|_{\fro},
  \\
  \|(\id-\opAA)(T_2)\|
  &\le 2\sqrt{3}\delta\|\Utruth\EssS_t\Misalign_t^\T\UperpTruth^\T
  + \UperpTruth\Misalign_t\EssS_t^\T\Utruth^\T\|_{\fro}
  \le 4\sqrt{2}\delta\|\EssS_t\|\|\Misalign_t\|_{\fro},
  \\
  \|(\id-\opAA)(T_3)\|
  &\le 2\delta\|\UperpTruth\Misalign_t\Misalign_t^\T\UperpTruth^\T\|_{\fro}
  \le 2\delta\|\Misalign_t\EssS_t^{-1}\SigmaTruth\|\|\EssS_t\|\|\SigmaTruth^{-1}\|\|\Misalign_t\|_{\fro}
  \le\delta\|\EssS_t\|\|\Misalign_t\|_{\fro},
  \\
  \|(\id-\opAA)(T_4)\|
  &\le 2\delta n\|\UperpTruth\Overpar_t\Overpar_t^\T\UperpTruth^\T\|
  \le 2\delta n\|\Overpar_t\|^2,
\end{align*}
where the third line follows from $\|\SigmaTruth^{-1}\|=\kappa\|\Xtruth\|^{-1}$ 
and from~\eqref{eqn:p1.5-3} in view that $C_\delta$ is sufficiently large and $c_{\ref{lem:p1.5}}$ is sufficiently small.
The conclusion \eqref{eqn:Delta-bound} follows from summing up the above inequalities. 

For the remaining part of the lemma, 
 note that the following inequalities that bound the individual terms of \eqref{eqn:Delta-bound} can be inferred from \eqref{subeq:condition-t-1}: namely, 
\[
  \|\EssS_t\EssS_t^\T-\SigmaTruth\|_{\fro}
  \le\sqrt{2r_\star}\|\EssS_t\EssS_t^\T-\SigmaTruth\|
  \le \sqrt{2r_\star}(C_{\ref{lem:p1.5}.a}^2\kappa^2+1)\|\Xtruth\|^2
\]  
by \eqref{eqn:p1.5-4}, and 
\begin{align*}
  \|\EssS_t\|\|\Misalign_t\|_{\fro}& \le\sqrt{r_\star}\|\EssS_t\|\|\Misalign_t\| \\
  &\le \sqrt{r_\star} (C_{\ref{lem:p1.5}.a} \kappa^3 \|\Xtruth\|) \cdot
  \|\Misalign_t\EssS_t^{-1}\SigmaTruth\|\cdot\|\EssS_t\|\cdot\|\SigmaTruth^{-1}\|
  \\
  &\le \sqrt{r_\star} (C_{\ref{lem:p1.5}.a}\kappa^3 \|\Xtruth\|) \cdot(c_{\ref{lem:p1.5}}\kappa^{-C_\delta/2}\|\Xtruth\|)
  \cdot (C_{\ref{lem:p1.5}.a}\kappa^3 \|\Xtruth\|) \cdot \smin^{-1}(\SigmaTruth)
  \\
  &= \sqrt{r_\star} c_{\ref{lem:p1.5}} C_{\ref{lem:p1.5}.a}^2 \kappa^6 \|\Xtruth\|^2\kappa^{-C_\delta/2}
  \\
  &\le \sqrt{r_\star}C_{\ref{lem:p1.5}.a}^2\|\Xtruth\|^2,
\end{align*}
where the first inequality uses the fact that 
$\Misalign_t=\Noise_t\SVDV_t$ contains a rank-$r_\star$ factor $\SVDV_t$, 
hence has rank at most $r_\star$; 
the second line follows from \eqref{eqn:p1.5-4}, the third line follows from \eqref{eqn:p1.5-3} and \eqref{eqn:p1.5-4},
and the last line follows from choosing $c_\delta$ sufficiently small 
such that $c_{\ref{lem:p1.5}}\le 1$ (which is possible since $c_{\ref{lem:p1.5}}\lesssim c_\delta/c_\lambda)$ 
and from choosing $C_\delta$ such that $\kappa^6 \kappa^{-C_\delta/2}\le 1$. 
Finally, from \eqref{eqn:p1.5-2} and its corollary \eqref{eqn:p1.5-2-var}, we have
\[
  2n\|\Overpar_t\|^2\le 2n\alpha^{3/2}\|\Xtruth\|^{1/2}
  \le \|\Xtruth\|^2,
\]
since from \eqref{eqn:alpha-cond} it is easy to show that 
$\alpha\le (2n)^{-2/3}\|\Xtruth\|$.

Combining these inequalities and \eqref{eqn:Delta-bound} yields
\begin{equation}
  \label{eqn:Delta-bound-coarse-primitive}
  \|\Delta_t\|\le 8\delta\sqrt{r_\star}(\sqrt{2} C_{\ref{lem:p1.5}.a}^2\kappa^2+1+C_{\ref{lem:p1.5}.a}^2+1)\|\Xtruth\|^2
  \le 16\delta\sqrt{r_\star}\kappa^2(C_{\ref{lem:p1.5}.a}^2+1)\|\Xtruth\|^2.
\end{equation}
Recalling that by \eqref{eqn:delta-cond} we have 
$\delta\sqrt{r_\star}\kappa^2\le c_\delta\kappa^{-C_\delta+2}\le c_\delta\kappa^{-2C_\delta/3}$ 
as long as $C_\delta\ge 6$, we obtain the desired conclusion. 
We may choose $c_{\ref{lem:Delta-bound}}=32(C_{\ref{lem:p1.5}.a}+1)^2 c_\delta$, and the bound $c_{\ref{lem:Delta-bound}}\lesssim c_\delta/c_\lambda$ 
follows from $C_{\ref{lem:p1.5}.a}\lesssim c_\lambda^{-1/2}$.
\end{proof}

We next present several useful decompositions of the signal term $\Signal_{t+1}$ 
and the noise term~$\Noise_{t+1}$, which are extremely useful in later developments.  
  \begin{lemma}\label{lem:update-approx}
  For any $t$ such that $\EssS_t$ is invertible and~\eqref{subeq:condition-t-1} holds, 
  we have
  \begin{subequations}
  \begin{align}
    \Signal_{t+1} 
    & = \left( (1-\eta) I +\eta(\SigmaTruth^2+\lambda I+\Err{a}_t)
      (\EssS_t\EssS_t^\T +\lambda I)^{-1}
    \right)\EssS_t \SVDV_t^\T + \eta\Err{b}_t,
    \label{eqn:S-update-approx}\\
    \Noise_{t+1} 
    & = \Misalign_t\EssS_t^{-1}\left((1-\eta)\EssS_t\EssS_t^\T+\lambda I
      + \eta\Err{c}_t\right)(\EssS_t\EssS_t^\T+\lambda I)^{-1}\EssS_t \SVDV_t^\T 
    \nonumber \\
    &\qquad 
    + \eta \Err{e}_t(\EssS_t\EssS_t^\T+\lambda I)^{-1}\EssS_t \SVDV_t^\T 
    + \Overpar_t \Vperp{t}^\T + \eta\Err{d}_t,
    \label{eqn:N-update-approx}
  \end{align}
  \end{subequations}
  where the error terms satisfy
  \begin{subequations}
    \begin{align}
      \uinorm{\Err{a}_t} &\le 
      2c_{\ref{lem:p1.5}}\kappa^{-4}\|\Xtruth\| \cdot \uinorm{\Misalign_t\EssS_t^{-1}\SigmaTruth}
      + 2\uinorm{\Utruth^\T\Delta_t}, 
      \label{eqn:Err-1-bound}\\
      \uinorm{\Err{b}_t} &\le 
      \left(\frac{\|\Overpar_t\|}{\smin(\EssS_t)}\right)^{3/4}\smin(\EssS_t)
      \le \frac{1}{20}\kappa^{-10}\smin(\EssS_t), 
      \label{eqn:Err-2-bound}\\
      \uinorm{\Err{c}_t}&\le 
      \kappa^{-6}\|\Xtruth\| \cdot \uinorm{\Misalign_t\EssS_t^{-1}\SigmaTruth}, 
      \label{eqn:Err-3-bound}\\
      \uinorm{\Err{d}_t}&\le 
      \left(\frac{\|\Overpar_t\|}{\smin(\EssS_t)}\right)^{3/4}\smin(\EssS_t), 
      \label{eqn:Err-4-bound}\\
      \uinorm{\Err{e}_t} &\le 
      2\uinorm{\Utruth^\T\Delta_t} 
      + c_{\ref{lem:Delta-bound}}\kappa^{-6}\|\Xtruth\| \cdot  \uinorm{\Misalign_t\EssS_t^{-1}\SigmaTruth}.
      \label{eqn:Err-5-bound}
    \end{align}
  Moreover, we have
  \begin{align}
    \|\Err{b}_t\|\le\frac1{24\Cmax\kappa}\|\Overpar_t\|, 
    \label{eqn:Err-2-bound-spectral}\\
    \|\Err{d}_t\|\le\frac1{24\Cmax\kappa}\|\Overpar_t\|.
    \label{eqn:Err-4-bound-spectral}
  \end{align}
   \end{subequations}
Here, $ \uinorm{\cdot}$ can either be the Frobenius norm or the spectral norm. 
\end{lemma}

To proceed, we would need the approximate update equation of the rotated signal term $\EssS_{t+1}$, and the rotated misalignment term $\Misalign_{t+1}\EssS_{t+1}^{-1}$ later in the proof. 
Since directly analyzing the evolution of these two terms seems challenging, we resort to two surrogate matrices $\Signal_{t+1}\SVDV_t+\Signal_{t+1}\Vperp{t} Q$, 
and $(\Noise_{t+1}V_t+\Noise_{t+1} \Vperp{t} Q)(\Signal_{t+1}\SVDV_t + \Signal_{t+1}\Vperp{t} Q)^{-1}$, 
as documented in the following two lemmas.

\begin{lemma}\label{lem:S-surrogate}
  For any $t$ such that $\EssS_t$ is invertible and \eqref{subeq:condition-t-1} holds,  and any matrix $Q\in\reals^{(r-r_\star)\times r_\star}$ with $\|Q\|\le 2$, 
 we have
 \begin{equation}\label{eqn:Err-S1-bound}
   \Signal_{t+1}\SVDV_t+\Signal_{t+1}\Vperp{t} Q
   =(I+\eta\Err{\ref{lem:S-surrogate}}_t)\left((1-\eta)I+\eta(\SigmaTruth^2+\lambda I)
   (\EssS_t\EssS_t^\T +\lambda I)^{-1}\right)\EssS_t,
 \end{equation}
 where $\Err{\ref{lem:S-surrogate}}_t \in\mathbb{R}^{r_\star\times r_\star}$ is a matrix (depending on $Q$) satisfying 
 \[
   \|\Err{\ref{lem:S-surrogate}}_t\|\le\frac{1}{200(C_{\ref{lem:p1.5}.a}+1)^4\kappa^6}.
 \]
Here, $C_{\ref{lem:p1.5}.a} > 0$ is given in Lemma~\ref{lem:p1.5}. 
\end{lemma}

\begin{lemma}\label{lem:NS-surrogate}
  For any $t$ such that $\EssS_t$ is invertible and \eqref{subeq:condition-t-1} holds, and any matrix $Q\in\reals^{(r-r_\star)\times r_\star}$ with $\|Q\|\le 2$, 
 we have
  \begin{align*}
    &(\Noise_{t+1}V_t+\Noise_{t+1} \Vperp{t} Q)(\Signal_{t+1}\SVDV_t + \Signal_{t+1}\Vperp{t} Q)^{-1}
    \\
    &=\Misalign_t\EssS_t^{-1}(1+\eta\Err{\ref{lem:NS-surrogate}.a}_t)
    \big((1-\eta)\EssS_t\EssS_t^\T +\lambda I\big)
    \big((1-\eta)\EssS_t\EssS_t^\T+\lambda I+\eta\SigmaTruth^2\big)^{-1}
    (1+\eta\Err{\ref{lem:S-surrogate}}_t)^{-1}
    +\eta\Err{\ref{lem:NS-surrogate}.b}_t
  \end{align*}
  where $\Err{\ref{lem:NS-surrogate}.a}_t$, $\Err{\ref{lem:NS-surrogate}.b}_t$ are matrices (depending on $Q$) satisfying
  \begin{subequations}
  \begin{align}
    \|\Err{\ref{lem:NS-surrogate}.a}_t\|
    &\le \frac{1}{200(C_{\ref{lem:p1.5}.a}+1)^4\kappa^6},
    \label{eqn:Err-N2-bound}\\
    \uinorm{\Err{\ref{lem:NS-surrogate}.b}_t}
    &\le 400c_\lambda^{-1}\kappa^6\|\Xtruth\|^{-2}\uinorm{\Utruth^\T\Delta_t}
    +\frac{1}{64(C_{\ref{lem:p1.5}.a}+1)^2\kappa^5\|\Xtruth\|}\uinorm{\Misalign_t\EssS_t^{-1}\SigmaTruth}
    \nonumber\\
    &\phantom{\le}
    +\frac{1}{64}\left(\frac{\|\Overpar_t\|}{\smin(\EssS_t)}\right)^{2/3}.
  \end{align}
  \end{subequations}
  Here, $ \uinorm{\cdot}$ can either be the Frobenius norm or the spectral norm, and $C_{\ref{lem:p1.5}.a} > 0$ is given in Lemma~\ref{lem:p1.5}. 
\end{lemma}

\subsection{Proof of Lemma \ref{lem:update-approx}} \label{sec:proof-master-main}

We split the proof into three steps: (1) provide several useful approximation results regarding the 
matrix inverses utilizing the facts that $\|\Overpar_t\|$ and $\|\Misalign_t\EssS_t^{-1}\SigmaTruth\|$ 
are small (as shown by Lemma \ref{lem:p1.5}); (2) proving the claims~\eqref{eqn:S-update-approx}, 
\eqref{eqn:Err-1-bound}, \eqref{eqn:Err-2-bound}, and~\eqref{eqn:Err-2-bound-spectral} associated with 
the signal term $\Signal_{t+1}$;  (3)  proving the claims~\eqref{eqn:N-update-approx}, 
\eqref{eqn:Err-3-bound}, \eqref{eqn:Err-4-bound}, \eqref{eqn:Err-5-bound}, and~\eqref{eqn:Err-4-bound-spectral} associated with 
the noise term $\Noise_{t+1}$. Note that our approximation results in step (1) include choices of some matrices $\{Q_i\}$ with small spectral norms, whose choices may be different from lemma to lemma for simplicity of presentation;

\subsubsection{Step 1: preliminaries}
We know from \eqref{subeq:condition-t-1} that the 
overparametrization error $\Overpar_t$ is negligible compared 
to the signals $\EssS_t$ and $\smin(\Xtruth)$. This combined with the decomposition~\eqref{eqn:x-decomp-final} reveals 
a desired approximation $(\myX_t^\T \myX_t+\lambda I)^{-1} \approx 
(\SVDV_t(\EssS_t^\T \EssS_t + \Misalign_t^\T \Misalign_t)\SVDV_t^\T + \lambda I)^{-1}$. 
This approximation is formalized in the lemma below. 

\begin{lemma}\label{prop:pre-approx0}
  If $\lambda\ge 4(\|\Overpar_t\|^2 \vee 2\|\Misalign_t\|\|\Overpar_t\|)$ for some $t$, then
  \begin{align}
    (\myX_t^\T \myX_t+\lambda I)^{-1}
    & = \left(\SVDV_t(\EssS_t^\T \EssS_t + \Misalign_t^\T \Misalign_t)\SVDV_t^\T + \lambda I\right)^{-1} 
    \nonumber\\
    &\phantom{=} 
    + \left(\SVDV_t(\EssS_t^\T \EssS_t + \Misalign_t^\T \Misalign_t)\SVDV_t^\T + \lambda I\right)^{-1}
    \Err{\ref{prop:pre-approx0}.a}_t
    \left(\SVDV_t(\EssS_t^\T \EssS_t + \Misalign_t^\T \Misalign_t)\SVDV_t^\T + \lambda I\right)^{-1}
    \nonumber\\
    & = \left(\SVDV_t(\EssS_t^\T \EssS_t + \Misalign_t^\T \Misalign_t)\SVDV_t^\T + \lambda I\right)^{-1}
    \left(I + \Err{\ref{prop:pre-approx0}.b}_t\right)
  \label{eqn:scale-factor-approx}
  \end{align}
  where the error terms $\Err{\ref{prop:pre-approx0}.a}_t$, $\Err{\ref{prop:pre-approx0}.b}_t$ can be expressed as
  \begin{subequations}
  \begin{align}
    \Err{\ref{prop:pre-approx0}.a}_t &= (\Vperp{t} \Overpar_t^\T \Overpar_t \Vperp{t}^\T + \SVDV_t\Misalign_t^\T\Overpar_t \Vperp{t}^\T + \Vperp{t}{\Overpar_t}^\T\Misalign_t \SVDV_t^\T)Q_1,
    \label{eqn:Err0}\\
    \Err{\ref{prop:pre-approx0}.b}_t&=\lambda^{-1}\Err{\ref{prop:pre-approx0}.a}_t Q_2,
    \label{eqn:Err0'}
  \end{align}
  \end{subequations}
  for some matrices $Q_1, Q_2$ such that $\max\{\|Q_1\|, \|Q_2\|\}\le 2$.
\end{lemma}

\begin{proof}
  Expanding $\myX_t^\T \myX_t$ according to \eqref{eqn:x-decomp-final}, we have
  \begin{equation*}
    \myX_t^\T \myX_t = \SVDV_t(\EssS_t^\T\EssS_t+\Misalign_t^\T\Misalign_t)\SVDV_t^\T 
    + \Vperp{t} \Overpar_t^\T \Overpar_t \Vperp{t}^\T + \SVDV_t\Misalign_t^\T\Overpar_t \Vperp{t}^\T + \Vperp{t}{\Overpar_t}^\T\Misalign_t \SVDV_t^\T.
  \end{equation*}
  The conclusion readily follows from Lemma~\ref{prop:inv} by setting therein $A=\SVDV_t(\EssS_t^\T\EssS_t+\Misalign_t^\T\Misalign_t)\SVDV_t^\T+\lambda I$ 
  and $B=\Vperp{t} \Overpar_t^\T \Overpar_t \Vperp{t}^\T + \SVDV_t\Misalign_t^\T\Overpar_t \Vperp{t}^\T + \Vperp{t}{\Overpar_t}^\T\Misalign_t \SVDV_t^\T$, where the condition $\|A^{-1}B\| \leq 1/2$ is satisfied since
  \begin{align*}
  \|A^{-1}B\| \leq \smin(A)^{-1}\|B\| \leq \lambda^{-1}\cdot(\|\Overpar_t\|^2 + 2\|\Overpar_t\|\|\Misalign_t\|) \leq 1/2.
  \end{align*}
\end{proof}

Moreover, the dominating term on the right hand side of~\eqref{eqn:scale-factor-approx} can be 
equivalently written as 
\begin{align}\label{eqn:scale-factor-approx-2}
  \left(\SVDV_t(\EssS_t^\T \EssS_t + \Misalign_t^\T \Misalign_t)\SVDV_t^\T + \lambda I\right)^{-1}
  \notag&= \Big(\SVDV_t(\EssS_t^\T \EssS_t + \Misalign_t^\T \Misalign_t + \lambda I)\SVDV_t^\T + \lambda \Vperp{t}\Vperp{t}^\T \Big)^{-1}\\
  &= \SVDV_t(\EssS_t^\T \EssS_t + \Misalign_t^\T \Misalign_t+\lambda I)^{-1}\SVDV_t^\T+\lambda^{-1}\Vperp{t} \Vperp{t}^\T.
\end{align}
When the misalignment error $\|\Misalign_t\EssS_t^{-1}\SigmaTruth\|$ is small, we expect 
$(\EssS_t^\T \EssS_t + \Misalign_t^\T \Misalign_t+\lambda I)^{-1} \approx (\EssS_t^\T \EssS_t +\lambda I)^{-1}$,
which is formalized in the following lemma that establishes  $(\EssS_t\EssS_t^\T+\EssS_t\Misalign_t^\T\Misalign_t\EssS_t^{-1}+\lambda I)^{-1}
     \approx (\EssS_t\EssS_t^\T+\lambda I)^{-1}$, due to the following approximation 
\begin{align*}
  (\EssS_t^\T\EssS_t+\Misalign_t^\T\Misalign_t+\lambda I)^{-1}
  &=\EssS_t^{-1}(\EssS_t\EssS_t^\T+\EssS_t\Misalign_t^\T\Misalign_t\EssS_t^{-1}+\lambda I)^{-1}\EssS_t\\
  &\approx \EssS_t^{-1}(\EssS_t\EssS_t^\T + \lambda I)^{-1}\EssS_t = (\EssS_t^\T\EssS_t + \lambda I)^{-1}.
\end{align*}

\begin{lemma}\label{prop:rearranged-inv-approx}
  If $\|\Misalign_t\EssS_t^{-1}\SigmaTruth\|\le\smin(\Xtruth)/16$ for some $t$, 
  then
  \begin{align}
    (\EssS_t\EssS_t^\T+\EssS_t\Misalign_t^\T\Misalign_t\EssS_t^{-1}+\lambda I)^{-1}
    = (I+\Err{\ref{prop:rearranged-inv-approx}}_t)(\EssS_t\EssS_t^\T+\lambda I)^{-1}, 
  \label{eqn:rearranged-inv-approx}
  \end{align}
  where the error term $\Err{\ref{prop:rearranged-inv-approx}}_t$ is a matrix defined as
  \begin{equation}\label{eqn:Err-S0-def}
  \Err{\ref{prop:rearranged-inv-approx}}_t = 
  \kappa^2\|\Xtruth\|^{-2}\|\Misalign_t\EssS_t^{-1}\SigmaTruth\|
  Q_1(\Misalign_t\EssS_t^{-1}\SigmaTruth)Q_2 ,
  \end{equation}
  where $Q_1$, $Q_2$ are matrices of appropriate dimensions
  satisfying $\|Q_1\|\le 1$, $\|Q_2\|\le 2$. 
  In particular, we have
  \begin{equation}\label{eqn:Err-S0-bound}
    \uinorm{\Err{\ref{prop:rearranged-inv-approx}}_t}
    \le 2\kappa^2\|\Xtruth\|^{-2}\|\Misalign_t\EssS_t^{-1}\SigmaTruth\| \cdot
      \uinorm{\Misalign_t\EssS_t^{-1}\SigmaTruth},
  \end{equation}
  where $\uinorm{\cdot}$ can be either the operator norm or the Frobenius norm.
\end{lemma}

\begin{proof}
  In order to apply Lemma~\ref{prop:inv}, setting $A = \EssS_t\EssS_t^\T+\lambda I$ and $B=\EssS_t\Misalign_t^\T\Misalign_t\EssS_t^{-1} $, it is straightforward to verify that
  $$ \| A^{-1} B\| = \| ( \EssS_t\EssS_t^\T+\lambda I)^{-1} \EssS_t\Misalign_t^\T\Misalign_t\EssS_t^{-1} \|  \leq \|\Misalign_t\EssS_t^{-1}\|^2 \leq \|\Misalign_t\EssS_t^{-1}\SigmaTruth\|^2\|\SigmaTruth^{-1}\|^2 \leq (1/16)^2,$$
    where we use the obvious fact that $\|(\EssS_t\EssS_t^\T+\lambda I)^{-1}\EssS_t\EssS_t^\T\|\le 1$. 
    Applying Lemma~\ref{prop:inv}, we obtain
  \begin{align*}
    & (\EssS_t\EssS_t^\T+\EssS_t\Misalign_t^\T\Misalign_t\EssS_t^{-1}+\lambda I)^{-1}
    - (\EssS_t\EssS_t^\T+\lambda I)^{-1}
    \\
&    =  (\EssS_t\EssS_t^\T+\lambda I)^{-1}\EssS_t
    \Misalign_t^\T\Misalign_t\EssS_t^{-1} Q
    (\EssS_t\EssS_t^\T+\lambda I)^{-1}
    \\
  &   = (\EssS_t\EssS_t^\T+\lambda I)^{-1}\EssS_t\EssS_t^\T\SigmaTruth^{-1}
    (\Misalign_t\EssS_t^{-1}\SigmaTruth)^\T(\Misalign_t\EssS_t^{-1}\SigmaTruth)\SigmaTruth^{-1} Q
    (\EssS_t\EssS_t^\T+\lambda I)^{-1}
    \end{align*}
    for some matrix $Q$ with $\|Q\|\le 2$.  
    Since one may further write
    \begin{align*}
    & (\EssS_t\EssS_t^\T+\EssS_t\Misalign_t^\T\Misalign_t\EssS_t^{-1}+\lambda I)^{-1}
    - (\EssS_t\EssS_t^\T+\lambda I)^{-1}
    \nonumber\\
 &   =  \|\SigmaTruth^{-1}\|^2 \|\Misalign_t\EssS_t^{-1}\SigmaTruth\|(\EssS_t\EssS_t^\T+\lambda I)^{-1}\EssS_t\EssS_t^\T\frac{\SigmaTruth^{-1}}{\|\SigmaTruth^{-1}\|}
    \frac{(\Misalign_t\EssS_t^{-1}\SigmaTruth)^\T}{\|\Misalign_t\EssS_t^{-1}\SigmaTruth\|}(\Misalign_t\EssS_t^{-1}\SigmaTruth)\frac{\SigmaTruth^{-1}}{\|\SigmaTruth^{-1}\|} Q
    (\EssS_t\EssS_t^\T+\lambda I)^{-1},
    \end{align*}
    the conclusion follows by setting $\Err{\ref{prop:rearranged-inv-approx}}_t$ as in (\ref{eqn:Err-S0-def}) with 
    \begin{align*}
    Q_1 = (\EssS_t\EssS_t^\T+\lambda I)^{-1}\EssS_t\EssS_t^\T\frac{\SigmaTruth^{-1}}{\|\SigmaTruth^{-1}\|}
    \frac{(\Misalign_t\EssS_t^{-1}\SigmaTruth)^\T}{\|\Misalign_t\EssS_t^{-1}\SigmaTruth\|}, \quad Q_2 = \frac{\SigmaTruth^{-1}}{\|\SigmaTruth^{-1}\|} Q.
    \end{align*}
    The last inequality \eqref{eqn:Err-S0-bound} is then a direct consequence of \eqref{eqn:Err-S0-def}.
\end{proof}

\subsubsection{Step 2: a key recursion}

Recall the definition $\Delta_t$ in \eqref{eq:defn-Delta}, we can rewrite the update equation~\eqref{eqn:update} as
\begin{equation}\label{eqn:update-with-Delta}
  \myX_{t+1} = \myX_t - \eta(X_t X_t^\T-\Mtruth)X_t(X_t^\T X_t+\lambda I)^{-1} 
  + \eta\Delta_t X_t(X_t^\T X_t+\lambda I)^{-1}.
\end{equation}
Multiplying both sides of \eqref{eqn:update-with-Delta} 
by $\Utruth^\T$ on the left, we obtain
\begin{align}
  \Signal_{t+1} & = \Signal_t-\eta \Signal_t \myX_t^\T \myX_t(\myX_t^\T \myX_t+\lambda I)^{-1}
  + \eta\SigmaTruth^2 \Signal_t(\myX_t^\T \myX_t+\lambda I)^{-1}
  + \eta \Utruth^\T\Delta_t \myX_t(\myX_t^\T \myX_t+\lambda I)^{-1}
  \nonumber\\
  & = (1-\eta) \Signal_t 
  + \eta(\SigmaTruth^2+\lambda I+\Utruth^\T\Delta_t \Utruth)\Signal_t(\myX_t^\T \myX_t+\lambda I)^{-1} 
  + \eta \Utruth^\T\Delta_t\UperpTruth \Noise_t(\myX_t^\T \myX_t+\lambda I)^{-1}.
  \label{eqn:S-update-decomp}
\end{align}
Similarly, multiplying both sides of \eqref{eqn:update-with-Delta} by $\UperpTruth^\T$, we obtain
\begin{align}
  \Noise_{t+1}
  & = \Noise_t\left(I-\eta \myX_t^\T \myX_t(\myX_t^\T \myX_t+\lambda I)^{-1}\right)
      + \eta \UperpTruth^\T\Delta_t \myX_t(\myX_t^\T \myX_t+\lambda I)^{-1}
  \nonumber\\
  & = (1-\eta)\Noise_t 
  + \eta\lambda \Noise_t(\myX_t^\T \myX_t+\lambda I)^{-1}
  + \eta \UperpTruth^\T\Delta_t \Utruth \Signal_t(\myX_t^\T \myX_t+\lambda I)^{-1}
  + \eta \UperpTruth^\T\Delta_t \UperpTruth \Noise_t(\myX_t^\T \myX_t+\lambda I)^{-1}.
  \label{eqn:N-update-decomp}
\end{align}

These expressions motivate the need to study the terms $\Signal_t(\myX_t^\T \myX_t+\lambda I)^{-1}$ and $\Noise_t(\myX_t^\T \myX_t+\lambda I)^{-1}$, 
which we formalize in the following lemma. 

\begin{lemma}\label{lem:scaled-S-N-approx}
Under the same setting as Lemma~\ref{lem:update-approx},  we have 
\begin{subequations}
\begin{align}
  \Signal_t(\myX_t^\T \myX_t+\lambda I)^{-1} 
  & = (I+\Err{\ref{prop:rearranged-inv-approx}}_t)(\EssS_t\EssS_t^\T+\lambda I)^{-1}\EssS_t \SVDV_t^\T
  + \Err{\ref{lem:scaled-S-N-approx}.a}_t,
  \label{eqn:scaled-S-approx}\\
  \Noise_t(\myX_t^\T \myX_t+\lambda I)^{-1}
  & = \Misalign_t\EssS_t^{-1}
  (I+\Err{\ref{prop:rearranged-inv-approx}}_t)
  (\EssS_t\EssS_t^\T+\lambda I)^{-1}\EssS_t \SVDV_t^\T 
  + \lambda^{-1}\Overpar_t \Vperp{t}^\T
  + \Err{\ref{lem:scaled-S-N-approx}.b}_t,
  \label{eqn:scaled-N-approx}
\end{align}
\end{subequations}
where $\Err{\ref{prop:rearranged-inv-approx}}_t$ is given in (\ref{eqn:Err-S0-def}), and the error terms $\Err{\ref{lem:scaled-S-N-approx}.a}_t$, $\Err{\ref{lem:scaled-S-N-approx}.b}_t$ can be expressed as

\begin{subequations}
\begin{align}
  \Err{\ref{lem:scaled-S-N-approx}.a}_t &= 
  \kappa\lambda^{-1}\|\Xtruth\|^{-1}
  \|\Overpar_t\|Q_1(\Misalign_t\EssS_t^{-1}\SigmaTruth)^\T Q_2, 
  \\
  \Err{\ref{lem:scaled-S-N-approx}.b}_t &= \left(\Misalign_t(\EssS_t^\T\EssS_t+\Misalign_t^\T\Misalign_t+\lambda I)^{-1}\SVDV_t^\T
    +\lambda^{-1}\Overpar_t \Vperp{t}^\T\right)
  \Err{\ref{prop:pre-approx0}.b}_t
  \nonumber\\
  &= \lambda^{-1}(\|\Misalign_t\|Q_3+\|\Overpar_t\|Q_4)\Err{\ref{prop:pre-approx0}.b}_t.
\end{align}
\end{subequations}
for some matrices $\{Q_i\}_{1\leq i \leq 4}$ with spectral norm bounded by $2$, and $\Err{\ref{prop:pre-approx0}.b}_t$ defined in (\ref{eqn:Err0'}).
\end{lemma}

\begin{proof}
To begin, combining Lemma~\ref{prop:pre-approx0} and the discussion thereafter (cf. \eqref{eqn:scale-factor-approx}--\eqref{eqn:scale-factor-approx-2})
and the fact that $\EssS_t = \Signal_t\SVDV_t$, we have
for some matrix $Q$ with $\|Q\|\le 2$ that
\begin{align}
  \Signal_t(\myX_t^\T \myX_t+\lambda I)^{-1} 
  & = \EssS_t(\EssS_t^\T\EssS_t+\Misalign_t^\T\Misalign_t+\lambda I)^{-1}\SVDV_t^\T
  \left(I+\Err{\ref{prop:pre-approx0}.b}_t\right)
  \nonumber\\
  & = \EssS_t(\EssS_t^\T\EssS_t+\Misalign_t^\T\Misalign_t+\lambda I)^{-1} \SVDV_t^\T 
  + \EssS_t(\EssS_t^\T\EssS_t+\Misalign_t^\T\Misalign_t+\lambda I)^{-1}
  \lambda^{-1}\Misalign_t^\T\Overpar_t Q
  \nonumber\\
  & = (\EssS_t\EssS_t^\T+\EssS_t\Misalign_t^\T\Misalign_t\EssS_t^{-1}+\lambda I)^{-1}
  \EssS_t \SVDV_t^\T
  \nonumber\\
  &\phantom{=} 
  + \EssS_t(\EssS_t^\T\EssS_t+\Misalign_t^\T\Misalign_t+\lambda I)^{-1}\EssS_t^\T
  (\Misalign_t\EssS_t^{-1})^\T (\Overpar_t/\lambda)Q.
  \label{eqn:scaled-S-approx-tmp}
\end{align}
Note that the condition of Lemma~\ref{prop:pre-approx0} can be verified as follows: since 
\begin{align*}
\|\Overpar_t\| &\leq C_{\ref{lem:p1.5}.b}^{-C_{\ref{lem:p1.5}.b}} \kappa^{-3} \cdot \|\Xtruth\| \cdot \smin\big((\SigmaTruth^2 + \lambda I)^{-1/2}\big) \cdot \|\EssS_t\| \leq C_{\ref{lem:p1.5}.b}^{-C_{\ref{lem:p1.5}.b}}C_{\ref{lem:p1.5}.a}\smin(\Xtruth),\\
\|\Misalign_t\| &\leq \|\Misalign_t\EssS_t^{-1}\SigmaTruth\| \cdot \|\SigmaTruth^{-1}\| \cdot \|\EssS_t\| \leq c_{\ref{lem:p1.5}}\kappa^{-C_\delta/2}\|\Xtruth\| \cdot \frac{C_{\ref{lem:p1.5}.a}\kappa^3\|\Xtruth\|}{\smin(\Xtruth)} \leq c_{\ref{lem:p1.5}}C_{\ref{lem:p1.5}.a}\smin(\Xtruth)
\end{align*}
provided $C_\delta \geq 6$, the bounds $c_{\ref{lem:p1.5}} \lesssim c_\delta/c_\lambda$ and $C_{\ref{lem:p1.5}.a} \lesssim c_\lambda^{-1/2}$ imply that when we choose $C_\alpha$ to be large enough (depending on $c_\lambda$, $c_\delta$), 
\begin{align*}
2\|\Misalign_t\|\|\Overpar_t\| \vee \|\Overpar_t\|^2 \leq \lambda/4,
\end{align*}
as desired.

Now the first term in (\ref{eqn:scaled-S-approx-tmp}) can be handled by invoking Lemma~\ref{prop:rearranged-inv-approx}, since its condition is verified by $\|\Misalign_t\EssS_t^{-1}\SigmaTruth\| \leq c_{\ref{lem:p1.5}}\kappa^{-(C_\delta/2-1)}\smin(\Xtruth) \leq \smin(\Xtruth)/16$ provided $C_\delta \geq 2$ and $c_{\ref{lem:p1.5}} \leq 1/16$ by choosing $c_\delta$ sufficiently small (depending on $c_\lambda$). Namely,
\begin{align*}
(\EssS_t\EssS_t^\T+\EssS_t\Misalign_t^\T\Misalign_t\EssS_t^{-1}+\lambda I)^{-1}
  \EssS_t \SVDV_t^\T = (I+\Err{\ref{prop:rearranged-inv-approx}}_t)(\EssS_t\EssS_t^\T+\lambda I)^{-1}\EssS_t \SVDV_t^\T.
\end{align*}
For the second term, by noting that 
\[
  \|\EssS_t(\EssS_t^\T\EssS_t+\Misalign_t^\T\Misalign_t+\lambda I)^{-1}\EssS_t^\T\|
  \le \|\EssS_t(\EssS_t^\T\EssS_t+\lambda I)^{-1}\EssS_t^\T\|\le 1,
\]
it can be expressed as
\[\lambda^{-1}\|\Overpar_t\|\EssS_t(\EssS_t^\T\EssS_t+\lambda I)^{-1}\EssS_t^\T(\Misalign_t\EssS_t^{-1})^\T (\Overpar_t/\|\Overpar_t\|)Q
=\kappa\lambda^{-1}\|\Xtruth\|^{-1}
\|\Overpar_t\|Q_1(\Misalign_t\EssS_t^{-1}\SigmaTruth)^\T Q_2\]
for $Q_1 = \EssS_t(\EssS_t^\T\EssS_t+\lambda I)^{-1}\EssS_t^\T \cdot \kappa^{-1}\|\Xtruth\|\SigmaTruth^{-1}$ with $\|Q_1\|\le 1$ and 
$Q_2=(\Overpar_t/\|\Overpar_t\|)Q$ which satisfies $\|Q_2\|\le \|Q\|\le 2$.
Applying the above two bounds to (\ref{eqn:scaled-S-approx-tmp}) 
 yields \eqref{eqn:scaled-S-approx}.

Similarly, moving to \eqref{eqn:scaled-N-approx}, it follows that
\begin{align}
  \Noise_t(\myX_t^\T \myX_t+\lambda I)^{-1}
  = & \left(\Misalign_t(\EssS_t^\T\EssS_t+\Misalign_t^\T\Misalign_t+\lambda I)^{-1}\SVDV_t^\T
  +\lambda^{-1}\Overpar_t \Vperp{t}^\T\right)
  \left(I+\Err{\ref{prop:pre-approx0}.b}_t\right)
  \nonumber\\
  = & \Misalign_t(\EssS_t^\T\EssS_t+\Misalign_t^\T\Misalign_t+\lambda I)^{-1}\SVDV_t^\T
  + \lambda^{-1}\Overpar_t \Vperp{t}^\T
  + \Err{\ref{lem:scaled-S-N-approx}.b}_t,
  \label{eqn:scaled-N-approx-tmp}
\end{align}
where we have
\begin{align*}
  \Err{\ref{lem:scaled-S-N-approx}.b}_t &= \left(\Misalign_t(\EssS_t^\T\EssS_t+\Misalign_t^\T\Misalign_t+\lambda I)^{-1}\SVDV_t^\T
    +\lambda^{-1}\Overpar_t \Vperp{t}^\T\right)
  \Err{\ref{prop:pre-approx0}.b}_t
  \\
  &= \lambda^{-1}(\|\Misalign_t\|Q_3+\|\Overpar_t\|Q_4)\Err{\ref{prop:pre-approx0}.b}_t
\end{align*}
for some matrices $Q_3, Q_4$ with $\|Q_3\|, \|Q_4\|\le 1$. 
In the last line we used $\|(\EssS_t^\T\EssS_t+\Misalign_t^\T\Misalign_t+\lambda I)^{-1}\|\le\lambda^{-1}$. For the first term of \eqref{eqn:scaled-N-approx-tmp}, we use Lemma~\ref{prop:rearranged-inv-approx} and obtain
\begin{align*}
\Misalign_t(\EssS_t^\T\EssS_t+\Misalign_t^\T\Misalign_t+\lambda I)^{-1}\SVDV_t^\T &= \Misalign_t\EssS_t^{-1}(\EssS_t\EssS_t^\T+\EssS_t\Misalign_t^\T\Misalign_t\EssS_t^{-1}+\lambda I)^{-1}
  \EssS_t \SVDV_t^\T\\
  &= \Misalign_t\EssS_t^{-1}(I+\Err{\ref{prop:rearranged-inv-approx}}_t)(\EssS_t\EssS_t^\T+\lambda I)^{-1}\EssS_t \SVDV_t^\T.
\end{align*}
This yields the representation in \eqref{eqn:scaled-N-approx}.
\end{proof}

\subsubsection{Step 3: proofs associated with $\Signal_{t+1}$. }
With the help of Lemma~\ref{lem:scaled-S-N-approx}, we are ready to prove~\eqref{eqn:S-update-approx} and
 the associated norm bounds~\eqref{eqn:Err-1-bound}, \eqref{eqn:Err-2-bound}, and~\eqref{eqn:Err-2-bound-spectral}. 
To begin with, we plug~\eqref{eqn:scaled-S-approx}, \eqref{eqn:scaled-N-approx} 
into~\eqref{eqn:S-update-decomp} and use $\Signal_t = \EssS_t\SVDV_t^\T$ to obtain
\begin{equation*}
  \Signal_{t+1} = \left((1-\eta)I + \eta(\SigmaTruth^2+\lambda I+\Err{a}_t)
    (\EssS_t\EssS_t^\T +\lambda I)^{-1}
  \right)\EssS_t \SVDV_t^\T + \eta\Err{b}_t,
\end{equation*}
where the error terms $\Err{a}_t $ and $\Err{b}_t $ are
\begin{align*}
  \Err{a}_t &\defeq \Utruth^\T\Delta_t \Utruth+
  (\SigmaTruth^2+\Utruth^\T\Delta_t \Utruth+\lambda I)\Err{\ref{prop:rearranged-inv-approx}}_t 
  + \Utruth^\T\Delta_t \UperpTruth\Misalign_t\EssS_t^{-1}
  (I+\Err{\ref{prop:rearranged-inv-approx}}_t),
  \\
  \Err{b}_t &\defeq (\SigmaTruth^2+\Utruth^\T\Delta_t \Utruth+\lambda I)\Err{\ref{lem:scaled-S-N-approx}.a}_t 
  + \Utruth^\T\Delta_t \UperpTruth(\lambda^{-1}\Overpar_t \Vperp{t}^\T + \Err{\ref{lem:scaled-S-N-approx}.b}_t).
\end{align*}
This establishes the identity~\eqref{eqn:S-update-approx}. 
To control $\uinorm{\Err{a}_t}$, we observe that
\begin{align*}
  \uinorm{\Err{a}_t}
  &\le \uinorm{\Utruth^\T\Delta_t} + \|\SigmaTruth^2+\Utruth^\T\Delta_t \Utruth+\lambda I\|\cdot\uinorm{\Err{\ref{prop:rearranged-inv-approx}}_t} 
  + \uinorm{\Utruth^\T\Delta_t} \cdot \|\Misalign_t\EssS_t^{-1}\SigmaTruth\| \cdot 
  \|\SigmaTruth^{-1}\| \cdot (1+\|\Err{\ref{prop:rearranged-inv-approx}}_t\|)
  \\
  &\le \left(1+c_{\ref{lem:Delta-bound}}\kappa^{-2C_\delta/3}+c_\lambda\right)
  \|\Xtruth\|^2 \cdot \uinorm{\Err{\ref{prop:rearranged-inv-approx}}_t}
  + \uinorm{\Utruth^\T\Delta_t} 
  + c_{\ref{lem:p1.5}}\kappa^{-C_\delta/2}\|\Xtruth\| \cdot \smin^{-1}(\Xtruth) \cdot (1+\|\Err{\ref{prop:rearranged-inv-approx}}_t\|) \cdot 
  \uinorm{\Utruth^\T\Delta_t}
  \\
  &\le 2\|\Xtruth\|^2 \cdot \uinorm{\Err{\ref{prop:rearranged-inv-approx}}_t}
  + \left(1 + c_{\ref{lem:p1.5}}(1+\|\Err{\ref{prop:rearranged-inv-approx}}_t\|)\right)\uinorm{\Utruth^\T\Delta_t},
\end{align*}
where the second line follows from Lemma~\ref{lem:Delta-bound} 
and Equations \eqref{eqn:lambda-cond}, \eqref{eqn:p1.5-3};
the last line holds since $c_{\ref{lem:Delta-bound}}, c_\lambda$ are sufficiently small and $C_\delta$ is sufficiently large.
Now we invoke the bound \eqref{eqn:Err-S0-bound} in Lemma~\ref{prop:rearranged-inv-approx} 
to see
\begin{align*}
  \uinorm{\Err{\ref{prop:rearranged-inv-approx}}_t}\le 2\kappa^2\|\Xtruth\|^{-2}
  \|\Misalign_t\EssS_t^{-1}\SigmaTruth\|\uinorm{\Misalign_t\EssS_t^{-1}\SigmaTruth}
  &\le 2c_{\ref{lem:p1.5}}\kappa^2\kappa^{-C_\delta/2}\|\Xtruth\|^{-1}\uinorm{\Misalign_t\EssS_t^{-1}\SigmaTruth}
  \\
  &\le 2c_{\ref{lem:p1.5}}\kappa^{-6}\|\Xtruth\|^{-1}\uinorm{\Misalign_t\EssS_t^{-1}\SigmaTruth},
\end{align*}
where the last line follows again by choosing sufficiently large $C_\delta$.
Furthermore, since $\|\Misalign_t\EssS_t^{-1}\SigmaTruth\|\le c_{\ref{lem:p1.5}}\kappa^{-C_\delta/2} \|\Xtruth\|$ for small enough $c_{\ref{lem:p1.5}}$, we obtain $\|\Err{\ref{prop:rearranged-inv-approx}}_t\|\le 1$. Combining these inequalities yields the claimed bound
\[
  \uinorm{\Err{a}_t}\le 2c_{\ref{lem:p1.5}}
  \kappa^{-4}\|\Xtruth\|\cdot\uinorm{\Misalign_t\EssS_t^{-1}\SigmaTruth}
    +  2\uinorm{\Utruth^\T\Delta_t}.
\]

The bound of $\uinorm{\Err{b}_t}$ and $\|\Err{b}_t\|$ can be proved in a similar way, 
utilizing the bound for $\|\Overpar_t\|$ in \eqref{eqn:p1.5-2-var}. 
In fact, a computation similar to the above shows 
\begin{align*}
  \uinorm{\Err{b}_t}
  &\le 2\|\Xtruth\|^2 \cdot \uinorm{\Err{\ref{lem:scaled-S-N-approx}.a}}
  + \lambda^{-1}\|\Delta_t\|\cdot\uinorm{\Overpar_t}
  + \|\Delta_t\|\cdot\uinorm{\Err{\ref{lem:scaled-S-N-approx}.b}}
  \\
  &\le 2\kappa\lambda^{-1} \cdot \|\Xtruth\| \cdot \|\Overpar_t\| \cdot \|Q_1\| \cdot \|Q_2\|
  \cdot\uinorm{\Misalign_t\EssS_t^{-1}\SigmaTruth}
  + 100c_\lambda^{-1}\smin^{-1}(\Mtruth)c_{\ref{lem:Delta-bound}}\kappa^{-2C_\delta/3}
  \|\Xtruth\|^2 \cdot \uinorm{\Overpar_t} 
  \\
  &\phantom{\le{}} 
  + 8\lambda^{-2}c_{\ref{lem:Delta-bound}}\kappa^{-2C_\delta/3}
  (\|\Misalign_t\|+\|\Overpar_t\|)\|\Misalign_t\| \cdot \uinorm{\Overpar_t}
  \\
  &\le 800\kappa^7 c_\lambda^{-1}\|\Xtruth\|^{-1}\|\Overpar_t\| \cdot
  \uinorm{\Misalign_t\EssS_t^{-1}\SigmaTruth}
  + \frac1{48(\Cmax+1)\kappa}\uinorm{\Overpar_t}.
\end{align*}
Here, $\Cmax$ is the constant given by Lemma \ref{lem:p1.5}. Similarly, we have
\begin{align*}
\|\Err{b}_t\| \leq 800\kappa^7 c_\lambda^{-1}\|\Xtruth\|^{-1}\|\Overpar_t\| \cdot
  \|\Misalign_t\EssS_t^{-1}\SigmaTruth\|
  + \frac1{48(\Cmax+1)\kappa}\|\Overpar_t\|.
\end{align*}
The bound \eqref{eqn:Err-2-bound-spectral} now follows directly from the bound of $\|\Misalign_t\EssS_t^{-1}\SigmaTruth\|$ in Lemma \ref{lem:p1.5}, provided $c_\delta$ is sufficiently small and $C_\delta$ is sufficiently large.
To prove \eqref{eqn:Err-2-bound}, we note that 
\begin{equation}\label{eq:fact-uninorm}
\uinorm{A}\le n\|A\|
\end{equation}
for any unitarily invariant norm $\uinorm{\cdot}$ and real matrix $A\in\mathbb{R}^{p\times q}$ with $p\vee q \leq n$ (which can be easily verified 
when $\uinorm{\cdot}=\|\cdot\|$ or $\|\cdot\|_{\fro}$).
Thus 
\begin{equation}
  \label{eqn:Err-2-bound-raw}
  \uinorm{\Err{b}_t}
  \le \left(800\kappa^7 c_\lambda^{-1}c_{\ref{lem:p1.5}}\kappa^{-C_\delta/2}
  +\frac{1}{24(\Cmax+1)\kappa}\right)n\|\Overpar_t\|
  \le\left(\frac{\|\Overpar_t\|}{\smin(\EssS_t)}\right)^{3/4}\smin(\EssS_t)
\end{equation}
where the last inequality follows from the control of $\|\Overpar_t\|$ 
given by~\eqref{eqn:p1.5-1-var} provided $c_{\ref{lem:p1.5}}$ is sufficiently small 
and $C_{\ref{lem:p1.5}.b}$ therein is sufficiently large. This establishes the first inequality in \eqref{eqn:Err-2-bound}, and the second inequality therein follows directly from  \eqref{eqn:p1.5-1-var}.  

\subsubsection{Step 4: proofs associated with $\Misalign_{t+1}$. } 
Now we move on to prove the identity \eqref{eqn:N-update-approx}, and the norm 
controls~\eqref{eqn:Err-3-bound}, \eqref{eqn:Err-4-bound}, \eqref{eqn:Err-5-bound}, and~\eqref{eqn:Err-4-bound-spectral} associated with 
the misalignment term $\Misalign_{t+1}$.
Plugging \eqref{eqn:scaled-S-approx}, \eqref{eqn:scaled-N-approx} into \eqref{eqn:N-update-decomp} 
and using the decomposition $\Noise_t=\Misalign_t \SVDV_t^\T + \Overpar_t \Vperp{t}^\T$, 
we have 
\begin{align*}
  \Noise_{t+1} 
  & = \Misalign_t\EssS_t^{-1}\left((1-\eta)\EssS_t\EssS_t^\T+\lambda I
    + \eta\Err{c}_t\right)(\EssS_t\EssS_t^\T+\lambda I)^{-1}\EssS_t \SVDV_t^\T 
  \\
  &\phantom{=} 
  + \eta \Err{e}_t(\EssS_t\EssS_t^\T+\lambda I)^{-1}\EssS_t \SVDV_t^\T 
  + \Overpar_t \Vperp{t}^\T + \eta\Err{d}_t,
\end{align*}
where the error terms are defined to be 
\begin{align*}
  \Err{c}_t &\defeq \lambda\Err{\ref{prop:rearranged-inv-approx}}_t, 
  \\
  \Err{d}_t &\defeq (\lambda I + \UperpTruth^\T\Delta_t\UperpTruth) \Err{\ref{lem:scaled-S-N-approx}.b}_t 
  + \lambda^{-1}\UperpTruth^\T\Delta_t\UperpTruth \Overpar_t \Vperp{t}^\T
  + \UperpTruth^\T\Delta_t \Utruth\Err{\ref{lem:scaled-S-N-approx}.a}_t,
  \\
  \Err{e}_t &\defeq \UperpTruth^\T\Delta_t \Utruth(I+\Err{\ref{prop:rearranged-inv-approx}}_t) 
  + \UperpTruth^\T\Delta_t\UperpTruth \Misalign_t\EssS_t^{-1}
  (I+\Err{\ref{prop:rearranged-inv-approx}}_t).
\end{align*}
This establishes the decomposition~\eqref{eqn:N-update-approx}. 
The remaining norm controls follow from the expressions above and similar computation 
as we have done for $\Signal_{t+1}$. For the sake of brevity, we omit the details.

\subsection{Proof of Lemma \ref{lem:S-surrogate}}
\label{apd:sec:S-surrogate}
Use the identity~\eqref{eqn:S-update-approx} in Lemma~\ref{lem:update-approx} and the fact that $\SVDV_t$ and $\Vperp{t}$ have orthogonal columns to obtain
\begin{align}
  \Signal_{t+1}\SVDV_t+\Signal_{t+1}\Vperp{t} Q
  & = \left((1-\eta)I+\eta(\SigmaTruth^2+\lambda I+\Err{a}_t)
    (\EssS_t\EssS_t^\T +\lambda I)^{-1}\right)\EssS_t 
  + \eta\Err{b}_t(V_t + \Vperp{t} Q)
  \nonumber\\
  & = (I+\eta\Err{\ref{lem:S-surrogate}}_t)\left( (1-\eta)I +\eta(\SigmaTruth^2+\lambda I)
    (\EssS_t\EssS_t^\T +\lambda I)^{-1}\right)\EssS_t
  \nonumber\\
  & = (I+\eta\Err{\ref{lem:S-surrogate}}_t)\big( (1-\eta) \EssS_t \EssS_t^\T + \lambda I+\eta\SigmaTruth^2 \big)
  (\EssS_t\EssS_t^\T +\lambda I)^{-1}\EssS_t,
  \label{eqn:tilde-S-update}
\end{align}
where $\Err{\ref{lem:S-surrogate}}_t$ is defined to be
\begin{align*}
  \Err{\ref{lem:S-surrogate}}_t & \defeq \left(\Err{a}_t(\EssS_t\EssS_t^\T+\lambda I)^{-1}
    + \Err{b}_t(V_t + \Vperp{t} Q)\EssS_t^{-1}\right)
  \left( ( 1-\eta) I+\eta(\SigmaTruth^2+\lambda I)(\EssS_t\EssS_t^\T+\lambda I)^{-1}\right)^{-1}
  \nonumber\\
  & = \Err{a}_t\left((1-\eta)(\EssS_t\EssS_t^\T+\lambda I)+\eta(\SigmaTruth^2+\lambda I)\right)^{-1}
  \nonumber\\
  &\phantom{=} + \Err{b}_t(V_t + \Vperp{t} Q)\EssS_t^{-1}
  \left( (1-\eta) I +\eta(\SigmaTruth^2+\lambda I)(\EssS_t\EssS_t^\T+\lambda I)^{-1}\right)^{-1}
  \\
  &=: T_1 + T_2,
\end{align*}
where the invertibility of $\EssS_t$ follows from Lemma~\ref{lem:p1.5}, 
and the invertibility of $(1-\eta) I +\eta(\SigmaTruth^2+\lambda I)(\EssS_t\EssS_t^\T+\lambda I)^{-1}$ 
follows from \eqref{eqn:AB-bound}.

Since $(1-\eta)(\EssS_t\EssS_t^\T+\lambda I)+\eta(\SigmaTruth^2+\lambda I)\succeq \lambda I$ 
and $\lambda\ge\frac{1}{100}c_\lambda \smin(\Mtruth)$ by \eqref{eqn:lambda-cond}, 
we have
\begin{align*}
  \|T_1\|\le \lambda^{-1}\|\Err{a}_t\|\le 100 c_\lambda^{-1} \smin^{-1}(\Mtruth) \|\Err{a}_t\|.
\end{align*}
In view of the bound~\eqref{eqn:Err-1-bound} on $ \|\Err{a}_t\|$ in Lemma~\ref{lem:update-approx}, 
we further have 
\begin{align*}
\|T_1\| &\le 100c_\lambda^{-1}\smin^{-2}(\Xtruth)(\kappa^{-4}\|\Xtruth\|\cdot\|\Misalign_t\EssS_t^{-1}\SigmaTruth\|+\|\Delta_t\|)
  \\ 
  &\le 100c_\lambda^{-1}\kappa^2\|\Xtruth\|^{-2}(\kappa^{-4}c_{\ref{lem:p1.5}}\kappa^{-C_\delta/2}+c_{\ref{lem:Delta-bound}}\kappa^{-2C_\delta/3})\|\Xtruth\|^2
  \\
  &\le\frac{1}{400(C_{\ref{lem:p1.5}.a}+1)^4\kappa^5},
\end{align*}
where the second inequality follows from \eqref{eqn:p1.5-3} in Lemma \ref{lem:p1.5}
and Lemma~\ref{lem:Delta-bound}, 
and the last inequality holds as long as $c_{\ref{lem:p1.5}}$ and $c_{\ref{lem:Delta-bound}}$ are sufficiently small 
and  $C_\delta$ is sufficiently large (by first fixing $c_\lambda$ and then choosing $c_\delta$ to be sufficiently small). 

The term $T_2$ can be controlled in a similar way. Since $\|A B\| \leq \|A\| \cdot \|B\|$, one has
\begin{align*}
  \|T_2\|&\le \|\Err{b}_t\| \cdot (\|\SVDV_t\|+\|\Vperp{t}\|\|Q\|) \cdot \|\EssS_t^{-1}\| \cdot
  \smin^{-1}\left((1-\eta)I+\eta(\SigmaTruth^2+\lambda I)(\EssS_t\EssS_t^\T+\lambda I)^{-1}\right)
  \\
  &\stackrel{\mathrm{(i)}}{\le} 3\|\Err{b}_t\| \cdot \smin^{-1}(\EssS_t) \cdot \frac{\kappa}{1-\eta}
  \stackrel{\mathrm{(ii)}}{\le} 6\kappa\left(\frac{\|\Overpar_t\|}{\smin(\EssS_t)}\right)^{3/4} 
  \stackrel{\mathrm{(iii)}}{\le} \frac{1}{400(C_{\ref{lem:p1.5}.a}+1)^4\kappa^5}.
\end{align*}
Here, (i) follows from the bound~\eqref{eqn:AB-bound} and the facts that 
$\|\SVDV_t\| \vee \|\Vperp{t}\|\le 1$, $\|Q\|\le 2$; (ii)
arises from the control~\eqref{eqn:Err-2-bound} on $\|\Err{b}_t\|$ in Lemma~\ref{lem:update-approx} 
as well as the condition $\eta\le c_\eta\le 1/2$;
and (iii) follows from the implication~\eqref{eqn:p1.5-1-var} of Lemma~\ref{lem:p1.5}.

The proof is completed by summing up the bounds on $\| T_1 \|$ and $ \| T_2 \|$.


\subsection{Proof of Lemma \ref{lem:NS-surrogate}}
\label{apd:sec:NS-surrogate}
Similar to the proof of Lemma~\ref{lem:S-surrogate}, we can use the identity~\eqref{eqn:N-update-approx} in Lemma~\ref{lem:update-approx} 
and the fact that $\SVDV_t$ and $\Vperp{t}$ have orthogonal columns to obtain
\begin{align}
  \Noise_{t+1}\SVDV_t+\Noise_{t+1}\Vperp{t} Q
  & = \Misalign_t\EssS_t^{-1}\big((1-\eta)\EssS_t\EssS_t^\T+\lambda I
    +\eta\Err{c}_t\big)
  (\EssS_t\EssS_t^\T+\lambda I)^{-1}\EssS_t + \eta \Err{\ref{lem:NS-surrogate}.c}_t
  \nonumber\\
  & = \Misalign_t\EssS_t^{-1}(I+\eta\Err{\ref{lem:NS-surrogate}.a}_t)\big((1-\eta)\EssS_t\EssS_t^\T+\lambda I\big)
  (\EssS_t\EssS_t^\T+\lambda I)^{-1}\EssS_t 
  + \eta \Err{\ref{lem:NS-surrogate}.c}_t,
  \label{eqn:tilde-N-update}
\end{align}
where the error terms are defined to be
\begin{align}
  \Err{\ref{lem:NS-surrogate}.c}_t & \defeq  \Err{e}_t(\EssS_t\EssS_t^\T+\lambda I)^{-1}\EssS_t 
  + \eta^{-1}\Overpar_t Q + \Err{d}_t(V_t+\Vperp{t} Q), 
  \\
  \Err{\ref{lem:NS-surrogate}.a}_t & \defeq  \Err{c}_t\big((1-\eta)\EssS_t\EssS_t^\T+\lambda I\big)^{-1}.
\end{align}
Combine~\eqref{eqn:tilde-N-update} and \eqref{eqn:tilde-S-update} to arrive at 
\begin{align}
  &(\Noise_{t+1}V_t+\Noise_{t+1} \Vperp{t} Q)(\Signal_{t+1}\SVDV_t + \Signal_{t+1}\Vperp{t} Q)^{-1}
  \nonumber\\
  &\quad = \Misalign_t\EssS_t^{-1}(I+\eta\Err{\ref{lem:NS-surrogate}.a}_t)
  \big((1-\eta)\EssS_t\EssS_t^\T +\lambda I\big)
  \big((1-\eta)\EssS_t\EssS_t^\T+\lambda I+\eta\SigmaTruth^2\big)^{-1}
  (I+\eta\Err{\ref{lem:S-surrogate}}_t)^{-1}
  +\eta\Err{\ref{lem:NS-surrogate}.b}_t,
\end{align}
where, using 
$$(\EssS_t\EssS_t^\T +\lambda I)
  \big((1-\eta)\EssS_t\EssS_t^\T+\lambda I+\eta\SigmaTruth^2\big)^{-1} = \big((1-\eta)I+\eta(\SigmaTruth^2+\lambda I)(\EssS_t\EssS_t^\T+\lambda I)^{-1}\big)^{-1},$$ 
  we have 
\begin{align*}
  \Err{\ref{lem:NS-surrogate}.b}_t
  &\defeq\Err{\ref{lem:NS-surrogate}.c}_t\EssS_t^{-1}(\EssS_t\EssS_t^\T +\lambda I)
  \big((1-\eta)\EssS_t\EssS_t^\T+\lambda I+\eta\SigmaTruth^2\big)^{-1}
  (I+\eta\Err{\ref{lem:S-surrogate}}_t)^{-1}
  \\
  &=\Err{e}_t\big((1-\eta)\EssS_t\EssS_t^\T + \lambda I + \eta\SigmaTruth^2\big)^{-1}
  (I + \eta\Err{\ref{lem:S-surrogate}}_t)^{-1} 
  \\
  &\phantom{={}}+ \eta^{-1}\Overpar_t Q\EssS_t^{-1}
  \left((1-\eta)I + \eta(\SigmaTruth^2 + \lambda I)(\EssS_t\EssS_t^\T + \lambda I)^{-1}\right)^{-1}
  (I + \eta\Err{\ref{lem:S-surrogate}}_t)^{-1} 
  \\
  &\phantom{={}}+\Err{d}_t (\SVDV_t + \Vperp{t} Q) \EssS_t^{-1}
  \left((1-\eta)I + \eta(\SigmaTruth^2 + \lambda I)(\EssS_t\EssS_t^\T + \lambda I)^{-1}\right)^{-1}
  (I+\eta\Err{\ref{lem:S-surrogate}}_t)^{-1} \\
  &\eqqcolon T_1 + T_2 + T_3.
\end{align*}

It remains to bound $\|\Err{\ref{lem:NS-surrogate}.a}\|$ 
and $\uinorm{\Err{\ref{lem:NS-surrogate}.b}}$. 
By \eqref{eqn:Err-3-bound}, we have
\begin{align*}
  \|\Err{\ref{lem:NS-surrogate}.a}\|
  \le \lambda^{-1}\|\Err{c}_t\|
  &\le 100c_\lambda^{-1}\smin^{-2}(\Xtruth)
  \cdot \kappa^{-4}\|\Xtruth\|\|\Misalign_t\EssS_t^{-1}\SigmaTruth\|
  \\
  &\le 100c_\lambda^{-1}c_{\ref{lem:p1.5}}\kappa^{-2}\kappa^{-C_\delta/2}
  \\
  &\le\frac{1}{200(C_{\ref{lem:p1.5}.a}+1)^4\kappa^5},
\end{align*}
where the penultimate inequality follows from \eqref{eqn:p1.5-3} 
and the last inequality holds with the proviso that $c_{\ref{lem:p1.5}}$ is sufficiently small 
and $C_\delta$ is sufficiently large.

Now we move to bound $\uinorm{\Err{\ref{lem:NS-surrogate}.b}}$. To this end, the relation 
$\|(I+\eta\Err{\ref{lem:S-surrogate}}_t)^{-1}\|\le 2$ is quite helpful. This follows from 
Lemma~\ref{lem:S-surrogate} in which we have established that $\|\Err{\ref{lem:S-surrogate}}_t\|\le 1/2$.
As a result of this relation, we obtain  
\begin{align*}
\uinorm{T_1} & \leq 2 \lambda^{-1} \uinorm{\Err{e}_t}, \\
\uinorm{T_2} & \leq 2\uinorm{\Overpar_t} \cdot \|Q\| \cdot \| \EssS_t^{-1} \|
\cdot \left\| \left( ( 1-\eta) I +\eta(\SigmaTruth^2+\lambda I)
(\EssS_t\EssS_t^\T+\lambda I)^{-1}\right)^{-1}
\right\|, 
\\
\uinorm{T_3} & \leq 2 \uinorm{\Err{d}_t}\cdot (1+\|Q\|)\cdot \|\EssS_t^{-1}\|
\cdot \left\| \left( ( 1-\eta) I +\eta(\SigmaTruth^2+\lambda I)
  (\EssS_t\EssS_t^\T+\lambda I)^{-1}\right)^{-1}
\right\|. 
\end{align*}
Similar to the control of $T_1$ in the proof of Lemma~\ref{lem:S-surrogate}, we can take the condition $\lambda\ge\frac{1}{100} \kappa^{-4} c_\lambda \smin^2(\Xtruth)$ and the bound~\eqref{eqn:Err-5-bound} collectively 
to see that
\begin{align*}
 \uinorm{T_1} \leq 400 c_\lambda^{-1}\kappa^6 \|\Xtruth\|^{-2}\uinorm{\Utruth^\T\Delta_t}
 + \frac{1}{64(C_{\ref{lem:p1.5}.a}+1)^2\kappa^4\|\Xtruth\|}\uinorm{\Misalign_t\EssS_t^{-1}\SigmaTruth}.
\end{align*}
Regarding the terms $T_2$ and $T_3$, we see from~\eqref{eqn:AB-bound} that 
\begin{align*}
\left\| \left( (1-\eta) I + \eta(\SigmaTruth^2+\lambda I)(\EssS_t\EssS_t^\T+\lambda I)^{-1}\right)^{-1} \right\|
\leq \frac{\kappa}{1-\eta} \leq 2 \kappa,
\end{align*}
as long $\eta$ is sufficiently small. 
Recalling the assumption $\|Q\|\le 2$, 
this allows us to obtain
\begin{align*}
 \uinorm{T_2} &\leq  8 \eta^{-1} \kappa \frac{\uinorm{ \Overpar_t} }{\smin(\EssS_t)}
 \leq 8 \eta^{-1} \kappa n\frac{\|\Overpar_t\|}{\smin(\EssS_t)}, \\
  \uinorm{T_3} &\leq 12 \kappa \uinorm{\Err{d}_t}/\smin(\EssS_t),
\end{align*}
where the first inequality again uses the elementary fact $\uinorm{ \Overpar_t} \leq n \|\Overpar_t\|$ in \eqref{eq:fact-uninorm}.

The desired bounds then follow from plugging in the bounds~\eqref{eqn:Err-4-bound} 
and~\eqref{eqn:p1.5-2-var}.


\section{Proofs for Phase I}\label{sec:proof-p1}
The goal of this section is to prove Lemma~\ref{lem:p1.5} in an inductive manner.
We achieve this goal in two steps. In Section~\ref{sec:phase-1-base-case}, we 
find an 
iteration number $t_1 \leq \Tcvg / 16$ such that the claim~\eqref{subeq:condition-t-1} 
is true at $t_1$. This establishes the base case. Then in Section~\ref{sec:phase-1-induction}, we prove 
the induction step, namely if the claim~\eqref{subeq:condition-t-1} holds for some iteration $t \geq t_1$, 
we aim to show that \eqref{subeq:condition-t-1} 
continues to hold for the iteration $t+1$. 
These two steps taken collectively finishes the proof of Lemma~\ref{lem:p1.5}.

\subsection{Establishing the base case: Finding a valid $t_1$}\label{sec:phase-1-base-case}
The following lemma ensures the existence of such an iteration number $t_1$.
\begin{lemma}\label{lem:p1}
  Under the same setting as Theorem \ref{thm:main}, 
  we have for some $t_1\le \Tcvg/16$ such that \eqref{eqn:p1.5-0} holds 
  and that~\eqref{subeq:condition-t-1} hold with $t=t_1$.
\end{lemma}

\noindent The rest of this subsection is devoted to the proof of this lemma.
\medskip

Define an auxiliary sequence 
\begin{align}\label{def:power_seq}
\hat{X}_{t} \defeq \Big(I+\frac{\eta}{\lambda}\opAA(\Mtruth)\Big)^t X_0,
\end{align}
which can be viewed as power iterations on the matrix $\opAA(\Mtruth)$ from the initialization $X_0$. 

In what follows, we first establish that the true iterates $\{X_t\}$ stay close to the auxiliary iterates $\{ \hat{X}_t \}$ as long as the initialization size $\alpha$ is small; see Lemma~\ref{lem:approx_power_method}. This proximity then allows us to invoke the result in~\citet{stoger2021small} (see Lemma~\ref{lem:Mahdi}) to establish Lemma~\ref{lem:p1}. For the rest of the appendices, we work on the following event given in (\ref{asp:init}): 
\begin{align*}
  \Event=\{\|G\|\le C_G\}\cap\{\smin^{-1}(\Uspec^\T G)\le (2n)^{C_G}\}.
\end{align*}

\paragraph{Step 1: controlling distance between $X_t$ and $\hat{X}_t$.}

The following lemma guarantees the closeness between the two iterates $\{X_t\}$ and $\{ \hat{X}_t \}$, with the proof deferred to Appendix~\ref{sec:proof-dist-power-method}. Recall that $C_G$ is the constant defined in the event $\mathcal{E}$ in (\ref{asp:init}), and $c_\lambda$ is the constant given in Theorem \ref{thm:main}.

\begin{lemma}\label{lem:approx_power_method}
Suppose that $\lambda \geq \frac{1}{100}c_\lambda\smin^2(\Xtruth)$. For any $\theta\in(0,1)$, there exists a large enough constant $K = K(\theta,c_\lambda, C_G) > 0$ such that the following holds: As long as $\alpha$ obeys
\begin{align}\label{ineq:alpha_cond_tmp}
\log\frac{\|\Xtruth\|}{\alpha} \geq \frac{K}{\max(\eta, \kappa^{-2})}\log(2\kappa n)\cdot\Big(1+\log\Big(1+\frac{\eta}{\lambda}\|\opAA(\Mtruth)\|\Big)\Big),
\end{align}
one has for all $t \le \frac{1}{\theta\eta}\log(\kappa n)$:
\begin{align}\label{ineq:approx_power_method}
\bigpnorm{X_t-\hat{X}_t}{} &\leq t\Big(1+\frac{\eta}{\lambda}\pnorm{\opAA(\Mtruth)}{}\Big)^t\frac{\alpha^2}{\|\Xtruth\|}.
\end{align}
Moreover, $\pnorm{X_t}{} \leq \pnorm{\Xtruth}{}$ for all such $t$.
\end{lemma}

\paragraph{Step 2: borrowing a lemma from~\citet{stoger2021small}.}

Compared to the original sequence $X_t$, the behavior of the power iterates $\hat{X}_t$ is much easier to analyze. Now that we have sufficient control over $\|X_t-\hat{X}_t\|$, it is possible to show that $X_t$ has the desired properties in Lemma~\ref{lem:p1} by first establishing the corresponding property of $\hat{X}_t$ and then invoking a standard matrix perturbation argument. Fortunately, such a strategy has been implemented by \citet{stoger2021small} and wrapped into the following helper lemma.

Denote 
\begin{align*}
s_{j} \defeq \singval_{j}\Big(I + \frac{\eta}{\lambda}\opAA(\Mtruth)\Big) = 1+ \frac{\eta}{\lambda}\singval_{j}\big(\opAA(\Mtruth)\big),
\qquad j=1,2,\ldots,n
\end{align*}
and recall that $\Uspec$ (resp. $\SVDU_{\EssX_t}$) is an orthonormal basis of the eigenspace associated with the $r_\star$ largest eigenvalues of $\opAA(\Mtruth)$ (resp. $\EssX_t$).

\begin{lemma}\label{lem:Mahdi}
There exists some small universal $c_{\ref{lem:Mahdi}} > 0$ such that the following hold. 
Assume that for some $\gamma \leq c_{\ref{lem:Mahdi}}$,
\begin{align}\label{ineq:mahdi_p1_cond1}
\pnorm{(\id-\opAA)(\Mtruth)}{} \leq \gamma\smin^2(\Xtruth),
\end{align}
and furthermore, 
\begin{align}\label{eqn:Mahdi-phi}
\phi \defeq\frac{\alpha\|G\|s_{r_\star+1}^t+\|X_t - \hat{X}_t\|}{\alpha \smin(\Uspec^\T G)s_{r_\star}^t}\le c_{\ref{lem:Mahdi}}\kappa^{-2}.
\end{align}
Then there exists some universal $C_{\ref{lem:Mahdi}} > 0$ such that the following hold: \begin{subequations}
\begin{align}
\smin(\EssS_t)&\ge\frac{\alpha}{4}\smin(\Uspec^\T G)s_{r_\star}^t, \label{ineq:mahdi_p1_1}\\
\|\Overpar_t\|&\le C_{\ref{lem:Mahdi}}\phi\alpha\smin(\Uspec^\T G)s_{r_\star}^t, \label{ineq:mahdi_p1_2}\\ 
\|\UperpTruth^\T \SVDU_{\EssX_t}\|&\le C_{\ref{lem:Mahdi}}(\gamma+\phi), \label{ineq:mahdi_p1_3}
\end{align}
where $\Ess\myX_t \defeq \myX_t \SVDV_t \in \reals^{n\times r_\star}$.
\end{subequations}
\end{lemma}
\begin{proof}[Proof of Lemma~\ref{lem:Mahdi}]
This follows from the claims of \citet[Lemma~8.5]{stoger2021small} by noting that $\pnorm{\Overpar_t}{} = \pnorm{\UperpTruth^\T X_tV_{t,\perp}}{} \leq  \pnorm{X_tV_{t,\perp}}{}$ for (\ref{ineq:mahdi_p1_2}).\footnote{The equation (31) in \citet[Lemma~8.5]{stoger2021small} is stated in a weaker form than what they actually proved, and our~\eqref{ineq:mahdi_p1_2} indeed follows from the penultimate inequality in the proof of \citet[Lemma~8.5]{stoger2021small}.}
\end{proof}

\paragraph{Step 3: completing the proof.}
Now, with the help of Lemma~\ref{lem:Mahdi}, we are ready to prove Lemma~\ref{lem:p1}. We start with verifying the two assumptions in Lemma~\ref{lem:Mahdi}.

\paragraph{Verifying assumption~\eqref{ineq:mahdi_p1_cond1}.}
By the RIP in (\ref{eqn:rip}), Lemma \ref{lem:spectral_rip}, and the condition of $\delta$ in (\ref{eqn:delta-cond}), we have 
\begin{align}\label{ineq:opAA_bound}
\bigpnorm{(\id-\opAA)(\Mtruth)}{}\leq \sqrt{r_\star}\delta\|\Mtruth\|
\leq c_\delta\kappa^{-( C_\delta -2)}\smin^2(\Xtruth)
=:\gamma \smin^2(\Xtruth).
\end{align}
Here $\gamma=c_\delta\kappa^{-( C_\delta -2)} \leq c_{\ref{lem:Mahdi}}$, as  
$ c_\delta $ is assumed to be sufficiently small.

\paragraph{Verifying assumption~\eqref{eqn:Mahdi-phi}.}
By Weyl's inequality and (\ref{ineq:opAA_bound}), we have 
\begin{align*}
  \Big|s_j-1-\frac{\eta}{\lambda}\singval_j(\Mtruth)\Big| \leq \frac{\eta}{\lambda}\bigpnorm{(\id - \opAA)(\Mtruth)}{} \leq \frac{\eta}{\lambda}c_\delta\kappa^{-(C_\delta-2)}\smin^2(\Xtruth)
  \leq \frac{100c_\delta}{c_\lambda}\eta,
\end{align*}
where the last inequality follows from the condition $\lambda \geq\frac{1}{100}\kappa^{-4} c_\lambda\smin^2(\Xtruth)$. 
Furthermore, using the condition $\lambda \leq c_\lambda\smin^2(\Xtruth)$ assumed in \eqref{eqn:lambda-cond}, 
the above bound implies that, for some $C = C(c_\lambda, c_\delta) > 0$,
\begin{subequations} 
\begin{align}
s_1 &\le 1+\frac{\eta}{\lambda}\|\Mtruth\|+\frac{100c_\delta}{c_\lambda}\eta \leq 1+C\eta\kappa^6, 
\label{eqn:s1-bound}\\
s_{r_\star} &\geq 1 + \frac{\eta}{\lambda}\smin^2(\Xtruth) 
- \frac{100c_\delta}{c_\lambda}\eta \geq 1 + \frac{\eta}{2c_\lambda}, 
\label{eqn:sr-bound-lower}\\
s_{r_\star} &\leq 1 + \frac{\eta}{\lambda}\smin^2(\Xtruth) 
+ \frac{100c_\delta}{c_\lambda}\eta \leq 1 + \frac{2\eta}{\lambda / \smin^2(\Xtruth)},
\label{eqn:sr-bound-upper}\\
s_{r_\star+1}&\le 1 + \frac{100c_\delta}{c_\lambda}\eta \leq 1 + \frac{\eta}{4c_\lambda},
\label{eqn:sr1-bound}
\end{align}
\end{subequations}
where we use the fact that $\singval_{r_\star+1}(\Mtruth)=0$, and $c_\delta \leq 1/400$. 
Consequently we have $s_{r_\star}/s_{r_\star+1}\ge 1+c'\eta$ 
for some $c' = c'(c_\lambda) > 0$, assuming $c_\eta\le c_\lambda$. 
Thus for any large constant $L>0$, 
there is some constant $c''=c''(c')>0$ such that, 
setting $L' = c''L\log(L)$ we have 
\begin{align*}
  (s_{r_\star}/s_{r_\star+1})^t\geq (L\kappa n)^{L},
  \quad\forall t\ge\frac{L'}{\eta}\log(\kappa n).
\end{align*}
On the event $\Event$ given in (\ref{asp:init}), 
we can choose $L$ large enough so that $L\ge 2C_G$, hence $\|G\|\le L$ 
and $\smin^{-1}(\Uspec^\T G)\le (2n)^{L/2}$. 
Summarizing these inequalities, we see for $t\ge\frac{L'}{\eta}\log(\kappa n)$, 
\begin{align}
\frac{\alpha\|G\|s_{r_\star+1}^t}{\alpha\smin(\Uspec^\T G)s_{r_\star}^t} 
&\leq L\smin^{-1}(\Uspec^\T G)(s_{r_\star+1}/s_{r_\star})^t
\nonumber\\
&\leq L(2n)^{L/2}(L\kappa n)^{-L}\leq (L\kappa n)^{-L/2}.
\label{eqn:verify-mahdi-lem-tmp1}
\end{align}
Furthermore, invoking Lemma~\ref{lem:approx_power_method} with $\theta=1/(2L')$ 
(note that \eqref{ineq:alpha_cond_tmp} is implied by the assumption \eqref{eqn:alpha-cond}, 
where $C_\alpha$ is assumed sufficiently large, 
considering $\lambda\ge\frac{1}{100}c_\lambda\smin^2(\Xtruth)$ 
and $\|\opAA(\Mtruth)\|\le \|\Mtruth\| + \gamma\smin^2(\Xtruth)\le 2\|\Xtruth\|^2$ 
by~\eqref{ineq:opAA_bound}), 
we obtain for any $t\le\frac{1}{\theta\eta}\log(\kappa n)=\frac{2L'}{\eta}\log(\kappa n)$ 
that $\|X_t-\hat{X}_t\|\le ts_1^t\alpha^2/\pnorm{\Xtruth}{}$. This implies
\begin{align}
\frac{\|X_t-\hat{X}_t\|}{\alpha\smin\big(\Uspec^\T G\big)s_{r_\star}^t} 
&\leq (s_1/s_{r_\star})^t\smin^{-1}\big(\Uspec^\T G\big)\alpha/\|\Xtruth\| 
\nonumber\\
&\leq s_1^t \smin^{-1}(\Uspec^\T G)\alpha/\|\Xtruth\|
\nonumber\\
&\leq \exp(t\log(s_1)+L\log(L\kappa n))\alpha/\|\Xtruth\| \leq (L\kappa n)^{-L/2}
\label{eqn:verify-mahdi-lem-tmp2}
\end{align}
where the second inequality follows from \eqref{eqn:sr-bound-lower}, 
the penultimate inequality follows from our choice of $L$ 
which ensured $\smin^{-1}(\Uspec^\top G)\le (2n)^{L/2}$,  
and the last inequality follows from 
\eqref{eqn:s1-bound}, our choice $t\le\frac{2L'}{\eta}\log(\kappa n)$ 
and our assumption \eqref{eqn:alpha-cond} on $\alpha$ 
which implies $\alpha/\|\Xtruth\|\le (2\kappa n)^{-C_\alpha}$, 
given that $C_\alpha$ is sufficiently large, 
e.g. $C_\alpha\ge C(L, c_\lambda, c_\eta)$. 
It may also be inferred from the above arguments 
that $L$ can be made arbitrarily large by increasing $C_\alpha$.

Combining the above arguments, we conclude that 
for any $t\in [(L'/\eta)\log(\kappa n), (2L'/\eta)\log(\kappa n)]$, 
both of \eqref{eqn:verify-mahdi-lem-tmp1}, \eqref{eqn:verify-mahdi-lem-tmp2} hold,
hence the condition in (\ref{eqn:Mahdi-phi}) can be verified by
\begin{align}\label{ineq:p1_phi_bound}
\phi=\frac{\alpha\|G\|s_{r_\star+1}^t+\|X_t-\hat{X}_t\|}{\alpha \smin(\Uspec^\T G)s_{r_\star}^t} 
&\leq 2(L\kappa n)^{-L/2} \\
&\leq c_{\ref{lem:Mahdi}}\kappa^{-2},
\nonumber
\end{align}
by choosing $L$ sufficiently large.

This completes the verification of both assumptions of Lemma~\ref{lem:Mahdi}. 
Upon noting that the upper threshold of $t$ satisfies $(2L'/\eta)\log(\kappa n) \leq T_{\min}/16$, 
we will now invoke the conclusions of Lemma \ref{lem:Mahdi} to prove Lemma~\ref{lem:p1} 
for some $t\in[(L'/\eta)\log(\kappa n), T_{\min}/16]$.

\paragraph{Proof of bound~\eqref{eqn:p1.5-0}.}
This can be inferred from \eqref{ineq:mahdi_p1_1} in the following way. 
Recalling that $\smin(\Uspec^\T G)\ge(2n)^{-C_G}$ on the event $\Event$, 
and $s_{r_\star}\ge 1$ by \eqref{eqn:sr-bound-lower}, 
we obtain from \eqref{ineq:mahdi_p1_1} that
\begin{align*}
\smin(\EssS_{t_1}) \geq \frac14\alpha(2n)^{-C_G} \geq \alpha^2/\pnorm{\Xtruth}{},
\end{align*}
given the condition \eqref{eqn:alpha-cond} 
which guarantees 
\[\frac{\alpha}{\|\Xtruth\|}\le (2n)^{-C_\alpha/\eta}
\le\frac14(2n)^{-C_G},\] 
as long as $\eta\le c_\eta\le 1$ and $C_\alpha\ge C_G + 2$. The proof is complete.

\paragraph{Proof of bound~\eqref{eqn:p1.5-1}.} 
We combine (\ref{ineq:mahdi_p1_1}), (\ref{ineq:mahdi_p1_2}), 
and (\ref{ineq:p1_phi_bound}) to obtain
\begin{align*}
\frac{\pnorm{\Overpar_{t_1}}{}}{\smin(\EssS_{t_1})} 
\le 4C_{\ref{lem:Mahdi}}\phi \le 4C_{\ref{lem:Mahdi}}(L\kappa n)^{-L/2}
\le (L\kappa n/2)^{-L/2},
\end{align*}
where the last inequality follows from taking $L$ sufficiently large.
We further note that \eqref{eqn:lambda-cond} implies
\begin{align*}
  \smin(\EssS_{t_1})
  \le\|\SigmaTruth^2+\lambda I\|^{1/2}\smin\left((\SigmaTruth^2+\lambda I)^{-1/2}\EssS_{t_1}\right)
  &\le(c_\lambda+1)^{1/2}\|\Xtruth\|\smin\left((\SigmaTruth^2+\lambda I)^{-1/2}\EssS_{t_1}\right)
  \\
  &\le 2\|\Xtruth\|\smin\left((\SigmaTruth^2+\lambda I)^{-1/2}\EssS_{t_1}\right), 
\end{align*}
assuming $c_\lambda\le 1$, 
hence 
\[\frac{\pnorm{\Overpar_{t_1}}{}}{\smin\left((\SigmaTruth^2+\lambda I)^{-1/2}\EssS_{t_1}\right)}
\le 2\|\Xtruth\|(L\kappa n/2)^{-L/2}
\le (C_{\ref{lem:p1.5}.b}\kappa n)^{-C_{\ref{lem:p1.5}.b}}\|\Xtruth\|,
\]
as desired, with $C_{\ref{lem:p1.5}.b}=L/4$ as long as $L$ is sufficiently large. 
It is also clear that $C_{\ref{lem:p1.5}.b}$ can be made arbitrarily large by enlarging $C_\alpha$ 
as $L$ can be.

\paragraph{Proof of bound~\eqref{eqn:p1.5-2}.} 
We apply (\ref{ineq:mahdi_p1_2}) to yield
\begin{align*}
\pnorm{\Overpar_{t_1}}{} 
\leq C_{\ref{lem:Mahdi}}\phi\alpha\smin(\Uspec^\T G)s_{r_\star}^{t_1} 
\leq C_G C_{\ref{lem:Mahdi}}(L\kappa n)^{-L/2}
\left(1+\frac{2\eta}{c_\lambda}\right)^{t_1} \alpha 
\leq \alpha^{5/6}\pnorm{\Xtruth}{}^{1/6}, 
\end{align*}
where the second inequality follows from $\smin(\Uspec^\T G)\le\|G\|\le C_G$ by assumption 
and from \eqref{eqn:sr-bound-upper}; 
the last inequality follows from $t_1\le(2L'/\eta)\log(\kappa n)$
and from the condition (\ref{eqn:alpha-cond}) on $\alpha$,
provided that $C_\alpha$ is sufficiently large. 

\paragraph{Proof of bound~\eqref{eqn:p1.5-3}.}
We apply (\ref{ineq:mahdi_p1_3}) to yield that
\begin{align*}
\pnorm{\UperpTruth^\T \SVDU_{\EssX_{t+1}}}{} 
\leq C_{\ref{lem:Mahdi}}(\gamma + \phi) 
\leq \frac{c_\delta}{c_\lambda}\kappa^{-2 C_\delta /3},
\end{align*}
using the bounds of $\gamma$ and $\phi$ in \eqref{ineq:opAA_bound} and \eqref{ineq:p1_phi_bound},
provided that $c_\lambda\le\frac12\min(1,C_{\ref{lem:Mahdi}}^{-1})$ 
and $L\ge 2(C_\delta+1)$. 
To further bound $\|\Misalign_{t+1}\EssS_{t+1}^{-1}\SigmaTruth\|$ we need the following lemma.
\begin{lemma}
  \label{lem:misalign-angle-transform}
  Assume $\EssS_t$ is invertible, and
   at least one of the following is true: 
  (i) $\|\UperpTruth^\T \SVDU_{\Ess\myX_t}\|\le 1/4$;
  (ii) $\|\Misalign_t\EssS_t^{-1}\SigmaTruth\|\le \kappa^{-1}\|\Xtruth\|/4$.
  Then
  \[\kappa^{-1}\|\Xtruth\|\|\UperpTruth^\T \SVDU_{\Ess\myX_t}\|
    \le\|\Misalign_t\EssS_t^{-1}\SigmaTruth\|
    \le 2\|\Xtruth\|\|\UperpTruth^\T \SVDU_{\Ess\myX_t}\|.
  \]
\end{lemma}

\noindent The proof is postponed to Section~\ref{sec:misalign-transform}.
Returning to the proof of bound \eqref{eqn:p1.5-3}, 
the above lemma yields
\[\|\Misalign_{t+1}\EssS_{t+1}^{-1}\SigmaTruth\|
\le\frac{2c_\delta}{c_\lambda}\|\Xtruth\|\kappa^{-2C_\delta/3}
\le c_{\ref{lem:p1.5}}\|\Xtruth\|\kappa^{-2C_\delta/3},
\]
for some $c_{\ref{lem:p1.5}}\lesssim c_\delta/c_\lambda$, as desired.

\paragraph{Proof of bound~\eqref{eqn:p1.5-4}.}
 We have
\begin{align*}
\pnorm{\EssS_{t_1}}{} = \pnorm{\Utruth^\T X_{t_1} V_{t_1}}{} \leq \pnorm{X_{t_1}}{} \leq \pnorm{\Xtruth}{},
\end{align*}
where the last step follows from Lemma \ref{lem:approx_power_method}.

\subsubsection{Proof of Lemma \ref{lem:approx_power_method}}\label{sec:proof-dist-power-method}
We prove the claim~\eqref{ineq:approx_power_method} by induction and also show that 
$\|X_t\|\le \|\Xtruth\|$ follows from \eqref{ineq:approx_power_method}.  For the base case $t=0$, it holds by definition. 
Assume that \eqref{ineq:approx_power_method} holds for some $t\leq \frac{1}{\theta\eta}\log(\kappa n) - 1$. 
We aim to prove that (i) $\pnorm{X_t}{}\leq \pnorm{\Xtruth}{}$ 
and that (ii) the inequality \eqref{ineq:approx_power_method} continues to hold for $t+1$.

\paragraph{Proof of $\pnorm{X_t}{}\leq \pnorm{\Xtruth}{}$.} 

By the induction hypothesis we know
\[
  \bigpnorm{X_t-\hat{X}_t}{} \leq t\left(1+\frac{\eta}{\lambda}\pnorm{\opAA(\Mtruth)}{}\right)^t\frac{\alpha^2}{\|\Xtruth\|}.
\]
In view of the constraint~\eqref{ineq:alpha_cond_tmp} on $\alpha$ and the restriction $t\leq \frac{1}{\theta\eta}\log(\kappa n)$, we have 
\begin{align*}
t\frac{\alpha}{\pnorm{\Xtruth}{}} \leq \frac{1}{\theta\eta}\log(\kappa n) \cdot \frac{\eta}{K}\frac{1}{\log(\kappa n)} = \frac{1}{K\theta} \leq 1
\end{align*}
as long as $K=K(\theta,c_\lambda,C_G)$ is sufficiently large. This further implies
\[
  \|X_t-\hat{X}_t\|\leq\left(t\frac{\alpha}{\|\Xtruth\|}\right)
  \left(1+\frac{\eta}{\lambda}\|\opAA(\Mtruth)\|\right)^t\alpha
  \le \left(1+\frac{\eta}{\lambda}\|\opAA(\Mtruth)\|\right)^t\alpha
  . 
\]
On the other hand, since $\pnorm{X_0}{}\leq C_G\alpha$ under the event $\Event$ (cf.~(\ref{asp:init})), in view of \eqref{def:power_seq}, we have
\begin{align*}
\|\hat{X}_t\|\le\Big(1+\frac{\eta}{\lambda}\|\opAA(\Mtruth)\|\Big)^t\|X_0\| \leq C_G \Big(1+\frac{\eta}{\lambda}\|\opAA(\Mtruth)\|\Big)^t\alpha.
\end{align*}
Thus for a large enough $K = K(\theta, c_\lambda, C_G)$, we have 
\begin{align}\label{ineq:Xt_bound}
\|X_t\|\le\|X_t-\hat{X}_t\|+\|\hat{X}_t\|\leq \Big(1+\frac{\eta}{\lambda}\|\opAA(\Mtruth)\|\Big)^t (C_G+1)\alpha 
\leq \sqrt{c_\lambda/200}\cdot\kappa^{-4}\|\Xtruth\|,
\end{align}
where the last inequality follows from the condition on $t$ and the choice of $\alpha$ in \eqref{ineq:alpha_cond_tmp}:
\begin{align*}
\log\frac{\pnorm{\Xtruth}{}}{\alpha} \geq \log\frac{\sqrt{200}(C_G+1)\kappa^4}{\sqrt{c_\lambda}} + t\log\Big(1+\frac{\eta}{\lambda}\|\opAA(\Mtruth)\|\Big).
\end{align*}  
The inequality~(\ref{ineq:Xt_bound}) clearly implies $\pnorm{X_t}{} \leq \pnorm{\Xtruth}{}$.

\paragraph{Proof of~\eqref{ineq:approx_power_method} at the induction step.}
The proof builds on  a key recursive relation on $\bigpnorm{X_{t+1}-\hat{X}_{t+1}}{}$, from which the induction follows readily from our assumption.

\paragraph{Step 1: building a recursive relation on $\bigpnorm{X_{t+1}-\hat{X}_{t+1}}{}$.}   By definition~\eqref{def:power_seq}, we have $\hat{X}_{t+1}=\big(I+\frac{\eta}{\lambda}\opAA(\Mtruth)\big)\hat{X}_t$, which implies the following decomposition:
  \begin{align}
    X_{t+1} - \hat{X}_{t+1} 
    &= \underbrace{ \Big[X_{t+1} - \Big(I+\frac{\eta}{\lambda}\opAA(\Mtruth)\Big)X_t\Big]}_{\eqqcolon T_1} +  \underbrace{ \Big(I+\frac{\eta}{\lambda}\opAA(\Mtruth)\Big)(X_t - \hat{X}_t)}_{\eqqcolon T_2} .  \label{eqn:X-hatX}
  \end{align}
We shall control each term separately. 
\begin{itemize}
\item The second term $T_2$ can be trivially bounded as  
 \begin{equation}\label{eqn:X-hatX-2}
  \|T_2\| = \biggpnorm{\Big(I+\frac{\eta}{\lambda}\opAA(\Mtruth)\Big)(X_t-\hat{X}_t)}{}
  \leq \Big(1+\frac{\eta}{\lambda}\pnorm{\opAA(\Mtruth)}{}\Big)\bigpnorm{X_t-\hat{X}_t}{}.
\end{equation}
\item Turning to the first term $T_1$, 
by the update rule~\eqref{eqn:update} of $X_{t+1}$ and the triangle inequality, we further have 
  \begin{align}\label{eqn:T1_banana}
\|T_1\| = \left\|X_{t+1}-\Big(I+\frac{\eta}{\lambda}\opAA(\Mtruth)\Big)X_t\right\| 
&\leq \left\|\eta\opAA(X_tX_t^\T)X_t(X_t^\T X_t + \lambda I)^{-1}\right\| \nonumber \\
&\phantom{\leq{}} + \left\|\eta \opAA(\Mtruth)X_t\big((X_t^\T X_t + \lambda I)^{-1} - \lambda^{-1}I\big)\right\|.
\end{align}
Since $\|(X_t^\T X_t+\lambda I)^{-1}\|\le\lambda^{-1}$, 
it follows that the first term in \eqref{eqn:T1_banana} can be bounded by
\[
  \left\|\eta\opAA(X_tX_t^\T)X_t(X_t^\T X_t + \lambda I)^{-1}\right\|
  \le\frac{\eta}{\lambda}\|\opAA(X_t^\T X_t)\|\|X_t\|.
\]
In addition, since $\sqrt{c_\lambda/200}\cdot \kappa^{-4}\|\Xtruth\| = \sqrt{c_\lambda \smin^2(\Xtruth)/200} \leq\sqrt{\lambda/2}$ 
by the condition $\lambda \geq\frac{1}{100} \kappa^{-4} c_\lambda\smin^2(\Xtruth)$, we have by (\ref{ineq:Xt_bound}) that $\pnorm{X_t}{} \leq \sqrt{\lambda/2}$. Therefore, invoking Lemma~\ref{prop:inv} 
implies that
$$(X_t^\T X_t+\lambda I)^{-1} - \lambda^{-1} I = \lambda^{-2} X_t^\T X_t Q ,
 \quad \text{for some }Q \; \text{with }\pnorm{Q}{} \le 2.$$ 
As a result, the second term in \eqref{eqn:T1_banana} can be bounded by 
\[
  \left\|\eta \opAA(\Mtruth)X_t\big((X_t^\T X_t + \lambda I)^{-1} - \lambda^{-1}I\big)\right\|
  \leq 2 \frac{\eta}{\lambda^2}\|\opAA(\Mtruth)\|\|X_t\|^3. 
\]
Combining the above two inequalities leads to   
\[
  \left\|T_1\right\| 
   \leq \frac{\eta}{\lambda}\left(\|\opAA(X_t^\T X_t)\|+\frac{2}{\lambda}\|\opAA(\Mtruth)\|\|X_t\|^2\right)\|X_t\|.
\]
In view of Lemma \ref{lem:spectral_rip},
we know $\pnorm{\opAA(\Mtruth)}{}\lesssim r_\star\pnorm{\Mtruth}{}$ 
and $\pnorm{\opAA(X_t X_t^\T)}{}\lesssim r\pnorm{X_t}{}^2$. Plugging these relations into the previous bound leads to
\begin{equation}\label{eqn:X-hatX-1}
  \| T_1 \|
  \lesssim\frac{\eta r}{\lambda}\Big(1+\frac{\pnorm{\Mtruth}{}}{\lambda}\Big)\pnorm{X_t}{}^3
  \lesssim\frac{\eta\kappa^6 r}{\|\Mtruth\|}\kappa^6\|X_t\|^3, 
\end{equation}
where the last inequality follows from 
$\lambda\gtrsim \kappa^{-4} \smin^2(\Xtruth)=\kappa^{-6}\|\Mtruth\|$ (cf.~\eqref{eqn:lambda-cond}).
\end{itemize}
Putting the bounds on $T_1$ and $T_2$ together leads to 
\begin{equation}\label{eqn:approx-power-method-recursive-err}
  \bigpnorm{X_{t+1}-\hat{X}_{t+1}}{} \leq \Big(1+\frac{\eta}{\lambda}\|\opAA(\Mtruth)\|\Big)\bigpnorm{X_t-\hat{X}_t}{}
  +\frac{C\eta \kappa^12 r}{\|\Mtruth\|}\|X_t\|^3
\end{equation}
for some universal constant $C=C(c_\lambda)>0$.

\paragraph{Step 2: finishing the induction.}
By the bound of $\|X_t\|$ in (\ref{ineq:Xt_bound}), it suffices to prove
\begin{align*}
&t\Big(1+\frac{\eta}{\lambda}\pnorm{\opAA(\Mtruth)}{}\Big)^{t+1}\frac{\alpha^2}{\|\Xtruth\|} + \frac{C(C_G+1)^3\eta\kappa^{12} r}{\|\Xtruth\|^2} \Big(1+\frac{\eta}{\lambda}\pnorm{\opAA(\Mtruth)}{}\Big)^{3t}\alpha^3\\
&\leq (t+1)\Big(1+\frac{\eta}{\lambda}\pnorm{\opAA(\Mtruth)}{}\Big)^{t+1}\frac{\alpha^2}{\|\Xtruth\|}.
\end{align*}
This is equivalent to
\begin{align*}
C(C_G+1)^3\eta\kappa^{12} r\left(1+\frac{\eta}{\lambda}\|\opAA(\Mtruth)\|\right)^{2t-1}\le\frac{\|\Xtruth\|}{\alpha},
\end{align*}
which again follows readily from our assumption $t\le\frac{1}{\theta\eta}\log(\kappa n)$ 
and the assumption \eqref{ineq:alpha_cond_tmp} on $\alpha$ which implies
\begin{align*}
  \log\left(\frac{\|\Xtruth\|}{\alpha}\right)
  &\ge (2t-1)\log\left(1+\frac{\eta}{\lambda}\|\opAA(\Mtruth)\|\right) + 12\log\kappa + \log n + K
  \\
  &\ge (2t-1)\log\left(1+\frac{\eta}{\lambda}\|\opAA(\Mtruth)\|\right) + 12\log(n \kappa r) + \log(C(C_G+1)^3)
\end{align*}
provided $K=K(\theta, c_\lambda, C_G)$ is sufficiently large. The proof is complete.

\subsubsection{Proof of Lemma \ref{lem:misalign-angle-transform}}
\label{sec:misalign-transform}

We begin with the following observation: 
\begin{align}
  \Misalign_t\EssS_t^{-1}
  &=\UperpTruth^\T \SVDU_{\Ess\myX_t} \SVDSigma_{\Ess\myX_t} \SVDV_{\Ess\myX_t}^\T 
  \SVDV_{\Ess\myX_t} \SVDSigma_{\Ess\myX_t}^{-1} (\Utruth^\T \SVDU_{\Ess\myX_t})^{-1} 
  \nonumber\\
  &=\UperpTruth^\T \SVDU_{\Ess\myX_t} (\Utruth^\T \SVDU_{\Ess\myX_t})^{-1}
  \label{eqn:misalign-angle-transform} 
\end{align}
where we use: (i) $\Misalign_t = \UperpTruth^\T (\SVDU_{\Ess\myX_t}\Sigma_{\Ess\myX_t}\SVDV_{\Ess\myX_t}^\T)$ and $\EssS_t = \Utruth^\T \SVDU_{\Ess\myX_t}\Sigma_{\Ess\myX_t}\SVDV_{\Ess\myX_t}^\T$; (ii) $\Ess\myX_t$ is invertible since $\EssS_t$ is invertible, and hence $\SVDV_{\Ess\myX_t}$ has rank $r_\star$ and $\Sigma_{\Ess\myX_t}, \Utruth^\T \SVDU_{\Ess\myX_t}$ are also invertible. 

We will show that the above quantity is small if (and only if) 
$\UperpTruth^\T \SVDU_{\Ess\myX_t}$ is small.   

Turning to the proof, we first show that (ii) implies (i), 
  thus it suffices to prove the lemma under the condition (i). 
  In fact, in virtue of \eqref{eqn:misalign-angle-transform} we have
  \[\|\UperpTruth^\T \SVDU_{\Ess\myX_t}\|
  \le\|\Misalign_t\EssS_t^{-1}\|\|\Utruth^\T \SVDU_{\Ess\myX_t}\|
  \le\|\Misalign_t\EssS_t^{-1}\| \leq \smin(\Xtruth)^{-1} \|\Misalign_t\EssS_t^{-1} \SigmaTruth \|, \]
  where we used $\|\Utruth^\T \SVDU_{\Ess\myX_t}\|\le \|\Utruth\|\|\SVDU_{\Ess\myX_t}\|\le 1$.
  Consequently, $\|\UperpTruth^\T \SVDU_{\Ess\myX_t}\|\le 1/4$ 
  if $\|\Misalign_t\EssS_t^{-1}\SigmaTruth\|\le \kappa^{-1}\|\Xtruth\|/4$, as claimed.

  We proceed to show that the conclusion holds assuming condition (i).
  The first inequality has already been established above.
  For the second inequality, using \eqref{eqn:misalign-angle-transform} again, 
  it suffices to prove $\|(\Utruth^\T \SVDU_{\Ess\myX_t})^{-1}\|\le 2$, 
  which is in turn equivalent to
  $\smin(\Utruth^\T \SVDU_{\Ess\myX_t})\ge 1/2$.
  Now note that 
  $\SVDU_{\Ess\myX_t}=\Utruth \Utruth^\T \SVDU_{\Ess\myX_t}+\UperpTruth\UperpTruth^\T \SVDU_{\Ess\myX_t}$,
  thus
  \begin{align*}
    \smin(\Utruth^\T \SVDU_{\Ess\myX_t})
    &=\singval_{r_\star}(\Utruth^\T \SVDU_{\Ess\myX_t})\\
    &\ge\singval_{r_\star}(\Utruth \Utruth^\T \SVDU_{\Ess\myX_t})\\
    &\ge\singval_{r_\star}(\SVDU_{\Ess\myX_t})-\|\UperpTruth\UperpTruth^\T \SVDU_{\Ess\myX_t}\|\\
    &\ge 1-\|\UperpTruth^\T \SVDU_{\Ess\myX_t}\|\ge 3/4.
  \end{align*}
  In the last line, we used $\singval_{r_\star}(\SVDU_{\Ess\myX_t})=1$, 
  which follows from $\SVDU_{\Ess\myX_t}$ being a $n\times r_\star$ orthonormal matrix, 
  and the assumption (i). 
  This completes the proof.

\subsection{Establishing the induction step}\label{sec:phase-1-induction}
The claimed invertibility of $\EssS_t$ follows from induction and from Lemma~\ref{lem:p2}. 
In fact, by \eqref{eqn:p1.5-0} we know $\EssS_{t_1}$ is invertible, 
and by Lemma~\ref{lem:p2} we know that if $\EssS_t$ is invertible, 
$\EssS_{t+1}$ would also be invertible since $\EssS_t$ (resp. $\EssS_{t+1}$)
has the same invertibility as $(\SigmaTruth^2+\lambda I)^{-1}\EssS_t$
(resp. $(\SigmaTruth^2+\lambda I)^{-1}\EssS_{t+1}$).
For the rest of the proof we focus on establishing \eqref{subeq:condition-t-1} by induction. 

For the induction step we need to understand the one-step behaviors of 
$\|\Overpar_t\|$, $\|\Misalign_t\EssS_t^{-1}\SigmaTruth\|$, and $\|\EssS_t\|$, which are supplied by
 the following lemmas.

\begin{lemma}\label{lem:overpar-update}
  For any $t$ such that \eqref{subeq:condition-t-1} holds, 
  \begin{equation}
    \|\Overpar_{t+1}\|
    \le \left(1+\frac{1}{12\Cmax\kappa}\eta\right)\|\Overpar_t\|.
  \end{equation}
\end{lemma}


\begin{lemma}\label{lem:misalign-update}
  For any $t$ such that \eqref{subeq:condition-t-1} holds, 
  setting $Z_t=\SigmaTruth^{-1}(\EssS_t\EssS_t^\T+\lambda I)\SigmaTruth^{-1}$, 
  there exists some universal constant $C_{\ref{lem:misalign-update}}>0$ such that
  \begin{equation}\label{eqn:misalign-update}
    \uinorm{\Misalign_{t+1}\EssS_{t+1}^{-1}\SigmaTruth}
    \le \left(1-\frac{\eta}{3(\|Z_t\|+\eta)}\right)\uinorm{\Misalign_t\EssS_t^{-1}\SigmaTruth}
      + \eta \frac{C_{\ref{lem:misalign-update}}\kappa^6}{c_\lambda\|\Xtruth\|}\uinorm{\Utruth^\T\Delta_t}
      + \eta \left(\frac{\|\Overpar_t\|}{\smin(\EssS_t)}\right)^{1/2}\|\Xtruth\|.
  \end{equation}
  In particular, if $c_{\ref{lem:p1.5}}= 100C_{\ref{lem:misalign-update}}(C_{\ref{lem:p1.5}.a}+1)^4c_\delta/c_\lambda$, then 
  $\|\Misalign_{t}\EssS_{t}^{-1}\SigmaTruth\|\le c_{\ref{lem:p1.5}}\kappa^{-C_\delta/2}\|\Xtruth\|$ implies
  $\|\Misalign_{t+1}\EssS_{t+1}^{-1}\SigmaTruth\|\le c_{\ref{lem:p1.5}}\kappa^{-C_\delta/2}\|\Xtruth\|$.
\end{lemma}


\begin{lemma}\label{lem:bounded-S}
  For any $t$ such that \eqref{subeq:condition-t-1} holds, 
  \begin{equation}
    \|\EssS_{t+1}\|\le \left(1-\frac{\eta}{2}\right)\|\EssS_t\|
    + 100c_\lambda^{-1/2}\eta\kappa^3\|\Xtruth\|.
  \end{equation}
  In particular, if $C_{\ref{lem:p1.5}.a} = 200c_\lambda^{-1/2}$, 
  then $\|\EssS_{t}\|\le C_{\ref{lem:p1.5}.a}\kappa^3\|\Xtruth\|$ implies $\|\EssS_{t+1}\|\le C_{\ref{lem:p1.5}.a}\kappa^3\|\Xtruth\|$.
\end{lemma}


We now return to the induction step. 
Recall that we need to show 
\eqref{eqn:p1.5-1}--\eqref{eqn:p1.5-4} hold for $t+1$. 
It is obvious that \eqref{eqn:p1.5-2}--\eqref{eqn:p1.5-4} 
hold for $t+1$ by the induction hypothesis and the above lemmas. 
It remains to prove \eqref{eqn:p1.5-1}. 
To this end we distinguish two cases: 
$\smin((\SigmaTruth^2+\lambda I)^{-1/2}\EssS_t)\le 1/3$
and $\smin((\SigmaTruth^2+\lambda I)^{-1/2}\EssS_t)> 1/3$. 
In the former case, \eqref{eqn:p1.5-1} for $t+1$ follows from 
Lemma~\ref{lem:overpar-update} and Lemma~\ref{lem:p2} 
(to be proved in Appendix \ref{pf:lem:p2}), 
which imply (provided $\Cmax\ge 2$)
\[\frac{\|\Overpar_{t+1}\|}{\smin((\SigmaTruth^2+\lambda I)^{-1/2}\EssS_{t+1})}
\le \frac{\left(1+\frac{\eta}{4\Cmax\kappa}\right)}{(1+\eta/8)}
\frac{\|\Overpar_{t}\|}{\smin((\SigmaTruth^2+\lambda I)^{-1/2}\EssS_{t})}
\le \frac{\|\Overpar_{t}\|}{\smin((\SigmaTruth^2+\lambda I)^{-1/2}\EssS_{t})},
\]
as desired. In the latter case where 
$\smin((\SigmaTruth^2+\lambda I)^{-1/2}\EssS_t)> 1/3$, 
one may apply the first part of Lemma~\ref{lem:p2}
to deduce that $\smin((\SigmaTruth^2+\lambda I)^{-1/2}\EssS_{t+1})\ge 1/10$
(given that $\eta\le c_\eta$ for some sufficiently small constant $c_\eta$). 
This combined with \eqref{eqn:p1.5-2} for $t+1$ (already proved) 
yields desired inequality \eqref{eqn:p1.5-1} for $t+1$, 
given our assumption \eqref{eqn:alpha-cond} on the smallness of $\alpha$.  
This completes the proof.


\subsubsection{Proof of Lemma \ref{lem:overpar-update}}
\label{pf:lem:overpar-update}
If $r=r_\star$, then we have $\|\Overpar_t\|=0$ for all $t\ge 0$. 
The conclusion follows trivially.
Therefore, we only consider the case when $r>r_\star$.
By definition, we have 
\begin{align*}
    \Overpar_{t+1} = \Noise_{t+1}\Vperp{t+1}     & = \Noise_{t+1}\SVDV_t\SVDV_t^\T  \Vperp{t+1}+\Noise_{t+1}\Vperp{t}\Vperp{t}^\T \Vperp{t+1} \\
    &= -\Noise_{t+1}\SVDV_t(\Signal_{t+1}\SVDV_t)^{-1}\Signal_{t+1}\Vperp{t}\Vperp{t}^\T \Vperp{t+1}+\Noise_{t+1}\Vperp{t}\Vperp{t}^\T \Vperp{t+1},
\end{align*} 
where the last inequality uses the fact that 
$\SVDV_t^\T \Vperp{t+1}=-(\Signal_{t+1}\SVDV_t)^{-1}\Signal_{t+1}\Vperp{t}\Vperp{t}^\T \Vperp{t+1}$.
To see this, note that 
$$\Signal_{t+1}\Vperp{t+1}=0 \qquad \Longrightarrow  \qquad \Signal_{t+1}\SVDV_t\SVDV_t^\T \Vperp{t+1}=-\Signal_{t+1}\Vperp{t}\Vperp{t}^\T \Vperp{t+1}.$$
Left-multiplying both sides by $(\Signal_{t+1}\SVDV_t)^{-1}$ yields the desired identity. Note that the invertibility of $\Signal_{t+1}\SVDV_t$ follows from the invertibility of $\EssS_t$ by inserting $Q=0$ in Lemma~\ref{lem:S-surrogate}.

By Lemma~\ref{lem:update-approx}, we immediately 
obtain that $\Signal_{t+1}\Vperp{t}=\eta\Err{b}_t\Vperp{t}$, 
and $\Noise_{t+1}\Vperp{t}=\Overpar_t + \eta \Err{d}_t\Vperp{t}$,
where $\| \Err{b}_t \| \vee \| \Err{d}_t \| \leq \frac{1}{24\Cmax\kappa}\|\Overpar_t\|$. 
Assume for now that 
\begin{equation}
  \label{eqn:NSinv-tmp}
  \|\Noise_{t+1}\SVDV_t(\Signal_{t+1}\SVDV_t)^{-1}\|\le 1.
\end{equation}
In addition, notice that $\|\Vperp{t}^\T \Vperp{t+1}\|\leq 1$ since both factors are orthonormal matrices, we have
\begin{align*}
  \|\Overpar_{t+1}\|
  &\le \|\Overpar_t\| + \eta\|\Noise_{t+1}\SVDV_t(\Signal_{t+1}\SVDV_t)^{-1}\|\|\Err{b}_t\| + \eta\|\Err{d}_t\| 
  \\
  &\le \left(1+\frac{1}{12\Cmax\kappa}\eta\right)\|\Overpar_t\|,
\end{align*}
as desired. It remains to prove~\eqref{eqn:NSinv-tmp}.

\paragraph{Proof of bound \eqref{eqn:NSinv-tmp}.}
This can be done by plugging $Q=0$ into Lemma~\ref{lem:NS-surrogate} and bounding the resulting expression.
This (in fact, a much stronger inequality) will be done in detail in the proof of Lemma~\ref{lem:misalign-update}, to be presented soon in Section~\ref{pf:lem:misalign-update}. 
In fact, the resulting expression is the same as~\eqref{eqn:misalign-update-tmp2} there (albeit with different values of $\Err{\ref{lem:S-surrogate}.a}_t$, $\Err{\ref{lem:NS-surrogate}.a}_t$, $\Err{\ref{lem:NS-surrogate}.b}_t$, 
which do not affect the proof).
Following the same strategy to control~\eqref{eqn:misalign-update-tmp2} there, we may show that $\|\Noise_{t+1}\SVDV_t(\Signal_{t+1}\SVDV_t)^{-1}\SigmaTruth\|$ enjoys the same bound~\eqref{eqn:misalign-update-simplified} as $\|\Misalign_{t+1}\EssS_{t+1}^{-1}\SigmaTruth\|$, 
the right hand side of which is less than $\kappa^{-1}\|\Xtruth\|=\|\SigmaTruth^{-1}\|^{-1}$ given~\eqref{eqn:p1.5-3} and~\eqref{eqn:p1.5-4}. 
Thus $\|\Noise_{t+1}\SVDV_t(\Signal_{t+1}\SVDV_t)^{-1}\| \le \|\Noise_{t+1}\SVDV_t(\Signal_{t+1}\SVDV_t)^{-1}\SigmaTruth\| \|\SigmaTruth^{-1}\| \le 1$ as claimed.


\subsubsection{Proof of Lemma \ref{lem:misalign-update}}
\label{pf:lem:misalign-update}

Denoting $\EssX_t\defeq\myX_t\SVDV_t$, we have $\Misalign_t=\UperpTruth^\T\Ess\myX_t$ and $\EssS_t=\Utruth^\T\Ess\myX_t$. 
Suppose for the moment that 
\begin{align}\label{eq:invertibility}
\|(V_t^\T V_{t+1})^{-1}\|\le 2,
\end{align}
whose proof is deferred to the end of this section.
We can write the update equation of $\Ess\myX_t$ as
\begin{align}
  \Ess\myX_{t+1}
  &=\myX_{t+1}\SVDV_{t+1}=\myX_{t+1}\SVDV_t\SVDV_t^\T \SVDV_{t+1}+\myX_{t+1}\Vperp{t} \Vperp{t}^\T \SVDV_{t+1} 
  \nonumber \\
  &=\left(\myX_{t+1}\SVDV_t+\myX_{t+1}\Vperp{t}\Vperp{t}^\T \SVDV_{t+1}(\SVDV_t^\T \SVDV_{t+1})^{-1}\right)\SVDV_t^\T \SVDV_{t+1}.
  \label{eqn:EssX-update}
\end{align}
Left-multiplying both sides of 
\eqref{eqn:EssX-update} with $\UperpTruth$ (or $\Utruth$),
we obtain
\begin{subequations}
\begin{align}
  \Misalign_{t+1} &= (\Noise_{t+1}\SVDV_t+\Noise_{t+1}\Vperp{t} Q)\SVDV_t^\T \SVDV_{t+1}, \\
  \EssS_{t+1} &= (\Signal_{t+1}\SVDV_t+\Signal_{t+1}\Vperp{t} Q)\SVDV_t^\T \SVDV_{t+1},
  \label{eqn:tilde-S-update-tmp}
\end{align}
\end{subequations}
where we define $Q \defeq \Vperp{t}^\T V_{t+1}(\SVDV_t^\T \SVDV_{t+1})^{-1}$. 
Consequently, we arrive at 
\begin{equation}
  \Misalign_{t+1}\EssS_{t+1}^{-1}
  = (\Noise_{t+1}\SVDV_t+\Noise_{t+1}\Vperp{t} Q)(\Signal_{t+1}\SVDV_t+\Signal_{t+1}\Vperp{t} Q)^{-1}.
  \label{eqn:misalign-update-tmp}
\end{equation}
Since $\|Q\| \leq 2$ (which is an immediate  implication of~\eqref{eq:invertibility}), 
we can invoke Lemma~\ref{lem:NS-surrogate} to obtain
\begin{align}
  \Misalign_{t+1}\EssS_{t+1}^{-1}\SigmaTruth
  = & \Misalign_t\EssS_t^{-1}(I+\eta\Err{\ref{lem:NS-surrogate}.a}_t)
  A_t(A_t+\eta\SigmaTruth^2)^{-1}(I+\eta\Err{\ref{lem:S-surrogate}}_t)^{-1}\SigmaTruth + \eta\Err{\ref{lem:NS-surrogate}.b}_t\SigmaTruth
  \nonumber\\
  = & \Misalign_t\EssS_t^{-1}\SigmaTruth(I+\eta\SigmaTruth^{-1}\Err{\ref{lem:NS-surrogate}.a}_t\SigmaTruth) 
  H_t(H_t+\eta I)^{-1}(I+\eta\SigmaTruth^{-1}\Err{\ref{lem:S-surrogate}}_t\SigmaTruth)^{-1}
  + \eta\Err{\ref{lem:NS-surrogate}.b}_t\SigmaTruth,
  \label{eqn:misalign-update-tmp2}
\end{align}
where for simplicity of notation, we denote
\begin{align*}
  A_t \defeq (1-\eta)\EssS_t\EssS_t^\T+\lambda I, \quad 
  \text{and} \quad 
  H_t \defeq \SigmaTruth^{-1}A_t\SigmaTruth^{-1}.
\end{align*}
In addition, we have
\begin{align*}
  \|\Err{\ref{lem:S-surrogate}}_t\|+\|\Err{\ref{lem:NS-surrogate}.a}_t\|
  &\le \frac{1}{64\kappa^5},
  \\
  \uinorm{\Err{\ref{lem:NS-surrogate}.b}_t}
  &\le 800c_\lambda^{-1}\kappa^2\|\Xtruth\|^{-2}\uinorm{\Utruth^\T\Delta_t}
  +\frac{1}{64(C_{\ref{lem:p1.5}.a}+1)^2\kappa^5\|\Xtruth\|}\uinorm{\Misalign_t\EssS_t^{-1}\SigmaTruth}
  +\frac{1}{64}\left(\frac{\|\Overpar_t\|}{\smin(\EssS_t)}\right)^{2/3}.
\end{align*}
Moreover, it is clear that $\eta\le c_\eta\le 1\le \kappa^4$ since $\kappa\ge 1$, 
and that $\|H_t\|\le \kappa^{2}(1+\|\EssS_t\|^2/\|\Xtruth\|^2)\le (C_{\ref{lem:p1.5}.a}+1)^2\kappa^4$. 
Hence we have
\[\|H_t\|+\eta\le 2(C_{\ref{lem:p1.5}.a}+1)^2\kappa^4\]
which implies 
\begin{equation}\label{eqn:misalign-err1}
  \|\Err{\ref{lem:S-surrogate}}_t\|+\|\Err{\ref{lem:NS-surrogate}.a}_t\|\le \frac{1}{24\kappa}\frac{1}{\|H_t\|+\eta}.
\end{equation}
Similarly we may also show
\begin{equation}\label{eqn:misalign-err2}
  \uinorm{\Err{\ref{lem:NS-surrogate}.b}_t}
  \le 800c_\lambda^{-1}\kappa^2\|\Xtruth\|^{-2}\uinorm{\Utruth^\T\Delta_t}
  +\frac{1}{12(\|H_t\|+\eta)\|\Xtruth\|}\uinorm{\Misalign_t\EssS_t^{-1}\SigmaTruth}
  +\frac12\left(\frac{\|\Overpar_t\|}{\smin(\EssS_t)}\right)^{2/3}.
\end{equation}

Since $H_t$ is obviously positive definite, we have
\begin{equation}
  \|H_t(H_t+\eta I)^{-1}\|\le 1-\frac{\eta}{\|H_t\|+\eta}.
\end{equation}
Thus
\begin{align}
  \uinorm{\Misalign_{t+1}\EssS_{t+1}^{-1}\SigmaTruth}
  \le & \left(1-\frac{\eta}{\|H_t\|+\eta}\right)
  (1-\eta\kappa\|\Err{\ref{lem:S-surrogate}}_t\|)^{-1}(1+\eta\kappa\|\Err{\ref{lem:NS-surrogate}.a}_t\|)
  \uinorm{\Misalign_{t}\EssS_{t}^{-1}\SigmaTruth} + \eta\uinorm{\Err{\ref{lem:NS-surrogate}.b}_t}\|\Xtruth\|.
  \nonumber\\
  \le & \left(1-\frac{\eta}{\|H_t\|+\eta}\right)
  \left(1+\frac{1}{12}\frac{\eta}{\|H_t\|+\eta}\right)^2
  \uinorm{\Misalign_t\EssS_t^{-1}\SigmaTruth}
  \nonumber\\
  & + \eta\frac{800\kappa^2}{c_\lambda\|\Xtruth\|}\uinorm{\Utruth^\T\Delta_t}
  + \frac{1}{12}\frac{\eta}{\|H_t\|+\eta}\uinorm{\Misalign_{t}\EssS_{t}^{-1}\SigmaTruth}
  + \frac12\eta\left(\frac{\|\Overpar_t\|}{\smin(\EssS_t)}\right)^{2/3}\|\Xtruth\|
  \nonumber\\
  \le & \left(1-\frac56\frac{\eta}{\|H_t\|+\eta}\right)
  \uinorm{\Misalign_t\EssS_t^{-1}\SigmaTruth}
  + \frac{1}{12}\frac{\eta}{\|H_t\|+\eta}\uinorm{\Misalign_{t}\EssS_{t}^{-1}\SigmaTruth}
  \nonumber\\
  & + \eta\frac{800\kappa^2}{c_\lambda\|\Xtruth\|}\uinorm{\Utruth^\T\Delta_t}
  + \frac12\eta\left(\frac{\|\Overpar_t\|}{\smin(\EssS_t)}\right)^{2/3}\|\Xtruth\|
  \nonumber\\
  \le & \left(1-\frac34\frac{\eta}{\|H_t\|+\eta}\right)
  \uinorm{\Misalign_t\EssS_t^{-1}\SigmaTruth}
  + \eta\frac{800\kappa^2}{c_\lambda\|\Xtruth\|}\uinorm{\Utruth^\T\Delta_t}
  + \frac12\eta\left(\frac{\|\Overpar_t\|}{\smin(\EssS_t)}\right)^{2/3}\|\Xtruth\|\
  \nonumber\\
  \le & \left(1-\frac34\frac{\eta}{\|Z_t\|+\eta}\right)
  \uinorm{\Misalign_t\EssS_t^{-1}\SigmaTruth}
  + \eta\frac{800\kappa^2}{c_\lambda\|\Xtruth\|}\uinorm{\Utruth^\T\Delta_t}
  + \frac12\eta\left(\frac{\|\Overpar_t\|}{\smin(\EssS_t)}\right)^{2/3}\|\Xtruth\|,
  \label{eqn:pf:misalign-update}
\end{align}
where in the second inequality we used $(1-x)^{-1}\le 1+x$ for $x<1$, 
in the penultimate inequality we used the elementary fact 
$(1-x)(1+\frac{1}{16}x)^2\le1-\frac56 x$ for $x\in[0,1]$, 
and in the last inequality we used the obvious fact 
\[\|H_t\|=\|\SigmaTruth^{-1} ((1-\eta)\EssS_t\EssS_t^\T + \lambda I) \SigmaTruth^{-1}\|
\le \|\SigmaTruth^{-1} (\EssS_t\EssS_t^\T + \lambda I) \SigmaTruth^{-1}\|
=\|Z_t\|.\]
The desired inequality~\eqref{eqn:misalign-update} follows from the above inequality by setting $C_{\ref{lem:misalign-update}}=800$.

For the remaining claim, we need to apply the conclusion 
  of the first part with $\uinorm{\cdot}=\|\cdot\|$. 
  Then we note the following bounds:
  \begin{enumerate}[label=(\roman*), leftmargin=*, align=left]
    \item $\|Z_t\|\le\|\SigmaTruth^{-1}\|^2(\|\EssS_t\|^2+\lambda)\le(C_{\ref{lem:p1.5}.a}+1)^2\kappa^4$
    by \eqref{eqn:p1.5-4} and \eqref{eqn:lambda-cond} (since we may choose $c_\lambda\le 1$);
    \item $\eta\le c_\eta\le (C_{\ref{lem:p1.5}.a}+1)^2\kappa^4$;
    \item $\|\Utruth^\T\Delta_t\|\le\|\Delta_t\|\le 16(C_{\ref{lem:p1.5}.a}+1)^2c_\delta\kappa^{-2C_\delta/3}\|\Xtruth\|^2$
    by Lemma~\ref{lem:Delta-bound};
    \item $(\|\Overpar_t\|/\smin(\EssS_t))^{1/2}\le c_\delta\kappa^{-2C_\delta/3}$ 
    by \eqref{eqn:p1.5-1}, if we choose $C_\alpha\ge 3c_\delta^{-1}+3C_\delta+3$.
  \end{enumerate}
  These together imply
  \begin{equation}
  \label{eqn:misalign-update-simplified}
    \|\Misalign_{t+1}\EssS_{t+1}^{-1}\SigmaTruth\|
    \le \left(1-\frac{\eta}{6(C_{\ref{lem:p1.5}.a}+1)^2\kappa^4}\right)\|\Misalign_t\EssS_t^{-1}\SigmaTruth\|
      + \eta \frac{16C_{\ref{lem:misalign-update}}\kappa^2}{c_\lambda}(C_{\ref{lem:p1.5}.a}+1)^2 c_\delta\kappa^{-2C_\delta/3}\|\Xtruth\|
      + \eta c_\delta\kappa^{-2C_\delta/3}\|\Xtruth\|.
  \end{equation}
  The conclusion follows easily by plugging in 
  $\|\Misalign_{t}\EssS_{t}^{-1}\SigmaTruth\|\le c_{\ref{lem:p1.5}}\kappa^{-C_\delta/2}\|\Xtruth\|$
  and using $\kappa^6\kappa^{-2C_\delta/3}\le\kappa^{-C_\delta/2}$
  when $C_\delta$ is sufficiently large.

\paragraph{Proof of bound~\eqref{eq:invertibility}.}
First, we observe that it is equivalent to show that  $\smin(\SVDV_t^\T\SVDV_{t+1})\ge 1/2$. 
But from $\SVDV_{t+1}\SVDV_{t+1}^\T+\Vperp{t+1}\Vperp{t+1}^\T=I$ we have
\begin{align*}
  \smin(\SVDV_t^\T\SVDV_{t+1})
  &=\singval_{r_\star}(\SVDV_t^\T\SVDV_{t+1})
  \ge \singval_{r_\star}(\SVDV_t^\T\SVDV_{t+1}\SVDV_{t+1}^\T)
  = \singval_{r_\star}(\SVDV_t^\T-\SVDV_t^\T\Vperp{t+1}\Vperp{t+1}^\T)
  \\
  &\ge \singval_{r_\star}(V_t^\T) - \|\SVDV_t^\T\Vperp{t+1}\Vperp{t+1}^\T\|
  \\
  &\ge 1 - \|\SVDV_t^\T\Vperp{t+1}\|,
\end{align*}
where the last inequality follows from $\singval_{r_\star}(V_t^\T)=1$ 
(since $\SVDV_t\in\reals^{r\times r_\star}$ is orthonormal)
and from that $\|\SVDV_t^\T\Vperp{t+1}\Vperp{t+1}^\T\|\le\|\SVDV_t^\T\Vperp{t+1}\|$.
This implies that, to show $\smin(\SVDV_t^\T\SVDV_{t+1})\ge 1/2$, 
it suffices to prove $\|\SVDV_t^\T\Vperp{t+1}\|\le 1/2$. 

Next we prove that $\|\SVDV_t^\T\Vperp{t+1}\|\le 1/2$.
Recall that by definition we have $\Signal_{t+1}\Vperp{t+1}=0$.
Right-multiplying both sides of \eqref{eqn:S-update-approx} 
by $\Vperp{t+1}$, we obtain
\begin{equation*}
  0=\left((1-\eta)I+\eta(\SigmaTruth^2+\lambda I+\Err{a}_t)
      (\EssS_t\EssS_t^\T +\lambda I)^{-1}
    \right)\EssS_t (\SVDV_t^\T\Vperp{t+1}) 
    + \eta\Err{b}_t\Vperp{t+1},
\end{equation*}
hence
\begin{equation*}
  \|\SVDV_t^\T\Vperp{t+1}\|
  \le \eta\|\Err{b}_t\Vperp{t+1}\|\|\EssS_t^{-1}\|
  \left\|\left((1-\eta)I+\eta(\SigmaTruth^2+\lambda I+\Err{a}_t)
    (\EssS_t\EssS_t^\T +\lambda I)^{-1}\right)^{-1}
  \right\|.
\end{equation*}
By \eqref{eqn:Err-2-bound} we have
\begin{equation*}
  \|\Err{b}_t\Vperp{t+1}\|\|\EssS_t^{-1}\|
  \le \frac{\|\Err{b}_t\|}{\smin(\EssS_t)}
  \le \frac{1}{10\kappa},
\end{equation*}
thus it suffices to show 
\begin{equation}
  \eta\left\|\left((1-\eta)I+\eta(\SigmaTruth^2+\lambda I+\Err{a}_t)
    (\EssS_t\EssS_t^\T +\lambda I)^{-1}\right)^{-1}
  \right\|\le 5\kappa,
\end{equation}
or equivalently, 
\begin{equation}
  \smin\left((1-\eta)I + \eta(\SigmaTruth^2+\lambda I+\Err{a}_t)
    (\EssS_t\EssS_t^\T +\lambda I)^{-1}
  \right)\ge \frac{\eta}{5\kappa}.
  \label{eqn:Vinv-reduction}
\end{equation}
To this end, we write 
\begin{align}
  & (1-\eta)I+\eta(\SigmaTruth^2+\lambda I+\Err{a}_t)
    (\EssS_t\EssS_t^\T +\lambda I)^{-1}
  \nonumber\\
  &=\left( I +\eta\Err{a}_t\left(
    (1-\eta)(\EssS_t\EssS_t^\T+\lambda I)
    +\eta(\SigmaTruth^2+\lambda I)\right)^{-1}
  \right)
  \left( (1-\eta)I+\eta(\SigmaTruth^2+\lambda I)(\EssS_t\EssS_t^\T +\lambda I)^{-1}\right)
  \label{eqn:Vinv-factors}
\end{align}
and control the two terms separately.
\begin{itemize}
\item To control the first factor,
starting from \eqref{eqn:Err-1-bound} we may deduce
\begin{align*}
  \|\Err{a}_t\|&\le\kappa^{-4}\|\Xtruth\|\|\Misalign_t\EssS_t^{-1}\SigmaTruth\|
  +\|\Utruth^\top\Delta_t\|
  \\
  &\le\kappa^{-4}\|\Xtruth\|c_{\ref{lem:p1.5}}\kappa^{-C_\delta/2}\|\Xtruth\|
  +c_{\ref{lem:Delta-bound}}\kappa^{-2C_\delta/3}\|\Xtruth\|^2
  \\
  &\le \kappa^{-2}\|\Xtruth\|^2/2=\smin^2(\Xtruth)/2,
\end{align*}
where the second inequality follows from \eqref{eqn:p1.5-3}
and Lemma~\ref{lem:Delta-bound}; 
the last inequality follows from choosing $c_\delta$ sufficiently small 
(recall that $c_{\ref{lem:p1.5}}, c_{\ref{lem:Delta-bound}}\lesssim c_\delta/c_\lambda$) 
and $C_\delta$ sufficiently large.
Furthermore, since $\EssS_t\EssS_t^\T$ is positive semidefinite, we have
\[
  \left\|\left(
    (1-\eta)(\EssS_t\EssS_t^\T+\lambda I)
    +\eta(\SigmaTruth^2+\lambda I)\right)^{-1}
  \right\|\le \eta^{-1}\smin^{-2}(\SigmaTruth)
  =\eta^{-1}\smin^{-2}(\Xtruth), 
\] 
hence
\begin{align}\label{eqn:Vinv-tmp1}
  & \smin\left(1+\eta\Err{a}_t\left(
    (1-\eta)(\EssS_t\EssS_t^\T+\lambda I)
    +\eta(\SigmaTruth^2+\lambda I)\right)^{-1}
  \right)
  \nonumber\\
  &\ge 1 - \eta\|\Err{a}_t\|\left\|\left(
    (1-\eta)(\EssS_t\EssS_t^\T+\lambda I)
    +\eta(\SigmaTruth^2+\lambda I)\right)^{-1}
  \right\|
  \nonumber\\
  &\ge 1 - \eta\cdot\frac{\smin^2(\Xtruth)}{2}\cdot\eta^{-1}\smin^{-2}(\Xtruth)
  =1/2. 
\end{align}
\item Now we control the second factor. 
By Lemma~\ref{prop:positive-AB} we have
\begin{align*}
\smin\left(1-\eta+\eta(\SigmaTruth^2+\lambda I)
    (\EssS_t\EssS_t^\T +\lambda I)^{-1}\right)
  &=(1-\eta)\smin\left(I+\frac{\eta}{1-\eta}(\SigmaTruth^2+\lambda I)
  (\EssS_t\EssS_t^\T +\lambda I)^{-1}\right)
  \\
  &\ge(1-\eta)\left(\frac{\|\SigmaTruth^2+\lambda I\|}{\smin(\SigmaTruth^2+\lambda I)}\right)^{-1/2}
  \\
  &=(1-\eta)\left(\frac{\|\Xtruth\|^2+\lambda}{\smin^2(\Xtruth)+\lambda}\right)^{-1/2}.
\end{align*}
It is easy to check that the function $\lambda\mapsto(a+\lambda)/(b+\lambda)$ 
is decreasing on $[0,\infty)$ for $a\ge b>0$, 
thus 
\[
  \frac{\|\Xtruth\|^2+\lambda}{\smin^2(\Xtruth)+\lambda}
  \le \frac{\|\Xtruth\|^2}{\smin^2(\Xtruth)}=\kappa^2,
\]
which implies
\begin{equation}\label{eqn:AB-bound}
  \smin\left((1-\eta)I + \eta(\SigmaTruth^2+\lambda I)
    (\EssS_t\EssS_t^\T +\lambda I)^{-1}\right)
  \ge\frac{1-\eta}{\kappa}.
\end{equation} 
\end{itemize}
Plugging \eqref{eqn:AB-bound} and \eqref{eqn:Vinv-tmp1}  into \eqref{eqn:Vinv-factors} yields 
\begin{equation}
  \smin\left((1-\eta)I + \eta(\SigmaTruth^2+\lambda I+\Err{a}_t)
    (\EssS_t\EssS_t^\T +\lambda I)^{-1}
  \right)\ge \frac{1-\eta}{2\kappa}\ge\frac{\eta}{5\kappa},
\end{equation}
where the last inequality follows from the assumption $\eta\le c_\eta$.
This shows \eqref{eqn:Vinv-reduction} as desired, 
thereby completing the proof.

\subsubsection{Proof of Lemma \ref{lem:bounded-S}}
\label{pf:lem:bounded-S}
Combine \eqref{eqn:tilde-S-update-tmp} and Lemma~\ref{lem:S-surrogate} to see that 
  \begin{align}
    \|\EssS_{t+1}\|
    &\le \|\Signal_{t+1}\SVDV_t+\Signal_{t+1}\Vperp{t} Q\|
    \nonumber\\
    &\le\|1+\eta\Err{\ref{lem:S-surrogate}}_t\| \cdot \left\|(1-\eta)(\EssS_t\EssS_t^\T+\lambda I)^{1/2}
      +\eta(\SigmaTruth^2 + \lambda I)
      (\EssS_t\EssS_t^\T+\lambda I)^{-1/2}
    \right\| \cdot 
    \left\|(\EssS_t\EssS_t^\T+\lambda I)^{-1/2}\EssS_t\right\|
    \nonumber\\
    &\le(1+\eta\|\Err{\ref{lem:S-surrogate}}_t\|)\left((1-\eta)(\|\EssS_t\|^2+\lambda )^{1/2}
      +4\eta\lambda^{-1/2}\|\Xtruth\|^2\right)
    (\|\EssS_t\|^2+\lambda)^{-1/2}\|\EssS_t\| 
    \nonumber\\
    &\le \left(1+\frac{\eta}4\right)\left((1-\eta)\|\EssS_t\| 
      + 4\eta\frac{\|\Xtruth\|^2\|\EssS_t\|}{\sqrt{\lambda(\|\EssS_t\|^2+\lambda)}}
    \right)
    \nonumber\\
    &\le \left(1-\frac{\eta}{2}\right)\|\EssS_t\|
    + 5\eta\frac{\|\Xtruth\|^2}{\sqrt{\lambda}},
  \end{align}
  where the third line follows from $\|\SigmaTruth^2 + \lambda I\|\le(1+\lambda)\|\Xtruth\|^2\le 2\|\Xtruth\|^2$ assuming $c_\lambda\le 1$ 
  and from the fact that the singular values of $(\EssS_t\EssS_t^\T+\lambda I)^{-1/2}\EssS_t$ are 
  $(\sigma_j^2(\EssS_t)+\lambda)^{-1/2}\sigma_j(\EssS_t)$, $j=1,\ldots,r_\star$,\footnote{This can be seen from plugging in $\EssS_t=\SVDU_t\SVDSigma_t$ by definition which implies $(\EssS_t\EssS_t^\T+\lambda I)^{-1/2}\EssS_t=\SVDU_t(\SVDSigma_t+\lambda I)^{-1/2}\SVDSigma_t$.}
  which is bounded by $(\|\EssS_t\|^2+\lambda)^{-1/2}\|\EssS_t\|$ since $\sigma\mapsto(\sigma^2+\lambda)^{-1/2}\sigma$ is increasing and since $\|\EssS_t\|$ is the largest singular value of $\EssS_t$. In the fourth line,
   we used the error bound $\|\Err{\ref{lem:S-surrogate}}_t\|\le1/4$ 
  and the last line follows from the elementary inequalities $1+\eta/4\le(1-\eta/2)(1-\eta)^{-1}\le 5/4$ 
  given that $\eta\le c_\eta$ for sufficiently small constant $c_\eta>0$.
  The conclusion readily follows from the above inequality and the assumption 
  $\lambda\ge\frac1{100} \kappa^{-4} c_\lambda\smin^2(\Xtruth)$.

\section{Proofs for Phase II}
\label{sec:proof_phaseII}
This section collects the proofs for Phase II. 
\subsection{Proof of Lemma~\ref{lem:p2}}
\label{pf:lem:p2}
Since $\| \SVDV_{t+1}^\T \SVDV_t \| \leq 1$, we have 
\begin{align*}
  \smin((\SigmaTruth^2+\lambda I)^{-1/2}\EssS_{t+1})
  &\ge\smin((\SigmaTruth^2+\lambda I)^{-1/2}\EssS_{t+1}\SVDV_{t+1}^\T \SVDV_t)
  \\
  &=\smin((\SigmaTruth^2+\lambda I)^{-1/2}\Signal_{t+1} \SVDV_t),
\end{align*}
where the second equality follows from  $\Signal_{t+1} = \EssS_{t+1}\SVDV_{t+1}^\T$ (cf.~\eqref{eq:signal_to_esssignal}). 
Apply Lemma~\ref{lem:S-surrogate} with $Q=0$ to see that 
\begin{equation}
  \Signal_{t+1}\SVDV_t 
  = (I+\eta\Err{\ref{lem:S-surrogate}}_t)\left((1-\eta)I+\eta(\SigmaTruth^2+\lambda I)
    (\EssS_t\EssS_t^\T +\lambda I)^{-1}\right)\EssS_t,
  \label{eqn:EssS-update-approx}
\end{equation}
where $\Err{\ref{lem:S-surrogate}}_t\in\mathbb{R}^{r_\star\times r_\star}$ satisfies $ \|\Err{\ref{lem:S-surrogate}}_t\|\le\frac{ 1}{200(C_{\ref{lem:p1.5}.a}+1)^4\kappa^5}$. 
To simplify the notation, we denote 
\[Y_t \defeq (\SigmaTruth^2+\lambda I)^{-1/2}\EssS_t,\]
which allows us to write~\eqref{eqn:EssS-update-approx} as
\begin{align}
  &(\SigmaTruth^2+\lambda I)^{-1/2}\Signal_{t+1}\SVDV_t
  \nonumber\\
  &\quad =  
  \left(I+\eta(\SigmaTruth^2+\lambda I)^{-1/2}\Err{\ref{lem:S-surrogate}}_t(\SigmaTruth^2+\lambda I)^{1/2}\right)
  \Big((1-\eta)I + \eta\big(Y_tY_t^\T+\lambda(\SigmaTruth^2+\lambda I)^{-1}\big)^{-1}
  \Big)Y_t.
  \label{eqn:pf-lem-p2-main}
\end{align}
Note that
\begin{align}
  \|(\SigmaTruth^2+\lambda I)^{-1/2}\Err{\ref{lem:S-surrogate}}_t(\SigmaTruth^2+\lambda I)^{1/2}\|
  &\le \|(\SigmaTruth^2+\lambda I)^{-1/2}\| \cdot \|(\SigmaTruth^2+\lambda I)^{1/2}\| \cdot \|\Err{\ref{lem:S-surrogate}}_t\|
  \nonumber\\
  &\le \kappa\|\Xtruth\|^{-1}\cdot(2\|\Xtruth\|) \cdot \|\Err{\ref{lem:S-surrogate}}_t\| \nonumber \\
  & \le 2\kappa \cdot \frac{ 1}{200(C_{\ref{lem:p1.5}.a}+1)^4\kappa^5} \le 1/32,
  \label{eqn:pf-lem-p2-aux}
\end{align}
where in the second inequality we used $\lambda\le c_\lambda\|\Mtruth\|\le\|\Xtruth\|^2$
as $c_\lambda\le 1$, and in the third inequality we used the claimed bound of $\|\Err{\ref{lem:S-surrogate}}_t\|$. Therefore, it follows that
\begin{equation}\label{eq:pumpkin_pie}
\smin\left(I+\eta(\SigmaTruth^2+\lambda I)^{-1/2}\Err{\ref{lem:S-surrogate}}_t(\SigmaTruth^2+\lambda I)^{1/2}\right)
\ge 1-\eta/32  .
\end{equation}
On the other hand, using $\smin(AB) \geq \smin(A)\smin(B)$ for any matrices $A,B$, it is obvious that 
\[
  \smin\Big(\!\left((1-\eta)I + 
      \eta(Y_tY_t^\T+\lambda(\SigmaTruth^2+\lambda I)^{-1})^{-1}
    \right)Y_t
  \Big) \ge (1-\eta)\smin(Y_t),
\]
which in turn implies that 
\[\smin\big((\SigmaTruth^2+\lambda I)^{-1/2}\Signal_{t+1}\SVDV_t\big)
  \ge (1-\eta/32)(1-\eta) \smin(Y_t) 
  \ge (1-2\eta) \smin(Y_t),
\]
as long as $\eta\le c_\eta$ for  some sufficiently small constant $c_\eta$. 
This proves the first part of Lemma~\ref{lem:p2}.

Now we move to the second part assuming $\smin(Y_t)\le 1/3$.
Using the assumption $\lambda\le  c_\lambda \smin(\Mtruth)$, we see that 
\begin{equation*}
  \|\lambda(\SigmaTruth^2+\lambda I)^{-1}\|\le  c_\lambda.
\end{equation*}
Given that $c_\lambda$ is sufficiently small (such that $c_\lambda\le c_{\ref{lem:la-aux}}$, where $c_{\ref{lem:la-aux}}$ is the positive constant  in Lemma \ref{lem:la-aux}), 
one may apply Lemma \ref{lem:la-aux} with $Y=Y_t$ and
$\Lambda=\lambda(\SigmaTruth^2+\lambda I)^{-1}$ to obtain
\begin{align*}
  \smin\big((\SigmaTruth^2+\lambda I)^{-1/2}\Signal_{t+1}\SVDV_t\big)
 & \ge  \smin\left(I+\eta(\SigmaTruth^2+\lambda I)^{-1/2}\Err{\ref{lem:S-surrogate}}_t(\SigmaTruth^2+\lambda I)^{1/2}\right)
  \left(1+\frac16\eta\right)\smin(Y_t) \\
  & \overset{\mathrm{(i)}}{\ge} (1-\eta/32)\left(1+\frac16\eta\right)\smin(Y_t) \overset{\mathrm{(ii)}}{\ge} \left(1+\frac18\eta \right)\smin(Y_t),
\end{align*}
where (i) uses \eqref{eq:pumpkin_pie}, and (ii) follows as long as $\eta\le c_\eta$ for  some sufficiently small constant $c_\eta$.  The desired conclusion follows.

\subsection{Proof of Corollary~\ref{cor:p2}}\label{sec:proof-cor-1}
We will prove a strengthened version of \eqref{eqn:cor-p2}, that is
\begin{equation}\label{eqn:cor-p2-strengthen}
  \smin\left((\SigmaTruth^2+\lambda I)^{-1/2}\EssS_t\right)
  \ge 1/\sqrt{10}.
\end{equation}
It is clear that~\eqref{eqn:cor-p2-strengthen} implies~\eqref{eqn:cor-p2}. Indeed, for each $u\in\mathbb{R}^{r_\star}$, by taking $v = (\SigmaTruth^2 + \lambda I)^{1/2}u$, we have
\begin{align*}
u^\T \EssS_t\EssS_t^\T u = v^\top (\SigmaTruth^2+\lambda I)^{-1/2}\EssS_t\EssS_t^\T (\SigmaTruth^2+\lambda I)^{-1/2}v \geq \frac{1}{10}\|v\|^2 \geq \frac{1}{10}u^\T \SigmaTruth^2 u,
\end{align*}
which implies \eqref{eqn:cor-p2}. It then boils down to establish \eqref{eqn:cor-p2-strengthen}. 

\paragraph{Step 1: establishing the claim for a midpoint $t_2$.}
From Lemma~\ref{lem:p1.5} we know that
\[
  \smin\left((\SigmaTruth^2+\lambda I)^{-1/2}\EssS_{t_1}\right)
  \ge \|\SigmaTruth^2+\lambda I\|^{-1/2}\smin(\EssS_{t_1})
 \overset{\mathrm{(i)}}{ \ge} (c_\lambda+1)^{-1/2}\|\Xtruth\|^{-1}\cdot \alpha^2/\|\Xtruth\|
  \ge \frac{1}{3}(\alpha/\|\Xtruth\|)^2,
\]
where (i) follows from 
the assumption~\eqref{eqn:lambda-cond} and Lemma~\ref{lem:p1.5}, 
and the last inequality follows by choosing $c_\lambda\le 1$. 
By the second part of Lemma~\ref{lem:p2}, starting from $t_1$,
whenever $\smin((\SigmaTruth^2+\lambda I)^{-1/2}\EssS_{t})< 1/\sqrt{10}<1/3$, it would increase exponentially with rate at least $(1+\frac{\eta}8)$. On the other end, it is easy to verify, 
given that $\eta\le c_\eta$ is sufficiently small,
\[
\left(1+\frac{\eta}8\right)^{\frac{16}{\eta}\log\big(\frac{3}{\sqrt{10}}\frac{\|\Xtruth\|^2}{\alpha^2}\big)} 
\ge \frac{3\|\Xtruth\|^2}{\sqrt{10}\alpha^2} 
\ge \frac{1}{\sqrt{10}} \frac{1}{\smin\left((\SigmaTruth^2+\lambda I)^{-1/2}\EssS_{t_1}\right)}.
\]
Therefore, it takes at most $\frac{16}{\eta}\log\Big(\frac{3}{\sqrt{10}}\frac{\|\Xtruth\|^2}{\alpha^2}\Big)\le \Tcvg/16$  more iterations to make $\smin((\SigmaTruth^2+\lambda I)^{-1/2}\EssS_{t})$ 
grow to at least $1/\sqrt{10}$. Equivalent, for some $t_2: t_1\le t_2\le t_1 + \Tcvg/16$, 
we have
\[\smin\left((\SigmaTruth^2+\lambda I)^{-1/2}\EssS_{t_2}\right)\ge 1/\sqrt{10}.\]

\paragraph{Step 2: establishing the claim for all $t\in[t_2,\Tmax]$.}
It remains to show that \eqref{eqn:cor-p2-strengthen} 
continues to hold for all $t\in[t_2,\Tmax]$. 
We prove this by induction on $t$.

Assume that \eqref{eqn:cor-p2-strengthen} holds for some $t\in[t_2,\Tmax-1]$. 
We show that it will also hold for $t+1$. 
We divide the proof into two cases. 

\paragraph{Case 1.} If $\smin((\SigmaTruth^2+\lambda I)^{-1/2}\EssS_{t})\le 1/3$, 
we deduce from the second part of Lemma~\ref{lem:p2} that 
\[\smin\left((\SigmaTruth^2+\lambda I)^{-1/2}\EssS_{t+1}\right)
\ge\left(1+\frac{\eta}{8}\right)
\smin\left((\SigmaTruth^2+\lambda I)^{-1/2}\EssS_{t}\right)
\ge\smin\left((\SigmaTruth^2+\lambda I)^{-1/2}\EssS_{t}\right),\] 
which by the induction hypothesis is no less than $1/\sqrt{10}$, as desired.

\paragraph{Case 2.} If $\smin((\SigmaTruth^2+\lambda I)^{-1/2}\EssS_{t})>1/3$,
the first part of Lemma~\ref{lem:p2} yields 
\[\smin\left((\SigmaTruth^2+\lambda I)^{-1/2}\EssS_{t+1}\right)
\ge(1-2\eta)\smin\left((\SigmaTruth^2+\lambda I)^{-1/2}\EssS_{t}\right)
\ge(1-2\eta)/3,\] 
which is greater than $1/\sqrt{10}$ provided $\eta\le c_\eta\le 1/100$,
as desired. 

Combining the two cases completes the proof.

\subsection{Proof of Lemma \ref{lem:p2.5}}
\label{pf:lem:p2.5} 

For simplicity, in this section we denote 
\begin{equation}\label{eqn:Gamma-def}
  \Gamma_t \defeq \SigmaTruth^{-1}\EssS_t\EssS_t^\T\SigmaTruth^{-1}-I =\SigmaTruth^{-1} ( \EssS_t\EssS_t^\T -\SigmaTruth^2)   \SigmaTruth^{-1} .
\end{equation}
It turns out that Lemma \ref{lem:p2.5} follows naturally from the following 
technical lemma, whose proof is deferred to the end of this section. 
\begin{lemma}\label{lem:local-unify}
  For any $t : t_2\le t\le\Tmax$, one has
  \begin{equation}\label{eqn:local-Gamma-bound}
    \uinorm{\Gamma_{t+1}}
    \le (1-\eta)\uinorm{\Gamma_t}
    +\eta\frac{C_{\ref{lem:local-unify}}\kappa^6}{\|\Xtruth\|^{2}}\uinorm{\Utruth^\T\Delta_t}
    +\frac{1}{16}\eta\|\Xtruth\|^{-1}\uinorm{\Misalign_t\EssS_t^{-1}\SigmaTruth}
    +\eta\left(\frac{\|\Overpar_t\|}{\|\Xtruth\|}\right)^{7/12},
  \end{equation}
  where $C_{\ref{lem:local-unify}}\lesssim c_\lambda^{-1/2}$ is some positive constant and $\uinorm{\cdot}$ can either be the Frobenius norm or the spectral norm.
\end{lemma}

From Lemma~\ref{lem:Delta-bound}, we know that 
$\|\Utruth^\T\Delta_t\|\le\|\Delta_t\|\le\frac{\|\Xtruth\|^2}{300C_{\ref{lem:local-unify}}\kappa^4}$ 
as $c_\delta$ is sufficiently small. 
Similarly, $\|\Misalign_t\EssS_t^{-1}\SigmaTruth\|\le\|\Xtruth\|/100$ 
and $(\|\Overpar_t\|/\|\Xtruth\|)^{7/12}\le 1/300$ by Lemma \ref{lem:p1.5}.
Applying Lemma~\ref{lem:local-unify} with the spectral norm, 
we prove Lemma~\ref{lem:p2.5} as desired.

\paragraph{Proof of Lemma \ref{lem:local-unify}.}
We start by rewriting \eqref{eqn:S-update-approx} as
\begin{align}
   \Signal_{t+1}  
  &=\big((1-\eta)I +\eta(\SigmaTruth^2+\lambda I)(\EssS_t\EssS_t^\T+\lambda I)^{-1}\big)\EssS_t\SVDV_t^\T  
  +\eta\Err{g}_t 
  \nonumber\\
    &=\big(I  - \eta(\EssS_t\EssS_t^\T+\lambda I)(\EssS_t\EssS_t^\T+\lambda I)^{-1}+\eta(\SigmaTruth^2+\lambda I)(\EssS_t\EssS_t^\T+\lambda I)^{-1}\big)\EssS_t \SVDV_t^\T  
  +\eta\Err{g}_t  
  \nonumber\\
  &=\big(I-\eta(\EssS_t\EssS_t^\T-\SigmaTruth^2)(\EssS_t\EssS_t^\T+\lambda I)^{-1}\big)\EssS_t \SVDV_t^\T  
  +\eta\Err{g}_t  , \label{eq:pineapple}
\end{align}
where
\begin{align}
\Err{g}_t=\Err{a}_t(\EssS_t\EssS_t^\T+\lambda I)^{-1}\EssS_t\SVDV_t^\T + \Err{b}_t.
\end{align}
By Corollary \ref{cor:p2}, we have $\smin(\EssS_t)^2\ge\frac{1}{100}\smin(\Mtruth)$ for $t \in[t_2,\Tmax]$,
so 
\begin{align*}
\|(\EssS_t\EssS_t^\T+\lambda I)^{-1}\EssS_t \SVDV_t^\T\|
\le\|(\EssS_t\EssS_t^\T+\lambda I)^{-1/2}\|
\|(\EssS_t\EssS_t^\T+\lambda I)^{-1/2}\EssS_t\|\le\smin^{-1}(\EssS_t)\lesssim 1/\smin(\Xtruth).
\end{align*}
Combined with the error bounds~\eqref{eqn:Err-1-bound},~\eqref{eqn:Err-2-bound}, we have 
for some universal constant $C>0$ that
\begin{equation}
  \label{eqn:Err-S2-bound}
  \uinorm{\Err{g}_t}\le   \uinorm{\Err{a}_t} + \eta  \uinorm{\Err{b}_t} \leq 
  \frac{C\kappa}{\|\Xtruth\|}\uinorm{\Utruth^\T\Delta_t}
  +Cc_{\ref{lem:update-approx}}\kappa^{-5}\uinorm{\Misalign_t\EssS_t^{-1}\SigmaTruth}
  +C\|\Overpar_t\|^{3/4}\|\Xtruth\|^{1/4}.
\end{equation}

\paragraph{Step 1: deriving a recursion of $\Gamma_t$.}
Define
\[
A_t := \big(I-\eta(\EssS_t\EssS_t^\T-\SigmaTruth^2)(\EssS_t\EssS_t^\T+\lambda I)^{-1}\big)\EssS_t \SVDV_t^\T.
\] 
Then we can rewrite \eqref{eq:pineapple} as $A_t=\Signal_{t+1}-\eta\Err{g}_t$, 
and by rearranging $A_tA_t^\T=(\Signal_{t+1}-\eta\Err{g}_t)(\Signal_{t+1}-\eta\Err{g}_t)^\T$ in view of \eqref{eq:signal_to_esssignal}, it follows that
\begin{align*}
 \EssS_{t+1}\EssS_{t+1}^\T= \Signal_{t+1}\Signal_{t+1}^\T & =A_t A_t^\T + \eta(\|\Signal_{t+1}\|+\|\Err{g}_t\|)
  (\Err{g}_t Q_1 + Q_2 {\Err{g}_t}^\T)   \\
  & =: A_t A_t^\T + \eta \Err{f}_t
\end{align*}
for some matrices $Q_1, Q_2$ with $\|Q_1\|, \|Q_2\|\le 1  $.
By mapping both sides of the above equation by $(\cdot)\mapsto\SigmaTruth^{-1}(\cdot)\SigmaTruth^{-1} - I$,
we obtain 
\begin{align}\label{eqn:local-update}
  \Gamma_{t+1}
  & = \big(I-\eta\Gamma_t(I+\Gamma_t+\lambda\SigmaTruth^{-2})^{-1}\big)(\Gamma_t+I)
      \big(I-\eta(I+\Gamma_t+\lambda\SigmaTruth^{-2})^{-1}\Gamma_t\big)
      - I + \eta\SigmaTruth^{-1}\Err{f}_t\SigmaTruth^{-1},
\end{align}
where we recall the definition of $\Gamma_t$ in \eqref{eqn:Gamma-def}.

\paragraph{Step 2: simplify the recursion.} Note that $\smin(\SigmaTruth^{-1}\EssS_t)\ge1/10$ implies $I+\Gamma_t\succeq \frac{1}{100}I$. 
From our assumption $\lambda\le c_\lambda\smin(\Mtruth)$, it follows that
$\|\lambda\SigmaTruth^{-2}\|\le c_\lambda\le 1/200\le \frac12\smin(I+\Gamma_t)$, 
thus in virtue of Lemma~\ref{prop:inv} we have
\begin{equation*}
  (I+\Gamma_t+\lambda\SigmaTruth^{-2})^{-1}=(I+\Gamma_t)^{-1}+(I+\Gamma_t)^{-1}(c_\lambda Q')(I+\Gamma_t)^{-1},
\end{equation*}
for some matrix $Q'$ with $\|Q'\|\le 2$.
Plugging this into~\eqref{eqn:local-update} yields
\begin{align}
  \Gamma_{t+1} 
  & = \big(I - \eta \Gamma_t(I+\Gamma_t)^{-1}\big) (\Gamma_t + I) \big(I - \eta (I+\Gamma_t)^{-1}\Gamma_t\big)
  + \eta\Err{h}_t + \eta\SigmaTruth^{-1}\Err{f}_t\SigmaTruth^{-1}
  \nonumber\\
  & = (1-2\eta)\Gamma_t+\eta^2\Gamma_t^2(1+\Gamma_t)^{-1} 
  + \eta \Err{h}_t + \eta\SigmaTruth^{-1}\Err{f}_t\SigmaTruth^{-1},
  \label{eqn:Gamma-update-approx}
\end{align} 
where the additional error term $\Err{h}_t$ is defined by
\begin{align}
  \Err{h}_t:=&\Gamma_t(I+\Gamma_t)^{-1}(c_\lambda Q')(1-\eta\Gamma_t(I+\Gamma_t)^{-1})
  +(1-\eta\Gamma_t(I+\Gamma_t)^{-1})(c_\lambda Q')(I+\Gamma_t)^{-1}\Gamma_t
  \nonumber\\
  &+\eta\Gamma_t(I+\Gamma_t)^{-1}(c_\lambda Q')
  (I+\Gamma_t)^{-2}(c_\lambda Q')(I+\Gamma_t)^{-1}\Gamma_t.
\end{align}

\paragraph{Step 3: controlling the error terms.} We now control the error terms in \eqref{eqn:Gamma-update-approx} separately.
\begin{itemize}
\item By~\eqref{eqn:p1.5-4} we have $\|\Signal_{t+1}\|\le  C_{\ref{lem:p1.5}.a} \kappa\|\Xtruth\|$, 
and by controlling the right hand side of \eqref{eqn:Err-S2-bound} using~\eqref{eqn:p1.5-3},~\eqref{eqn:p1.5-2-var}, 
and~\eqref{eqn:Delta-bound-coarse} in Lemma~\ref{lem:Delta-bound}, 
it is evident that $\|\Err{g}_t\|\le\kappa\|\Xtruth\|$. Hence,
the term $\Err{f}_t$ obeys
\begin{align}
  \uinorm{\Err{f}_t}  
  &\le (C_{\ref{lem:p1.5}.a}+1)\kappa^3 \|\Xtruth\|\cdot\uinorm{\Err{g}_t}
  \nonumber\\
  &\le C'C_{\ref{lem:p1.5}.a}
  \left(\kappa^4\uinorm{\Utruth^\T\Delta_t} 
  + c_{\ref{lem:update-approx}}\kappa^{-2}\|\Xtruth\|\uinorm{\Misalign_t\EssS_t^{-1}\SigmaTruth}+\kappa\|\Overpar_t\|^{3/4}\|\Xtruth\|^{5/4}\right)
  ,
  \label{eqn:Err-f-bound}
\end{align}
where $C'>0$ is again some universal constant.

\item Since $\Gamma_t\succeq \frac{1}{100}I-I=-\frac{99}{100}I$ as already proved, 
it is easy to see that $\|(1+\Gamma_t)^{-1}\|\le C$ 
and $\|\Gamma_t(1+\Gamma_t)^{-1}\|\le C$ for some universal constant $C>0$. Thus,
\begin{align}\label{eqn:local-err-2}
  \uinorm{\Err{h}_t}
  \le 2c_\lambda C(1+\eta C)\|Q'\|\cdot\uinorm{\Gamma_t}
  + \eta c_\lambda^2 C^4 \|Q'\|^2 \uinorm{\Gamma_t}
  \le \frac12 \uinorm{\Gamma_t},
\end{align}
where the last line follows by using $\|Q'\|\le 2$ and by choosing $c_\lambda$, $c_\eta$ sufficiently small. 

\item We still need to control $\eta^2\Gamma_t^2(1+\Gamma_t)^{-1}$. This can be accomplished by invoking  $\|\Gamma_t(1+\Gamma_t)^{-1}\|\le C$ again. In fact, we have
\begin{equation}
  \label{eqn:local-err-1}
  \eta^2\uinorm{\Gamma_t^2(1+\Gamma_t)^{-1}}
  \le \eta\cdot\eta\|\Gamma_t(1+\Gamma_t)^{-1}\|\cdot\uinorm{\Gamma_t}
  \le \eta\cdot\eta C\uinorm{\Gamma_t}
  \le \frac{\eta}{2}\uinorm{\Gamma_t}
\end{equation}
provided that $\eta\le c_\eta$ is sufficiently small. 
\end{itemize}

Plugging~\eqref{eqn:Err-f-bound},~\eqref{eqn:local-err-2},~\eqref{eqn:local-err-1} into \eqref{eqn:Gamma-update-approx}, 
we readily obtain

\begin{align*}
  \uinorm{\Gamma_{t+1}}
  & \le (1-2\eta)\uinorm{\Gamma_t} + \frac{\eta}{2}\uinorm{\Gamma_t}
  + \frac{\eta}{2}\uinorm{\Gamma_t} + \eta\kappa^2\|\Xtruth\|^{-2}\uinorm{\Err{f}_t}
  \\
  & \le (1-\eta)\uinorm{\Gamma_t} + 
  \eta\frac{C'C_{\ref{lem:p1.5}.a}\kappa^4}{\|\Xtruth\|^2}\uinorm{\Utruth^\T\Delta_t}
  + \eta c_{\ref{lem:update-approx}}C'C_{\ref{lem:p1.5}.a} \|\Xtruth\|^{-1} \uinorm{\Misalign_t\EssS_t^{-1}\SigmaTruth} 
  + \eta C'C_{\ref{lem:p1.5}.a}\kappa^3\|\Overpar_t\|^{3/4}\|\Xtruth\|^{-3/4}
  \\
  & \le (1-\eta)\uinorm{\Gamma_t}
  +\eta\frac{C_{\ref{lem:local-unify}}\kappa^4}{\|\Xtruth\|^{2}}\uinorm{\Utruth^\T\Delta_t}
  +\frac{1}{16}\eta\|\Xtruth\|^{-1}\uinorm{\Misalign_t\EssS_t^{-1}\SigmaTruth}
  + \eta \left(\frac{\|\Overpar_t\|}{\|\Xtruth\|}\right)^{7/12},
\end{align*}
where in the last line we set $C_{\ref{lem:local-unify}}=C'C_{\ref{lem:p1.5}.a}$, chose $c_{\ref{lem:update-approx}}$ sufficiently small and used~\eqref{eqn:p1.5-2-var}.
Finally note that $C_{\ref{lem:local-unify}}\lesssim C_{\ref{lem:p1.5}.a}\lesssim c_\lambda^{-1/2}$ as desired.

\subsection{Proof of Corollary \ref{cor:p2.5}}
\label{sec:pf:cor:p2.5}
From Lemma~\ref{lem:p2.5}, it is elementary (e.g.,~by induction on $t$) to show that 
\begin{equation}\label{eqn:pf:cor:p2.5-bound}
  \big\|\SigmaTruth^{-1}(\EssS_{t}\EssS_{t}^\T-\SigmaTruth^2)\SigmaTruth^{-1} \big\|
  \le (1-\eta)^{t-t_2}
  \big\|\SigmaTruth^{-1}(\EssS_{t_2}\EssS_{t_2}^\T-\SigmaTruth^2)\SigmaTruth^{-1} \big\|
  +\frac{1}{100},
  \quad\forall t\in[t_2,\Tmax].
\end{equation}
Suppose for the moment that 
\begin{equation}\label{eqn:pf:cor:p2.5-init}
  \big\|\SigmaTruth^{-1}(\EssS_{t_2}\EssS_{t_2}^\T-\SigmaTruth^2)\SigmaTruth^{-1} \big\|
  \le C_{\ref{lem:p1.5}.a}^2\kappa^4,
\end{equation}
where $C_{\ref{lem:p1.5}.a}$ is given in Lemma \ref{lem:p1.5}.
Then given that $\eta\le c_\eta$ for some sufficiently small $c_\eta$, we have 
$\log(1-\eta)\ge-\eta/2$. As a result,  
if $t_3-t_2\ge 8\log(10C_{\ref{lem:p1.5}.a}\kappa)/\eta\ge\log(C_{\ref{lem:p1.5}.a}^{-2}\kappa^{-4}/100)/\log(1-\eta)$, 
we have $(1-\eta)^{t_3-t_2}\le C_{\ref{lem:p1.5}.a}^{-2}\kappa^{-4}/100$. 
When $\Ccvg$ is sufficiently large we may choose such $t_3$ which simultaneously satisfies 
$t_3\le t_2+\Tcvg/16\le\Tmax$ since 
$8\log(10C_{\ref{lem:p1.5}.a}\kappa)/\eta \le\frac{\Ccvg}{32\eta}\log(\|\Xtruth\|/\alpha)=\Tcvg/32$.
Invoking \eqref{eqn:pf:cor:p2.5-bound}, we obtain
\begin{equation}
  \big\|\SigmaTruth^{-1}(\EssS_{t_3}\EssS_{t_3}^\T-\SigmaTruth^2)\SigmaTruth^{-1} \big\|
  \le (C_{\ref{lem:p1.5}.a}^{-2}\kappa^{-4}/100)(C_{\ref{lem:p1.5}.a}^2\kappa^4) + \frac{1}{100} 
  =\frac{1}{50}\le\frac{1}{10},
\end{equation}
which implies the desired bound \eqref{eqn:cor:p2.5}.

\paragraph{Proof of inequality~\eqref{eqn:pf:cor:p2.5-init}.}
It is straightforward to verify that 
\begin{equation*}
  \big\|\SigmaTruth^{-1}(\EssS_{t_2}\EssS_{t_2}^\T-\SigmaTruth^2)\SigmaTruth^{-1} \big\|
  \le \max\left(\|\SigmaTruth^{-1}\EssS_{t_2}\|^2-1, 1-\smin^2(\SigmaTruth^{-1}\EssS_{t_2})\right),
\end{equation*}
which combined with~\eqref{eqn:p1.5-4} implies that 
\[\|\SigmaTruth^{-1}\EssS_{t_2}\|^2-1
\le \|\SigmaTruth^{-1}\|^2\|\EssS_{t_2}\|^2
\le \smin^{-2}(\Xtruth)C_{\ref{lem:p1.5}.a}^2\kappa^2\|\Xtruth\|^2
=C_{\ref{lem:p1.5}.a}^2\kappa^4.
\]
In addition, by Corollary~\ref{cor:p2} we have
\begin{equation*}
  1-\smin^2(\SigmaTruth^{-1}\EssS_{t_2})
  \le 1-\frac{1}{10}=\frac{9}{10}.
\end{equation*}
Choosing $C_{\ref{lem:p1.5}.a}$ sufficiently large (say $C_{\ref{lem:p1.5}.a}\ge 1$) yields $C_{\ref{lem:p1.5}.a}^2\kappa^4\ge 9/10$, and hence 
the claim~\eqref{eqn:pf:cor:p2.5-init}.

\section{Proofs for Phase III}
\label{sec:pf:lem:p3}
To characterize the behavior of $\|\myX_t\myX_t^\T - \Mtruth\|_{\fro}$, 
it is particularly helpful to consider the following decomposition 
into three error terms related to the signal term, the misalignment term, 
and the overparametrization term. 

\begin{lemma}\label{lem:loss-decomp}
For all $t\ge t_3$, as long as 
$\|\SigmaTruth^{-1}(\EssS_{t}\EssS_{t}^\T-\SigmaTruth^2)\SigmaTruth^{-1}\|\le 1/10$,
one has
\begin{equation*}
  \|\myX_t\myX_t^\T - \Mtruth\|_{\fro}
  \le 4\|\Xtruth\|^2
  \left(
    \|\SigmaTruth^{-1}(\EssS_{t}\EssS_{t}^\T-\SigmaTruth^2)\SigmaTruth^{-1}\|_{\fro} 
    + \|\Xtruth\|^{-1}\|\Misalign_{t}\EssS_{t}^{-1}\SigmaTruth\|_{\fro}
  \right)
  + 4\|\Xtruth\|\|\Overpar_t\|.
\end{equation*}
\end{lemma}

Note that the overparametrization error $\|\Overpar_t\|$ stays small, 
as stated in~\eqref{eqn:p1.5-2} and \eqref{eqn:p1.5-2-var}.
Therefore we only need to focus on the shrinkage of the first two terms 
$\|\SigmaTruth^{-1}(\EssS_{t}\EssS_{t}^\T-\SigmaTruth^2)\SigmaTruth^{-1}\|_{\fro}
  + \|\Xtruth\|^{-1}\|\Misalign_{t}\EssS_{t}^{-1}\SigmaTruth\|_{\fro}$, which is the focus of the lemma below. 
      
\begin{lemma}\label{lem:p3-aux}
  For any $t: t_3\le t\le\Tmax$, one has
  \begin{align}
    &\|\SigmaTruth^{-1}(\EssS_{t+1}\EssS_{t+1}^\T-\SigmaTruth^2)\SigmaTruth^{-1}\|_{\fro} 
    + \|\Xtruth\|^{-1}\|\Misalign_{t+1}\EssS_{t+1}^{-1}\SigmaTruth\|_{\fro}
    \nonumber\\
    & \quad \le \left(1-\frac{\eta}{10}\right)
    \left(\|\SigmaTruth^{-1}(\EssS_{t}\EssS_{t}^\T-\SigmaTruth^2)\SigmaTruth^{-1}\|_{\fro}
        + \|\Xtruth\|^{-1}\|\Misalign_{t}\EssS_{t}^{-1}\SigmaTruth\|_{\fro}
    \right)
    + \eta\left(\frac{\|\Overpar_t\|}{\|\Xtruth\|}\right)^{1/2}.
    \label{eqn:p3-aux}
  \end{align}
  In particular, $\|\SigmaTruth^{-1}(\EssS_{t+1}\EssS_{t+1}^\T-\SigmaTruth^2)\SigmaTruth^{-1}\|\le 1/10$
  for all $t$ such that $t_3\le t\le \Tmax$.
\end{lemma}

We now show how Lemma~\ref{lem:p3} is implied by the above two lemmas.
To begin with, we apply Lemma~\ref{lem:p3-aux} repeatedly to obtain the following bound 
for all $t\in[t_3,\Tmax]$:
\begin{align}
  &\|\SigmaTruth^{-1}(\EssS_{t}\EssS_{t}^\T-\SigmaTruth^2)\SigmaTruth^{-1}\|_{\fro}
  + \|\Xtruth\|^{-1}\|\Misalign_{t}\EssS_{t}^{-1}\SigmaTruth\|_{\fro} \nonumber
  \\
  &\quad \le \left(1-\frac{\eta}{10}\right)^{t-t_3}\left(
    \|\SigmaTruth^{-1}(\EssS_{t_3}\EssS_{t_3}^\T-\SigmaTruth^2)\SigmaTruth^{-1}\|_{\fro}
  + \|\Xtruth\|^{-1}\|\Misalign_{t_3}\EssS_{t_3}^{-1}\SigmaTruth\|_{\fro}
  \right) + 10\max_{t_3\le\tau\le t}\left(\frac{\|\Overpar_\tau\|}{\|\Xtruth\|}\right)^{1/2}, \label{eq:iterative-decay}
\end{align}
which motivates us to control the error at time $t_3$.

We know from Corollary~\ref{cor:p2.5} that 
$\|\SigmaTruth^{-1}(\EssS_{t_3}\EssS_{t_3}^\T-\SigmaTruth^2)\SigmaTruth^{-1}\|\le 1/10$. 
Since $\SigmaTruth^{-1}(\EssS_{t_3}\EssS_{t_3}^\T-\SigmaTruth^2)\SigmaTruth^{-1}$
is a $r_\star\times r_\star$ matrix, we have
$\|\SigmaTruth^{-1}(\EssS_{t_3}\EssS_{t_3}^\T-\SigmaTruth^2)\SigmaTruth^{-1}\|_{\fro}\le\sqrt{r_\star}/10$. 
In addition, we infer from~\eqref{eqn:p1.5-3} that 
\[
  \|\Misalign_{t_3}\EssS_{t_3}^{-1}\SigmaTruth\|_{\fro}
  \le \sqrt{r_\star}\|\Misalign_{t_3}\EssS_{t_3}^{-1}\SigmaTruth\|
  \le \sqrt{r_\star}c_{\ref{lem:p1.5}}\kappa^{-C_\delta/2}\|\Xtruth\|
  \le \sqrt{r_\star}\|\Xtruth\|/10,
\]
as long as $c_{\ref{lem:p1.5}}$ is sufficiently small.
Combine the above two bounds to arrive at the conclusion that  
\begin{equation}
  \|\SigmaTruth^{-1}(\EssS_{t_3}\EssS_{t_3}^\T-\SigmaTruth^2)\SigmaTruth^{-1}\|_{\fro}
  + \|\Xtruth\|^{-1}\|\Misalign_{t_3}\EssS_{t_3}^{-1}\SigmaTruth\|_{\fro}
  \le \frac{\sqrt{r_\star}}{10}+\|\Xtruth\|^{-1}\frac{\sqrt{r_\star}\|\Xtruth\|}{10}
  = \frac{\sqrt{r_\star}}{5}. \label{eq:t-3-condition}
\end{equation}
Combining the two inequalities~\eqref{eq:iterative-decay} and~\eqref{eq:t-3-condition} yields 
for all $t \in [t_3, \Tmax]$
\begin{equation*}
  \|\SigmaTruth^{-1}(\EssS_{t}\EssS_{t}^\T-\SigmaTruth^2)\SigmaTruth^{-1}\|_{\fro}
  + \|\Xtruth\|^{-1}\|\Misalign_{t}\EssS_{t}^{-1}\SigmaTruth\|_{\fro}
  \le \frac15\left(1-\frac{\eta}{10}\right)^{t-t_3}\sqrt{r_\star}
  + 10\max_{t_3\le\tau\le t}\left(\frac{\|\Overpar_\tau\|}{\|\Xtruth\|}\right)^{1/2}.
\end{equation*}
We can then invoke Lemma~\ref{lem:loss-decomp} to see that 
\begin{align*}
  \|\myX_t\myX_t^\T-\Mtruth\|_{\fro}
  &\le \frac{4\|\Xtruth\|^2}{5}\left(1-\frac{\eta}{10}\right)^{t-t_3}\sqrt{r_\star}
  + 40\|\Xtruth\|^2\max_{t_3\le\tau\le t}\left(\frac{\|\Overpar_\tau\|}{\|\Xtruth\|}\right)^{1/2}
  + 4\|\Xtruth\|\|\Overpar_t\|
  \\
  &\le \left(1-\frac{\eta}{10}\right)^{t-t_3}\sqrt{r_\star}\|\Mtruth\|
  + 80\|\Mtruth\|\max_{t_3\le\tau\le t}\left(\frac{\|\Overpar_{\tau}\|}{\|\Xtruth\|}\right)^{1/2},
\end{align*}
where in the last line we use $\|\Overpar_t\|\le\|\Xtruth\|$---an implication of~\eqref{eqn:p1.5-2-var}. 
To see this, the assumption \eqref{eqn:alpha-cond} implies that 
$\alpha\le\|\Xtruth\|$ as long as $\eta\le 1/2$ and $C_\alpha\ge 4$, 
which in turn implies $\|\Overpar_t\|\le\alpha^{2/3}\|\Xtruth\|^{1/3}\le\|\Xtruth\|$.
This completes the proof for the first part of Lemma~\ref{lem:p3} with $c_{\ref{lem:p3}}=1/10$.

For the second part of Lemma~\ref{lem:p3}, 
notice that 
\[
8c_{\ref{lem:p3}}^{-1}\max_{t_3\le\tau\le\Tmax}(\|\Overpar_\tau\|/\|\Xtruth\|)^{1/2}
\le \frac12\left(\frac{\alpha}{\|\Xtruth\|}\right)^{1/3}
\]
by~\eqref{eqn:p1.5-2-var}, thus
\begin{equation*}
  \|X_t X_t^\T - \Mtruth\|_{\fro}
  \le (1-c_{\ref{lem:p3}}\eta)^{t-t_3} \sqrt{r_\star} \|\Mtruth\| 
  + \frac12\left(\frac{\alpha}{\|\Xtruth\|}\right)^{1/3}
\end{equation*}
for $t_3\le t\le\Tmax$. 
There exists some iteration number 
$t_4: t_3\le t_4\le t_3 + \frac{2}{c_{\ref{lem:p3}}\eta} \log(\|\Xtruth\|/\alpha) \le t_3 + \Tcvg/16$ 
such that
\[
(1 - c_{\ref{lem:p3}}\eta) ^ {t_4-t_3}
\le \left(\frac{\alpha}{\|\Xtruth\|}\right)^2
\le \frac{1}{2\sqrt{r_\star}} \left(\frac{\alpha}{\|\Xtruth\|}\right)^{1/3},
\]
where the last inequality is due to~\eqref{eqn:alpha-cond}.
It is then clear that $t_4$ has the property claimed in the lemma.

\subsection{Proof of Lemma \ref{lem:loss-decomp}}
Starting from~\eqref{eqn:loss-decomp-primal}, 
we may deduce
\begin{align}
  \|\myX_t\myX_t^\T - \Mtruth\|_{\fro}
  &\le \| \EssS_t\EssS_t^\T - \SigmaTruth^2 \|_{\fro}
  + 2\|\EssS_t\| \|\Misalign_t\|_{\fro} + \|\Misalign_t\| \|\Misalign_t\|_{\fro} + \|\Overpar_t\| \|\Overpar_t\|_{\fro}
  \nonumber\\
  &\le \|\Xtruth\|^2 
  \left(\|\SigmaTruth^{-1}\EssS_t\EssS_t^\T\SigmaTruth^{-1}-I\|_{\fro}
    + 2\|\SigmaTruth^{-1}\EssS_t\|^2 \|\Xtruth\|^{-1} \|\Misalign_t\EssS_t^{-1}\SigmaTruth\|_{\fro}
    + \sqrt{n}\left(\frac{\|\Overpar_t\|}{\|\Xtruth\|}\right)^2
  \right)
  \nonumber\\
  &\le 4\|\Xtruth\|^2
  \left(\|\SigmaTruth^{-1}\EssS_t\EssS_t^\T\SigmaTruth^{-1}-I\|_{\fro}
    + \|\Xtruth\|^{-1} \|\Misalign_t\EssS_t^{-1}\SigmaTruth\|_{\fro}
    + \frac{\|\Overpar_t\|}{\|\Xtruth\|}
  \right),
  \label{eqn:pf:loss-decomp}
\end{align}
where the penultimate line used $\|\Overpar_t\|_{\fro} \le \sqrt{n}\|\Overpar_t\|$, 
and the last line follows from $\|\SigmaTruth^{-1}\EssS_t\|^2
=\|\SigmaTruth^{-1}\EssS_t\EssS_t^\T\SigmaTruth^{-1}\|
\le 1+\|\SigmaTruth^{-1}\EssS_t\EssS_t^\T\SigmaTruth^{-1}-I\|\le 2$ 
(recall that $\|\SigmaTruth^{-1}\EssS_t\EssS_t^\T\SigmaTruth^{-1}-I\|\le 1/10$ 
by assumption) 
and from~\eqref{eqn:p1.5-2-var}.

\subsection{Proof of Lemma \ref{lem:p3-aux}}

Recall the definition of $\Gamma_t$ from~\eqref{eqn:Gamma-def}:
\begin{align*}
\Gamma_t \defeq \SigmaTruth^{-1}\EssS_t\EssS_t^\T\SigmaTruth^{-1}-I.
\end{align*}
Fix any $t\in[t_3,\Tmax]$, if~\eqref{eqn:p3-aux} were true for all $\tau\in[t_3,t]$, 
taking into account that $\|\Overpar_\tau\|/\|\Xtruth\|\le 1/10000$ 
for all $\tau\in[t_3,\Tmax]$ by~\eqref{eqn:p1.5-2-var}, 
we could show by induction that $\|\Gamma_\tau\|\le 1/10$ for all $\tau\in[t_3,t]$. 
Thus it suffices to assume $\|\Gamma_t\|\le 1/10$ and prove \eqref{eqn:p3-aux}.

Apply Lemma~\ref{lem:local-unify} with Frobenius norm to obtain
 \begin{equation}\label{eq:Gamma-local}
    \| {\Gamma_{t+1}} \|_{\fro}
    \le (1-\eta)\| {\Gamma_t} \|_{\fro}
    +\eta\frac{C_{\ref{lem:local-unify}}\kappa^4}{\|\Xtruth\|^{2}} \| {\Utruth^\T\Delta_t} \|_{\fro}
    +\frac{1}{16}\eta\|\Xtruth\|^{-1} \| {\Misalign_t\EssS_t^{-1}\SigmaTruth} \|_{\fro}
    +\eta\left(\frac{\|\Overpar_t\|}{\|\Xtruth\|}\right)^{7/12},
  \end{equation}
In addition,  Lemma~\ref{lem:misalign-update} tells us that 
\begin{equation*}
  \|{\Misalign_{t+1}\EssS_{t+1}^{-1}\SigmaTruth}\|_{\fro}
  \le \left(1-\frac{\eta}{3(\|Z_t\|+\eta)}\right) \|{\Misalign_t\EssS_t^{-1}\SigmaTruth}\|_{\fro}
    + \eta \frac{C_{\ref{lem:misalign-update}}\kappa^6}{c_\lambda\|\Xtruth\|}\|{\Utruth^\T\Delta_t}\|_{\fro}
    + \eta \left(\frac{\|\Overpar_t\|}{\smin(\EssS_t)}\right)^{2/3}\|\Xtruth\|,
\end{equation*}
where $Z_t = \SigmaTruth^{-1}(\EssS_t\EssS_t^\T+\lambda I)\SigmaTruth^{-1}$. 
It is easy to check that $\|Z_t\| \leq 1+\|\Gamma_t\| + c_\lambda\le 2$ as $\| \Gamma_t \| \leq 1/10$ and $c_\lambda$ is sufficiently small.
In addition, one has $\smin(\EssS_t)^2\ge(1-\|\Gamma_t\|)\smin(\Xtruth)^2$ 
and $\|\Overpar_t\|/\smin(\EssS_t)\le (2\kappa)^{-24}$. 
Combine these relationships together to arrive at
\begin{equation}\label{eqn:local-misalign-update}
  \|{\Misalign_{t+1}\EssS_{t+1}^{-1}\SigmaTruth}\|_{\fro}
  \le \left(1-\frac{\eta}{8}\right) \| \Misalign_t\EssS_t^{-1}\SigmaTruth \|_{\fro}
  + \eta \frac{C_{\ref{lem:misalign-update}}\kappa^6}{c_\lambda\|\Xtruth\|}
    \| \Utruth^\T\Delta_t \|_{\fro}
  + \frac12\eta \|\Xtruth\|\left(\frac{\|\Overpar_t\|}{\|\Xtruth\|}\right)^{7/12}.
\end{equation}

Summing up~\eqref{eq:Gamma-local},~\eqref{eqn:local-misalign-update}, 
we obtain
\begin{align}
  &\| \Gamma_{t+1} \|_{\fro} 
  + \|\Xtruth\|^{-1} \| \Misalign_{t+1}\EssS_{t+1}^{-1}\SigmaTruth \|_{\fro}
  \nonumber\\
  &\quad 
  \le \left(1 - \frac{\eta}{8}\right) 
  (\| \Gamma_{t} \|_{\fro} 
  + \|\Xtruth\|^{-1} \| \Misalign_{t}\EssS_{t}^{-1}\SigmaTruth \|_{\fro})
  + \eta\frac{2(C_{\ref{lem:misalign-update}}+C_{\ref{lem:local-unify}}c_\lambda)\kappa^8}{c_\lambda \|\Xtruth\|^2} 
  \| \Utruth^\T\Delta_t \|_{\fro}
  + 2\eta\left(\frac{\|\Overpar_t\|}{\|\Xtruth\|}\right)^{7/12}.
  \label{eqn:loc-with-Delta}
\end{align}
This is close to our desired conclusion, 
but we would need to eliminate $\| \Utruth^\T\Delta_t \|_{\fro}$. 
To this end we observe
\begin{align*}
  \| \Utruth^\T\Delta_t \|_{\fro}
  &\le \sqrt{r_\star}\|\Delta_t\|
  \\
  &\le 8\delta\sqrt{r_\star}
  \left(\|\EssS_t\EssS_t^\T-\SigmaTruth^2\|_{\fro}
    +\|\EssS_t\|\|\Misalign_t\|_{\fro}
  +n\|\Overpar_t\|^2\right)
  \\
  &\le 16c_\delta\kappa^{-4}\|\Xtruth\|^2\left(\| \Gamma_t \|_{\fro}
    +\|\Xtruth\|^{-1} \| \Misalign_t\EssS_t^{-1}\SigmaTruth \|_{\fro}
    +\left(\frac{\|\Overpar_t\|}{\|\Xtruth\|}\right)^{2/3}
  \right),
\end{align*}
where the first line follows from $\Utruth$ being of rank $r_\star$, 
the second line follows from Lemma~\ref{lem:Delta-bound}, 
and the last line follows from~\eqref{eqn:delta-cond} 
and from controlling the sum inside the brackets 
in a similar way as~\eqref{eqn:pf:loss-decomp}.

The conclusion follows from plugging 
the above inequality into~\eqref{eqn:loc-with-Delta}, 
noting that $c_\delta$ can be chosen sufficiently small 
and that $\|\Overpar_t\|/\|\Xtruth\|$ 
is sufficiently small due to~\eqref{eqn:p1.5-2-var}.

\subsection{Proof of Proposition \ref{prop:Ot}}
Recall that in the proof of Lemma~\ref{lem:overpar-update} (Appendix~\ref{pf:lem:overpar-update}), we have shown
\begin{align}
\label{eqn:overpar-update}
\| \Overpar_t \| \le \| \Overpar_t \| + \eta \| \Noise_{t+1} \SVDV_t (\Signal_{t+1} \SVDV_t)^{-1} \| \cdot \|\Err{b}_t\| + \eta \| \Err{d}_t \|.
\end{align}
This, along with all the conclusions in Section~\ref{sec:main} (Lemma~\ref{lem:p1.5}, Lemma~\ref{lem:p2}, Lemma~\ref{lem:p3}) and in the proof, hold for all $t \le \Tmax$. However, it is clear from the proof that these continue to hold for $t \le \tau$, where $\tau$ is the minimal number such that
\begin{align}
\label{eqn:stopping-defn}
\| \Overpar_{\tau + 1} \| > \alpha^{7/10} \|\Xtruth\|^{3/10},
\end{align}
cf. \eqref{eqn:p1.5-2-var}.
In other words, $\| \Overpar_{t} \| \le  \alpha^{7/10} \|\Xtruth\|^{3/10}$ for all $t \le \tau$. 
By Lemma~\ref{lem:p3} extended to the stopping time $\tau$, we have for $t_4 \le t \le \tau$ that
\begin{align}
\label{eqn:loss-ub-extended}
\| X_t X_t^\T - \Mtruth \|_{\fro} \le \alpha^{1/3} \|X_\star\|^{5/3}.
\end{align}
We recall that Lemma~\ref{lem:p3} was derived from Lemma~\ref{lem:p3-aux}. Following the same derivation, this time controlling the term
\[
\|\SigmaTruth^{-1}(\EssS_{t}\EssS_{t}^\T-\SigmaTruth^2)\SigmaTruth^{-1}\|_{\fro}
        + \|\Xtruth\|^{-1}\|\Misalign_{t}\EssS_{t}^{-1}\SigmaTruth\|_{\fro}
\]
directly using lemma~\ref{lem:p3-aux} instead of passing to $\| X_t X_t^\T - \Mtruth\|$, we find that for $t_4 \le t \le \tau$, the following stronger conclusion holds:
\begin{equation}
  \label{eqn:misalign-bound-extended}
  \|\Xtruth\|^{-1} \| \Misalign_t \EssS_t^{-1} \SigmaTruth \| \le \left(\frac{\alpha}{\|\Xtruth\|}\right)^{1/3}.
\end{equation}

Back to the recursive inequality \eqref{eqn:overpar-update}, 
We bound each terms, this time using \eqref{eqn:NSinv-tmp}, \eqref{eqn:Err-2-bound-raw} and a similar bound for $\Err{d}_t$, to obtain for all $t_4 \le t \le \tau$ that:
\begin{align*}
\| \Overpar_{t+1} \| 
&\le \|\Overpar_t\| + C\eta \kappa^C \|\Xtruth\|^{-1} (\|\Misalign_t \EssS_t \SigmaTruth\| + \|\Overpar_t\|) \|\Overpar_t\|
\\
& \le \|\Overpar_t\| + C\eta \kappa^C \left[ \left(\frac{\alpha}{\|\Xtruth\|}\right)^{1/3} + \left(\frac{\alpha}{\|\Xtruth\|}\right)^{7/10} \right] \|\Overpar_t\|
\\
& \le \left(1 + \eta \left( \frac{\alpha}{\|\Xtruth\|}\right)^{3/10} \right) \|\Overpar_t\|
\end{align*}
where $C>0$ is a universal constant; the second line follows from \eqref{eqn:misalign-bound-extended} and that $\|\Overpar_t\| \le \alpha^{7/10}\|\Xtruth\|^{3/10}$ for $t \le \tau$, and the last line follows from \eqref{eqn:alpha-cond}.

By induction on $t$, it is easy to see
\begin{align*}
\|\Overpar_{\tau + 1}\| 
&\le \left(1 + \eta \left( \frac{\alpha}{\|\Xtruth\|}\right)^{3/10} \right)^{\tau - \Tmax} \| \Overpar_{\Tmax} \|
\\
&\le \left(1 + \eta \left( \frac{\alpha}{\|\Xtruth\|}\right)^{3/10} \right)^{\tau - \Tmax} \alpha^{3/4} \|\Xtruth\|^{1/4},
\end{align*}
where the last inequality follows from \eqref{eqn:p1.5-2-var}. Plug this back into \eqref{eqn:stopping-defn}, we readily obtain
\begin{align*}
\tau - \Tmax 
\ge \frac{c\log\left(\frac{\|\Xtruth\|}{\alpha}\right)}{\log(1 + \eta\big(\frac{\alpha}{\|\Xtruth\|}\big)^{3/10})}
\ge \frac{2c\log\left(\frac{\|\Xtruth\|}{\alpha}\right)}{\eta (\alpha / \|\Xtruth\|)^{3/10}}
\ge \left(\frac{\|\Xtruth\|}{\alpha} \right)^{3/10},
\end{align*}
where $c = \frac{3}{4} - \frac{7}{10}>0$ is a universal constant, and the last two inequalities follow from \eqref{eqn:eta-cond} and \eqref{eqn:alpha-cond}. This completes the proof.


\section{Proofs for the noisy and the approximate low-rank settings}\label{sec:proof-noisy}
Both Theorem~\ref{thm:noisy} and Theorem~\ref{thm:approx} can be viewed as special cases of the following theorem.
\begin{theorem}
\label{lem:unified-noisy}
Assume the iterates $X_t$ of \myalg obeys
\begin{equation}\label{eqn:unified-noisy-update}
  X_{t+1} = X_t - \eta(\opAA(X_t X_t^\T - \Mtruth) - E) X_t (X_t^\T X_t + \lambda I)^{-1},
\end{equation}
for some matrix $E\in\reals^{n\times n}$, where $\Mtruth=\Xtruth\Xtruth^\T\in\reals^{n\times n}$ is a positive semidefinite matrix of rank $r_\star$, $\Xtruth\in\reals^{n\times r_\star}$. 
Assume further that
\begin{equation}
\label{eqn:aspt:E}
  \|E\|\le c_\sigma\kappa^{-C_\sigma} \|\Mtruth\|
\end{equation}
for some sufficiently small universal constant $c_\sigma>0$ and some sufficiently large universal constant $C_\sigma>0$. Then the following holds with high probability (with respect to the realization of the random initialization $G$). Under Assumptions~\ref{assumption:opA} and \ref{assumption:param}, there exist  universal constants $\Ccvg > 0$, $C_{\ref{lem:unified-noisy}}>0$,  such that for some $T\le\Tcvg \coloneqq \frac{\Ccvg}{\eta}\log\frac{\|\Xtruth\|}{\alpha}$, the iterates of \eqref{eqn:unified-noisy-update} obey
\begin{align*}
  \|X_T X_T^\T - \Mtruth\| &\le \max\left(\varepsilon\|\Mtruth\|,~ C_{\ref{lem:unified-noisy}} \kappa^4 \|\Utruth^\T E\|\right),
  \\
  \|X_T X_T^\T - \Mtruth\|_{\fro} &\le \max\left(\varepsilon\|\Mtruth\|,~C_{\ref{lem:unified-noisy}} \kappa^4 \|\Utruth^\T E\|_{\fro}\right).
\end{align*}
\end{theorem}
The proof is postponed to Appendix~\ref{sec:proof-unified-noisy}. The rest of this appendix is devoted to showing how to deduce Theorem~\ref{thm:noisy} and Theorem~\ref{thm:approx} from Theorem~\ref{lem:unified-noisy}.

\subsection{Proof of Theorem \ref{thm:noisy}}
In the noisy setting, the update rule~\eqref{eqn:noisy-update} of \myalg can be written as
\begin{equation}
\label{eqn:noisy-update-with-E}
  X_{t+1} = X_t - \eta \big(\opAA(X_t X_t^\T - \Mtruth) - E \big) X_t (X_t^\T X_t + \lambda I)^{-1},
\end{equation}
where
\begin{equation}\label{eqn:noisy-def-E}
  E \coloneqq \opA^*(\xi) = \sum_{i=1}^m \xi_i A_i.
\end{equation}

We use the following classical lemma to show that the matrix $E$ defined above fulfills the assumption of Theorem~\ref{lem:unified-noisy}.
\begin{lemma}
\label{lem:noise-bound}
  Under Assumption~\ref{assumption:opA}, the following holds with probability at least $1-2\exp(-cn)$.
  \begin{align*}
    \|E\| \le 8\sigma\sqrt{n},
    \quad
    \|\Utruth^\T E\|_{\fro} \le 8\sigma\sqrt{n r_\star}.
  \end{align*}
\end{lemma}
\begin{proof}
  The first inequality can be found in \cite{candes2010noisymc}, Lemma~1.1. The second inequality can be deduced from the first one as follows. Note that $\Utruth^\T E$ has rank at most $r_\star$, one has $\|\Utruth^\T E\|_{\fro} \le \sqrt{r_\star}\|\Utruth^\T E\|\le \sqrt{r_\star}\|E\|\le 8\sigma\sqrt{nr_\star}$, as desired.
\end{proof}

The conclusion of Theorem~\ref{thm:noisy} follows immediately by conditioning on the event that the inequalities in Lemma~\ref{lem:noise-bound} hold, and then invoking Theorem~\ref{lem:unified-noisy}.

\subsection{Proof of Theorem \ref{thm:approx}}
In the approximately low-rank setting, the update rule of \myalg can be written as
\begin{equation}
\label{eqn:approx-update-with-E}
  X_{t+1} = X_t - \eta \big(\opAA(X_t X_t^\T - M_{r_\star}) - E \big) X_t (X_t^\T X_t + \lambda I)^{-1},
\end{equation}
where
\begin{equation}\label{eqn:approx-def-E}
  E \coloneqq \opA^*\opA(M_{r_\star}').
\end{equation}

Recall that we assumed $\opA$ follows the Gaussian design in Theorem~\ref{thm:approx}. One may show that the matrix $E$ defined above fulfills the assumption of Theorem~\ref{lem:unified-noisy} using random matrix theory, detailed below.
\begin{lemma}
\label{lem:approx-err-bound}
  Under the assumptions on $\opA$ and $m$ in Theorem~\ref{thm:approx}, the following holds with probability at least $1-2\exp(-cn)$.
  \begin{align*}
    \|E\| \le 2\|M_{r_\star}'\| + 16\sqrt{\frac nm}\|M_{r_\star}'\|_{\fro},
    \quad
    \|\Utruth^\T E\|_{\fro} \le 16\|M_{r_\star}'\|_{\fro}.
  \end{align*}
\end{lemma}
\begin{proof}
For the first inequality, we use a standard covering argument. Let $\mathcal H$ be a $1/4$-net of $\mathbb S^{n-1}$, which can be chosen to satisfy $|\mathcal H|\le 9^n$. It is well known that
\begin{equation}
  \label{eqn:aux-covering}
  \|\opAA(M_{r_\star}')\| = \sup_{v\in\mathbb S^{n-1}} |\langle v, \opAA(M_{r_\star}') v\rangle| \le 2\sup_{v\in\mathcal H} |\langle v, \opAA(M_{r_\star}') v\rangle|.
\end{equation}
Note that $\langle v, \opAA(M_{r_\star}') v\rangle$ is an order-$2$ Gaussian chaos, which can be bounded by standard methods (see e.g. \cite{candes2010noisymc}), yielding
\[
  |\langle v, \opAA(M_{r_\star}') v\rangle - \langle v, M_{r_\star}' v\rangle|
  = |\langle v, \opAA(M_{r_\star}') v\rangle - \mathbb E\langle v, \opAA(M_{r_\star}') v\rangle|
  \le 8\sqrt{\frac nm}\|M_{r_\star}'\|_{\fro}
\]
with probability at least $1-2\exp(-4n)$. The desired inequality then follows from~\eqref{eqn:aux-covering} and a union bound.

For the second inequality, we first note that the random vector $\opA(M_{r_\star}')\in\reals^m$ is Gaussian with law $\mathcal N(0, \frac1m\|M_{r_\star}'\|_{\fro}^2 I)$. A standard Gaussian concentration inequality implies $\|\opA(M_{r_\star}')\|\le 2\|M_{r_\star}'\|_{\fro}$ with probability at least $1-2\exp(-m/2)$. To bound $\|\Utruth^\top\opAA(M_{r_\star}')\|_{\fro}$, the next step is to control the operator norm of $\Utruth^\top \opA^*$ as an operator on the following spaces:
\[
  \Utruth^\top \opA^*:\quad (\reals^m, \ell_2) \to \underbrace{(\reals^{r_\star\times n}, \|\cdot\|_{\fro})}_{\eqqcolon \mathcal M}.
\] 
In this sense, we may see that $\Utruth^\top \opA^*$ is a Gaussian operator, since the matrix form of this operator is a $(r_\star n)\times m$ matrix whose $i$-th column is the vectorization of $\Utruth^\top A_i$, which is i.i.d. Gaussian as $A_i$ is. 
Assume the covariance of such a column is $\Lambda^2\in\reals^{(r_\star n)\times(r_\star n)}$, then the matrix form of $\Utruth^\top \opA^*$ has the same distribution as $\Lambda G$, where $G$ is a $(r_\star n)\times m$ random matrix with i.i.d. standard Gaussian entries. Again, a standard bound in random matrix theory (c.f.~\eqref{eq:rmt_1}) implies that $\|G\|\le 4(\sqrt{m} + \sqrt{r_\star n})$ with probability at least $1-\exp(cm)$, given $m\ge Cnr_\star$ as assumed in Theorem~\ref{thm:approx}. Conditioning on this event, we have
\[
  \|\Utruth^\top\opA^*\|\le 4(\sqrt{m} + \sqrt{r_\star n})\|\Lambda\|. 
\]
To compute $\|\Lambda\|$, note that since $\Lambda G$ has the same distribution as the matrix form of $ \Utruth^\top \opA^* $, we have
\[
  \|\mathbb E(\Utruth^\top \opAA\Utruth)\|_{\mathcal M}
  = \|\mathbb E(\Lambda GG^\top \Lambda)\|
  = \|\Lambda(mI)\Lambda\| = m\|\Lambda\|^2,
\]
where the norm $\|\cdot\|_{\mathcal M}$ denotes the operator norm for operators on $\mathcal M$. 
But $\mathbb E(\opAA)=\mathcal I$, thus $\mathbb E(\Utruth^\top \opAA\Utruth)=\Utruth^\top\Utruth=I$ is the identity operator, hence $\|\mathbb E(\Utruth^\top \opAA\Utruth)\|_{\mathcal M}=1$. Plugging this into the above identity, we find $\|\Lambda\|=1/\sqrt{m}$. 
These together imply
\[
  \|\Utruth^\top \opA^*\| 
  \le 4\left(\sqrt{m} + \sqrt{r_\star n}\right) \cdot \frac{1}{\sqrt{m}}
  = 4\left(1+\sqrt{\frac{r_\star n}m}\right)
\] 
with probability at least $1-2\exp(-cm)$. The last quantity is less than $8$ by the assumption $m\ge Cnr_\star$ in Theorem~\ref{thm:approx}. Therefore
\[
  \|\Utruth^\top\opAA(M_{r_\star}')\|_{\fro}
  \le \|\Utruth^\top\opA^*\|\cdot\|\opA(M_{r_\star}')\|
  \le 8\cdot 2\|M_{r_\star}'\|_{\fro} = 16\|M_{r_\star}'\|_{\fro}
\]
with probability at least $1-\exp(-cm)$, as desired.
\end{proof}

The conclusion of Theorem~\ref{thm:approx} follows immediately by conditioning on the event that the inequalities in Lemma~\ref{lem:approx-err-bound} hold, and then invoking Theorem~\ref{lem:unified-noisy} with $M_\star$ substituted by $M_{r_\star}$.


\section{Proof of Theorem~\ref{lem:unified-noisy}}
\label{sec:proof-unified-noisy}

The proof is based on a reduction to the noiseless setting. 
We begin with two heuristic observations that connect the generalized setting with the noiseless one, and make these observations formal later.

\paragraph{Observation 1: Phase I approximates power method for $\opAA(\Mtruth) + E$.} As in the noiseless setting, in the first few iterations we expect $\|X_t\|$ to remain small, thus the update equation~\eqref{eqn:noisy-update} can be approximated by
\[
  X_{t+1} \approx (I + \eta(\opAA(\Mtruth) + E)) X_t.
\]
This coincides with the update equation of power method for $\opAA(\Mtruth) + E$. Recall that in the noiseless setting, the first phase is also akin to power method, albeit for $\opAA(\Mtruth)$. The key observation is that $\opAA(\Mtruth) + E$ enjoys all the same properties of $\opAA(\Mtruth)$ that were required to establish Lemma~\ref{lem:p1}. In fact, the only property of $\opAA(\Mtruth)$ used in the proof of Lemma~\ref{lem:p1} is
\[
    \|(\opAA - \mathcal I)\Mtruth\| \lesssim c_\delta \kappa^{-2C_\delta /3},
\]
but by the assumption~\eqref{eqn:aspt:E}, $\opAA(\Mtruth) + E$ also satisfies
\[
    \|\opAA(\Mtruth) + E - \Mtruth\| \le \|(\opAA - \mathcal I)\Mtruth\| + \|E\| \lesssim c_{\delta}\kappa^{-2C_\delta / 3}.
\]
Thus all conclusion of Lemma~\ref{lem:p1} remains valid in the generalized setting.

\paragraph{Observation 2: In Phase II and III, the update equation has the same form as that in the noiseless setting.}
Set
\[
  \Delta'_t = \Delta_t - E,
\]
then the update equation in the generalized setting can be expressed as
\[
  X_{t+1} = X_t - \eta(X_t X_t^\T - \Mtruth)X_t (X_t^\T X_t + \lambda I)^{-1} + \eta\Delta'_t X_t (X_t^\T X_t + \lambda I)^{-1},
\]
which has the same form with the noiseless update equation~\eqref{eqn:update-with-Delta}, if we replace $\Delta_t$ there by $\Delta'_t$. In the proof of Phase II, the only property of $\Delta_t$ we used is~\eqref{eqn:Delta-bound-coarse}, which still holds for $\Delta_t'$ since $\|E\|$ is small. Thus the proof can be simply carried over to the generalized setting of Theorem~\ref{lem:unified-noisy}. Moreover, in the proof of Phase III, the only places that involve controlling $\Delta_t$ in a different manner than~\eqref{eqn:Delta-bound-coarse} are ~\eqref{eqn:local-Gamma-bound} and~\eqref{eqn:local-misalign-update}. These equations require us to control $\uinorm{\Utruth^\T \Delta_t}$ for some unitarily invariant norm $\uinorm{\cdot}$. If we replace $\Delta_t$ by $\Delta'_t$, we can bound in these equations that
\[
  \uinorm{\Utruth^\T\Delta'_t}\le \uinorm{\Utruth^\T\Delta_t} + \uinorm{\Utruth^\T E}.
\]
Since any unitarily invariant $\uinorm{\cdot}$ is bounded by the operator norm up to a multiplicative constant\footnote{In this paper, $\uinorm{\cdot}$ is always taken to be either the operator norm or the Frobenius norm, for which this assertion is elementarily obvious.} (depending on the rank of the matrix), we may control $\uinorm{\Utruth^\T E}$ using the assumption~\eqref{eqn:aspt:E}.
Then we may combine~\eqref{eqn:local-Gamma-bound} and~\eqref{eqn:local-misalign-update} (assuming~\eqref{eqn:local-misalign-update} also holds with the Frobenius norm replaced by $\uinorm{\cdot}$) to obtain
\begin{align}
  &\uinorm{\SigmaTruth^{-1}(\EssS_{t+1}\EssS_{t+1}^\T-\SigmaTruth^2)\SigmaTruth^{-1}}
  + \|\Xtruth\|^{-1}\uinorm{\Misalign_{t+1}\EssS_{t+1}^{-1}\SigmaTruth}
  \nonumber\\
  & \quad \le \left(1-\frac{\eta}{10}\right)
  \left(\uinorm{\SigmaTruth^{-1}(\EssS_{t}\EssS_{t}^\T-\SigmaTruth^2)\SigmaTruth^{-1}}
      + \uinorm{\Xtruth\|^{-1}\|\Misalign_{t}\EssS_{t}^{-1}\SigmaTruth}
  \right)
  + \eta C\kappa^4\uinorm{\Utruth^\top E} + \eta\left(\frac{\|\Overpar_t\|}{\|\Xtruth\|}\right)^{1/2}.
  \label{eqn:p3-aux-noisy}
\end{align}
The conclusion of the theorem would immediately follow from the above inequality combined with Lemma~\ref{lem:noise-bound} and Lemma~\ref{lem:loss-decomp}, by taking $\uinorm{\cdot}$ to be the operator norm and the Frobenius norm.

Based on these observations, we formally state below the generalizations of key lemmas in the three phases required to prove Theorem~\ref{lem:unified-noisy}. 
Most of them have identical proofs to their noiseless counterparts, and in such cases the proofs will be omitted. The few of them that require a slightly modified proof will be discussed in full detail.


\subsection{Generalization of Phase I}
Our goal is to prove Lemma~\ref{lem:p1.5} in the generalized setting. 
\begin{lemma}\label{lem:noisy-p1.5}
The conclusions of Lemma~\ref{lem:p1.5}, along with its corollaries~\eqref{eqn:p1.5-1-var} and~\eqref{eqn:p1.5-2-var}, still hold in the setting of Theorem~\ref{lem:unified-noisy}.
\end{lemma}
As in the proof in the noiseless setting, this lemma is proved if we can prove the two parts of it respectively: the base case, where we show that there exists some $t_1\le \Tcvg/16$ such that \eqref{eqn:p1.5-0} holds 
and that~\eqref{subeq:condition-t-1} hold with $t=t_1$, and the induction step, where we show that \eqref{subeq:condition-t-1} continues to hold for $t\in[t_1, \Tmax]$.
\subsubsection{Establishing the base case}
We first show that Lemma~\ref{lem:p1} still holds in the generalized setting.
\begin{lemma}\label{lem:noisy-p1}
  Under the same setting as Theorem~\ref{lem:unified-noisy}, 
  we have for some $t_1\le \Tcvg/16$ such that \eqref{eqn:p1.5-0} holds 
  and that~\eqref{subeq:condition-t-1} hold with $t=t_1$.
\end{lemma}
We prove this result in a slightly more general setting. We consider a general symmetric matrix $\hat M\in\reals^{n\times n}$, and set
\[
  \hat{X}_t = \left(I + \frac{\eta}{\lambda}\hat M\right)^{t} X_0, \quad t=0,1,2,\cdots
\]
We also denote 
\begin{align*}
s_{j} \defeq \singval_{j}\Big(I + \frac{\eta}{\lambda}\hat M\Big) = 1+ \frac{\eta}{\lambda}\singval_{j}\big(\hat M\big),
\qquad j=1,2,\ldots,n
\end{align*}

The treatment of the noiseless setting in Appendix~\ref{sec:proof-p1} corresponds to the special case $\hat M = \opAA(\Mtruth)$. In the generalized setting, we choose $\hat M = \opAA(\Mtruth) + E$. The following two lemmas are generalized from the lemmas in Appendix~\ref{sec:proof-p1}, but have verbatim proofs as those, which are therefore omitted.
\begin{lemma}[Generalization of Lemma~\ref{lem:approx_power_method}]\label{lem:noisy_approx_power_method}
Suppose that $\lambda \geq \frac{1}{100} \kappa^{-4} c_\lambda\smin^2(\Xtruth)$. For any $\theta\in(0,1)$, there exists a large enough constant $K = K(\theta,c_\lambda, C_G) > 0$ such that the following holds. As long as $\alpha$ obeys
\begin{align}\label{ineq:noisy_alpha_cond_tmp}
\log\frac{\|\Xtruth\|}{\alpha} \geq \frac{K}{\max(\eta, \kappa^{-2})}\log(2\kappa n)\cdot\Big(1+\log\Big(1+\frac{\eta}{\lambda}\|\hat M\|\Big)\Big),
\end{align}
one has for all $t \le \frac{1}{\theta\eta}\log(\kappa n)$:
\begin{align}\label{ineq:noisy_approx_power_method}
\bigpnorm{X_t-\hat{X}_t}{} &\leq t\Big(1+\frac{\eta}{\lambda}\pnorm{\hat M}{}\Big)^t\frac{\alpha^2}{\|\Xtruth\|}.
\end{align}
Moreover, $\pnorm{X_t}{} \leq \pnorm{\Xtruth}{}$ for all such $t$.
\end{lemma}

\begin{lemma}[Generalization of Lemma~\ref{lem:Mahdi}]
\label{lem:noisy-Mahdi}
There exists some small universal constant $c_{\ref{lem:noisy-Mahdi}} > 0$ such that the following hold. 
Assume that for some $\gamma \leq c_{\ref{lem:noisy-Mahdi}}$,
\begin{align}\label{ineq:noisy_mahdi_p1_cond1}
\pnorm{\hat{M} - \Mtruth}{} \leq \gamma\smin^2(\Xtruth),
\end{align}
and furthermore, 
\begin{align}\label{eqn:noisy-Mahdi-phi}
\phi \defeq\frac{\alpha\|G\|s_{r_\star+1}^t+\|X_t - \hat{X}_t\|}{\alpha \smin(\Uspec^\T G)s_{r_\star}^t}\le c_{\ref{lem:noisy-Mahdi}}\kappa^{-2}.
\end{align}
Then for some universal $C_{\ref{lem:noisy-Mahdi}} > 0$ the following hold: \begin{subequations}
\begin{align}
\smin(\EssS_t)&\ge\frac{\alpha}{4}\smin(\Uspec^\T G)s_{r_\star}^t, \label{ineq:noisy_mahdi_p1_1}\\
\|\Overpar_t\|&\le C_{\ref{lem:noisy-Mahdi}}\phi\alpha\smin(\Uspec^\T G)s_{r_\star}^t, \label{ineq:noisy_mahdi_p1_2}\\ 
\|\UperpTruth^\T \SVDU_{\EssX_t}\|&\le C_{\ref{lem:noisy-Mahdi}}(\gamma+\phi), \label{ineq:noisy_mahdi_p1_3}
\end{align}
where $\Ess\myX_t \defeq \myX_t \SVDV_t \in \reals^{n\times r_\star}$.
\end{subequations}
\end{lemma}

We are now ready to prove Lemma~\ref{lem:noisy-p1}.

\begin{proof}[Proof of Lemma~\ref{lem:noisy-p1}]
Recall that the generalized setting corresponds to $\hat M = \opAA(\Mtruth) + E$. 
The proof is mostly identical to the proof of Lemma~\ref{lem:p1}. Similar to that proof, we first need to verify the two assumptions in Lemma~\ref{lem:noisy-Mahdi}. The rest of the proof goes exactly the same, thus is omitted here.

\paragraph{Verifying assumption~\eqref{ineq:noisy_mahdi_p1_cond1}.}
By the RIP in (\ref{eqn:rip}), Lemma \ref{lem:spectral_rip}, the condition of $\delta$ in (\ref{eqn:delta-cond}), and the assumption~\eqref{eqn:aspt:E}, we have 
\begin{align}\label{ineq:noisy_opAA_bound}
\|\hat M - \Mtruth\| = \bigpnorm{(\id-\opAA)(\Mtruth) + E}{} 
& \leq \sqrt{r_\star}\delta\|\Mtruth\| + c_\sigma\kappa^{-C_\sigma}\|\Mtruth\|
\nonumber \\
& \leq c_\delta\kappa^{-( C_\delta -2)}\smin^2(\Xtruth) +  c_\sigma \kappa^{-(C_\sigma-2)}\smin^2(\Xtruth)
\nonumber 
\\
& \eqqcolon\gamma \smin^2(\Xtruth).
\end{align}
Here $\gamma=c_\delta\kappa^{-( C_\delta -2)} + c_\sigma \kappa^{-(C_\sigma - 2)} \leq c_{\ref{lem:Mahdi}}$, as  
$ c_\delta $ and $c_\sigma$ are assumed to be sufficiently small.

\paragraph{Verifying assumption~\eqref{eqn:noisy-Mahdi-phi}.}
By Weyl's inequality and (\ref{ineq:noisy_opAA_bound}), we have 
\begin{align*}
  \Big|s_j-1-\frac{\eta}{\lambda}\singval_j(\Mtruth)\Big| \leq \frac{\eta}{\lambda}\bigpnorm{\hat M - \Mtruth}{} \leq \frac{\eta}{\lambda}\gamma\smin^2(\Xtruth)
  \leq \frac{100(c_\delta + c_\sigma)}{c_\lambda}\eta,
\end{align*}
where the last inequality follows from the condition $\lambda \geq\frac{1}{100}c_\lambda\smin^2(\Xtruth)$. 
Furthermore, using the condition $\lambda \leq c_\lambda\smin^2(\Xtruth)$ assumed in \eqref{eqn:lambda-cond}, 
the above bound implies that, for some $C = C(c_\lambda, c_\sigma, c_\delta) > 0$,
\begin{subequations} 
\begin{align}
s_1 &\le 1+\frac{\eta}{\lambda}\|\Mtruth\|+\frac{100(c_\delta + c_\sigma)}{c_\lambda}\eta \leq 1+C\eta\kappa^6, 
\label{eqn:noisy-s1-bound}\\
s_{r_\star} &\geq 1 + \frac{\eta}{\lambda}\smin^2(\Xtruth) 
- \frac{100(c_\delta + c_\sigma)}{c_\lambda}\eta \geq 1 + \frac{\eta}{2\lambda / \smin^2(X_\star)}, 
\label{eqn:noisy-sr-bound-lower}\\
s_{r_\star} &\leq 1 + \frac{\eta}{\lambda}\smin^2(\Xtruth) 
+ \frac{100(c_\delta + c_\sigma)}{c_\lambda}\eta \leq 1 + \frac{2\eta}{\lambda / \smin^2(X_\star)},
\label{eqn:noisy-sr-bound-upper}\\
s_{r_\star+1}&\le 1 + \frac{100(c_\delta + c_\sigma)}{c_\lambda}\eta \leq 1 + \frac{\eta}{4c_\lambda},
\label{eqn:noisy-sr1-bound}
\end{align}
\end{subequations}
where we use the fact that $\singval_{r_\star+1}(\Mtruth)=0$, and $c_\delta + c_\sigma \leq 1/400$. 
The rest of the verification is the same as the verification of~\eqref{eqn:Mahdi-phi} in the proof of Lemma~\ref{lem:p1}.
\end{proof}
  

\subsubsection{Establishing the induction step}
Following the proof of the noiseless setting, we would like to show that Lemmas~\ref{lem:overpar-update},~\ref{lem:misalign-update},~\ref{lem:bounded-S} still hold in the generalized setting, which in turn relies entirely on Lemmas~\ref{lem:update-approx},~\ref{lem:S-surrogate},~\ref{lem:NS-surrogate}. Since Lemma~\ref{lem:S-surrogate} and Lemma~\ref{lem:NS-surrogate} are both corollaries of Lemma~\ref{lem:update-approx}, it suffices to prove the generalization of Lemma~\ref{lem:update-approx} in the generalized setting. 

\begin{lemma}[Generalization of Lemma~\ref{lem:update-approx}]
\label{lem:noisy-update-approx}
Assume the update equation of $X_t$ has the following form (cf.~\eqref{eqn:update-with-Delta}):
\[
  X_{t+1} = X_t - \eta(X_t X_t^\T - \Mtruth) X_t (X_t^\T X_t + \lambda I)^{-1} + \eta\Delta'_t X_t(X_t^\T X_t + \lambda I)^{-1},
\] 
where $\Delta'_t\in\reals^{n\times n}$ is some symmetric matrix satisfying $\|\Delta'_t\|\le c_{\ref{lem:Delta-bound}}\kappa^{-2C_\delta/3}\|\Xtruth\|^2$. 
For any $t$ such that $\EssS_t$ is invertible and~\eqref{subeq:condition-t-1} holds, the equations~\eqref{eqn:S-update-approx} and~\eqref{eqn:N-update-approx} hold, where error terms are bounded by~\eqref{eqn:Err-2-bound}--\eqref{eqn:Err-4-bound} and the following modifications of \eqref{eqn:Err-1-bound} and \eqref{eqn:Err-5-bound}:
\begin{subequations}
\begin{align}
\uinorm{\Err{a}_t} &\le 
2c_{\ref{lem:p1.5}}\kappa^{-4}\|\Xtruth\| \cdot \uinorm{\Misalign_t\EssS_t^{-1}\SigmaTruth}
+ 2\uinorm{\Utruth^\T\Delta'_t}, 
\\
\uinorm{\Err{e}_t} &\le 
2\uinorm{\Utruth^\T\Delta'_t} 
+ c_{\ref{lem:Delta-bound}}\kappa^{-5}\|\Xtruth\| \cdot  \uinorm{\Misalign_t\EssS_t^{-1}\SigmaTruth}.
\end{align}
\end{subequations}
\end{lemma}
The proof is verbatim to Lemma~\ref{lem:update-approx}. 
Note that the noiseless setting corresponds to the special case $\Delta'_t = \Delta_t = (\mathcal I - \opAA)(\Mtruth)$, while the generalized setting corresponds to $\Delta'_t = \Delta_t - E = (\mathcal I - \opAA)(\Mtruth) - E$. To show that Lemma~\ref{lem:noisy-update-approx} is applicable to the generalized setting we need to verify that this choice of $\Delta'_t$ guarantees the smallness of $\|\Delta'_t\|$, which is proved in the following lemma.
\begin{lemma}[Generalization of~\eqref{eqn:Delta-bound-coarse} in Lemma~\ref{lem:Delta-bound}]
\label{lem:noisy-Delta-bound}
Under the same setting as Theorem~\ref{lem:unified-noisy}, for any $t$ such that \eqref{subeq:condition-t-1} holds, we have
\[
\|\Delta'_t\| \le c_{\ref{lem:Delta-bound}}\kappa^{-2C_\delta/ 3}\|\Xtruth\|^2.
\]
\end{lemma}
\begin{proof}
Combining \eqref{eqn:Delta-bound-coarse-primitive} in the proof of Lemma~\ref{lem:Delta-bound} the assumption $\|E\|\le c_\sigma\kappa^{-C_\sigma}\|\Mtruth\|$ in~\eqref{eqn:aspt:E}, we obtain
\begin{align*}
  \|\Delta'_t\| 
  & \le 16\delta\sqrt{r_\star}\kappa^2(C_{\ref{lem:p1.5}.a}^2+1)\|\Xtruth\|^2 + c_\sigma \kappa^{-C_\sigma}\|\Xtruth\|^2 
  \\
  & \le (16c_\delta\kappa^{-C_\delta +2 }(C_{\ref{lem:p1.5}.a}^2+1)^2 + c_\sigma\kappa^{-C_\sigma})\|\Xtruth\|^2
  \\
  & \le c_{\ref{lem:Delta-bound}}\kappa^{-2C_\delta/ 3}\|\Xtruth\|^2,
\end{align*}
if we choose $C_\sigma \ge C_\delta$, $c_\sigma\le c_\delta$, and note that $c_{\ref{lem:Delta-bound}}=32(C_{\ref{lem:p1.5}.a}+1)^2 c_\delta$ as defined in Lemma~\ref{lem:Delta-bound} (please refer to the argument after~\eqref{eqn:Delta-bound-coarse-primitive} for details).
\end{proof}

With these fundamental results in hand we can follow the same arguments as in the noiseless case to prove the following generalization of the lemmas in Appendix~\ref{sec:phase-1-induction}.
\begin{lemma}
\label{lem:noisy-misalign-update}
The conclusions of Lemmas~\ref{lem:overpar-update},~\ref{lem:bounded-S} still hold in the setting of Theorem~\ref{lem:unified-noisy}. Moreover, the following modification of Lemma~\ref{lem:misalign-update} holds in the setting of Theorem~\ref{lem:unified-noisy}. 
For any $t$ such that \eqref{subeq:condition-t-1} holds, 
setting $Z_t=\SigmaTruth^{-1}(\EssS_t\EssS_t^\T+\lambda I)\SigmaTruth^{-1}$, 
there exists some universal constant $C_{\ref{lem:misalign-update}}>0$ such that
\begin{equation}\label{eqn:noisy-misalign-update}
  \uinorm{\Misalign_{t+1}\EssS_{t+1}^{-1}\SigmaTruth}
  \le \left(1-\frac{\eta}{3(\|Z_t\|+\eta)}\right)\uinorm{\Misalign_t\EssS_t^{-1}\SigmaTruth}
    + \eta \frac{C_{\ref{lem:misalign-update}}\kappa^6}{c_\lambda\|\Xtruth\|}\uinorm{\Utruth^\T\Delta'_t}
    + \eta \left(\frac{\|\Overpar_t\|}{\smin(\EssS_t)}\right)^{1/2}\|\Xtruth\|.
\end{equation}
In particular, if $c_{\ref{lem:p1.5}}= 100C_{\ref{lem:misalign-update}}(C_{\ref{lem:p1.5}.a}+1)^4c_\delta/c_\lambda$, then 
$\|\Misalign_{t}\EssS_{t}^{-1}\SigmaTruth\|\le c_{\ref{lem:p1.5}}\kappa^{-C_\delta/2}\|\Xtruth\|$ implies
$\|\Misalign_{t+1}\EssS_{t+1}^{-1}\SigmaTruth\|\le c_{\ref{lem:p1.5}}\kappa^{-C_\delta/2}\|\Xtruth\|$.
\end{lemma}

By the arguments following Lemma~\ref{lem:bounded-S}, 
the above results are sufficient to prove the induction step, thereby completing the proof of Lemma~\ref{lem:p1.5} in the generalized setting.


\subsection{Generalization of Phase II}
We will prove Lemma~\ref{lem:p2} and Lemma~\ref{lem:p2.5}, the main results of Phase II, in the generalized setting. 
\begin{lemma}
\label{lem:noisy-p2}
The conclusions of Lemma~\ref{lem:p2} and Lemma~\ref{lem:p2.5}, along with Corollary~\ref{cor:p2} and Corollary~\ref{cor:p2.5}, still hold under the generalized setting of Theorem~\ref{lem:unified-noisy}.
\end{lemma}
Tracking the proof of Phase II in Appendix~\ref{sec:proof_phaseII}, one may verify that all proofs there hold verbatim in the generalized setting, with Lemma~\ref{lem:noisy-update-approx} in place of Lemma~\ref{lem:update-approx} (the proof also used Lemmas~\ref{lem:S-surrogate},~\ref{lem:NS-surrogate}, which are corollaries of Lemma~\ref{lem:update-approx}, hence hold in the generalized setting given Lemma~\ref{lem:noisy-update-approx}), except for Lemma~\ref{lem:local-unify}, which should be substituted by the following generalization:
\begin{lemma}\label{lem:noisy-local-unify}
  Under the same setting as Theorem~\ref{lem:unified-noisy}, for any $t : t_2\le t\le\Tmax$, one has
  \begin{equation}\label{eqn:noisy-local-Gamma-bound}
    \uinorm{\Gamma_{t+1}}
    \le (1-\eta)\uinorm{\Gamma_t}
    +\eta\frac{C_{\ref{lem:local-unify}}\kappa^4}{\|\Xtruth\|^{2}}\uinorm{\Utruth^\T\Delta'_t}
    +\frac{1}{16}\eta\|\Xtruth\|^{-1}\uinorm{\Misalign_t\EssS_t^{-1}\SigmaTruth}
    +\eta\left(\frac{\|\Overpar_t\|}{\|\Xtruth\|}\right)^{7/12},
  \end{equation}
  where $C_{\ref{lem:local-unify}}\lesssim c_\lambda^{-1/2}$ is some positive constant and $\uinorm{\cdot}$ can either be the Frobenius norm or the spectral norm.
\end{lemma}
The proof is identical to that of Lemma~\ref{lem:local-unify}, thus is omitted here. Following the proof in Appendix~\ref{sec:proof_phaseII}, these generalized results are sufficient to prove Lemma~\ref{lem:noisy-p2}, thereby completing the proof of Phase II in the generalized setting.

\subsection{Generalization of Phase III}
Our goal is to prove the following modification of Lemma~\ref{lem:p3} in the generalized setting.
\begin{lemma}[Generalization of Lemma~\ref{lem:p3}]
\label{lem:noisy-p3}
Under the same setting as Theorem~\ref{lem:unified-noisy}, there exists some universal constant $c_{\ref{lem:noisy-p3}}>0$ such that
for any $t: t_3\le t\le\Tmax$, with $\uinorm{\cdot}$ taken to be the operator norm $\|\cdot\|$ or the Frobenius norm $\|\cdot\|_{\fro}$, we have
\begin{equation}
\label{eqn:lem:noisy-p3-primal}
  \uinorm{\myX_{t} \myX_{t}^\T-\Mtruth}
  \le (1-c_{\ref{lem:noisy-p3}}\eta)^{t-t_3}r_\star\|\Mtruth\|
  + c_{\ref{lem:noisy-p3}}^{-1}\kappa^4\uinorm{\Utruth^\T E}
  +8c_{\ref{lem:noisy-p3}}^{-1}\|\Mtruth\|
  \max_{t_3\le\tau\le t}\left(\frac{\|\Overpar_{\tau}\|}{\|\Xtruth\|}\right)^{1/2}.
\end{equation}
In particular, there exists an iteration number 
$t_4: t_3 \leq t_4 \leq t_3 + \Tcvg / 16$ such that 
for any $t\in[t_4, \Tmax]$, we have
\begin{equation}\label{eqn:lem:noisy-p3-final}
  \uinorm{\myX_t\myX_t^\T - \Mtruth}
  \le \max(\alpha^{1/3}\|\Xtruth\|^{5/3}, c_{\ref{lem:noisy-p3}}^{-1}\kappa^4\uinorm{\Utruth^\T E})
  \le \max(\epsilon\|\Mtruth\|, c_{\ref{lem:noisy-p3}}^{-1}\kappa^4\uinorm{\Utruth^\T E}).
\end{equation}
Here, $\epsilon$ and $\alpha$ are as stated in Theorem \ref{thm:main}.
\end{lemma}

As in Appendix~\ref{sec:pf:lem:p3}, this will be accomplished by decomposing the error $\uinorm{X_t X_t^\T - \Mtruth}$ using Lemma~\ref{lem:loss-decomp}, and then control the components in the decomposition using Lemma~\ref{lem:p3-aux}. It is easy to check that the proof of Lemma~\ref{lem:loss-decomp} applies without modification to the generalized setting, and in fact works with the Frobenius norm replaced by any unitarily invariant norm. This leads to the following generalization.
\begin{lemma}[Generalization of Lemma~\ref{lem:loss-decomp}]
\label{lem:noisy-loss-decomp}
Under the same setting as Theorem~\ref{lem:unified-noisy}, for all $t\ge t_3$, as long as 
$\|\SigmaTruth^{-1}(\EssS_{t}\EssS_{t}^\T-\SigmaTruth^2)\SigmaTruth^{-1}\|\le 1/10$,
one has
\begin{equation*}
  \uinorm{\myX_t\myX_t^\T - \Mtruth}
  \le 4\|\Xtruth\|^2
  \left(
    \uinorm{\SigmaTruth^{-1}(\EssS_{t}\EssS_{t}^\T-\SigmaTruth^2)\SigmaTruth^{-1}}
    + \|\Xtruth\|^{-1} \uinorm{\Misalign_{t}\EssS_{t}^{-1}\SigmaTruth}
  \right)
  + 4\|\Xtruth\|\|\Overpar_t\|.
\end{equation*}
\end{lemma}
It remains to prove the generalization of Lemma~\ref{lem:p3-aux}, stated below.
\begin{lemma}[Generalization of Lemma~\ref{lem:p3-aux}]
\label{lem:noisy-p3-aux}
Under the same setting as Theorem~\ref{lem:unified-noisy}, there exists some universal constant $C_{\ref{lem:noisy-p3-aux}}>0$ such that for any $t: t_3\le t\le\Tmax$, 
with $\uinorm{\cdot}$ taken to be the operator norm $\|\cdot\|$ or the Frobenius norm $\|\cdot\|_{\fro}$, 
one has
\begin{align}
  &\uinorm{\SigmaTruth^{-1}(\EssS_{t+1}\EssS_{t+1}^\T-\SigmaTruth^2)\SigmaTruth^{-1}}
  + \|\Xtruth\|^{-1}\uinorm{\Misalign_{t+1}\EssS_{t+1}^{-1}\SigmaTruth}
  \nonumber\\
  & \quad \le \left(1-\frac{\eta}{10}\right)
  \left(\uinorm{\SigmaTruth^{-1}(\EssS_{t}\EssS_{t}^\T-\SigmaTruth^2)\SigmaTruth^{-1}}
      + \|\Xtruth\|^{-1}\uinorm{\Misalign_{t}\EssS_{t}^{-1}\SigmaTruth}
  \right)
  + \eta \frac{C_{\ref{lem:noisy-p3-aux}}\kappa^4}{\|\Xtruth\|^2}\uinorm{\Utruth^\T E}
  + \eta\left(\frac{\|\Overpar_t\|}{\|\Xtruth\|}\right)^{1/2}.
  \label{eqn:noisy-p3-aux}
\end{align}
In particular, $\|\SigmaTruth^{-1}(\EssS_{t+1}\EssS_{t+1}^\T-\SigmaTruth^2)\SigmaTruth^{-1}\|\le 1/10$
for all $t$ such that $t_3\le t\le \Tmax$.
\end{lemma}

We are prepared to formally prove Lemma~\ref{lem:noisy-p3}. Similar to the noiseless setting, we apply Lemma~\ref{lem:noisy-p3-aux} repeatedly to obtain the following bound 
for all $t\in[t_3,\Tmax]$:
\begin{align}
  &\uinorm{\SigmaTruth^{-1}(\EssS_{t}\EssS_{t}^\T-\SigmaTruth^2)\SigmaTruth^{-1}}
  + \|\Xtruth\|^{-1} \uinorm{\Misalign_{t}\EssS_{t}^{-1}\SigmaTruth} \nonumber
  \\
  &\quad \le \left(1-\frac{\eta}{10}\right)^{t-t_3}\left(
    \uinorm{\SigmaTruth^{-1}(\EssS_{t_3}\EssS_{t_3}^\T-\SigmaTruth^2)\SigmaTruth^{-1}}
  + \|\Xtruth\|^{-1} \uinorm{\Misalign_{t_3}\EssS_{t_3}^{-1}\SigmaTruth}
  \right) 
  \nonumber\\
  &\quad\phantom{\le}
  + \frac{10 C_{\ref{lem:noisy-p3-aux}}\kappa^4}{\|\Xtruth\|^2} \uinorm{\Utruth^\T E}
  + 10\max_{t_3\le\tau\le t}\left(\frac{\|\Overpar_\tau\|}{\|\Xtruth\|}\right)^{1/2}, \label{eq:noisy-iterative-decay}
\end{align}
which motivates us to control the error at time $t_3$. With the same arguments as in the noiseless setting (cf. Equation~\eqref{eq:t-3-condition} in Appendix~\ref{sec:pf:lem:p3}), we obtain
\begin{equation*}
  \|\SigmaTruth^{-1}(\EssS_{t_3}\EssS_{t_3}^\T-\SigmaTruth^2)\SigmaTruth^{-1}\|_{\fro}
  + \|\Xtruth\|^{-1}\|\Misalign_{t_3}\EssS_{t_3}^{-1}\SigmaTruth\|_{\fro}
  \le \frac{\sqrt{r_\star}}{5}.
\end{equation*}
Since the operator norm of a matrix is always less than or equal to the Frobenius norm of it, the above inequality also holds if the Frobenius norm is replaced by the operator norm. Recalling that in this lemma, $\uinorm{\cdot}$ is taken to be either the operator norm or the Frobenius norm, we have shown
\begin{equation}
  \uinorm{\SigmaTruth^{-1}(\EssS_{t_3}\EssS_{t_3}^\T-\SigmaTruth^2)\SigmaTruth^{-1}}
  + \|\Xtruth\|^{-1} \uinorm{\Misalign_{t_3}\EssS_{t_3}^{-1}\SigmaTruth}
  \le \frac{\sqrt{r_\star}}{5}. \label{eq:noisy-t-3-condition}
\end{equation}
Combining the two inequalities~\eqref{eq:noisy-iterative-decay} and~\eqref{eq:noisy-t-3-condition} yields 
for all $t \in [t_3, \Tmax]$
\begin{align*}
  &\uinorm{\SigmaTruth^{-1}(\EssS_{t}\EssS_{t}^\T-\SigmaTruth^2)\SigmaTruth^{-1}}
  + \|\Xtruth\|^{-1} \uinorm{\Misalign_{t}\EssS_{t}^{-1}\SigmaTruth}\\
  &\quad
  \le \frac15\left(1-\frac{\eta}{10}\right)^{t-t_3}\sqrt{r_\star}
  + \frac{10 C_{\ref{lem:noisy-p3-aux}}\kappa^4}{\|\Xtruth\|^2} \uinorm{\Utruth^\T E}
  + 10\max_{t_3\le\tau\le t}\left(\frac{\|\Overpar_\tau\|}{\|\Xtruth\|}\right)^{1/2}.
\end{align*}
We can then invoke Lemma~\ref{lem:noisy-loss-decomp} to see that 
\begin{align*}
  &\uinorm{\myX_t\myX_t^\T-\Mtruth}
  \\
  &\quad \le \frac{4\|\Xtruth\|^2}{5}\left(1-\frac{\eta}{10}\right)^{t-t_3}\sqrt{r_\star}
  + 10 C_{\ref{lem:noisy-p3-aux}}\kappa^4 \uinorm{\Utruth^\T E}
  + 40\|\Xtruth\|^2\max_{t_3\le\tau\le t}\left(\frac{\|\Overpar_\tau\|}{\|\Xtruth\|}\right)^{1/2}
  + 4\|\Xtruth\|\|\Overpar_t\|
  \\
  &\quad \le \left(1-\frac{\eta}{10}\right)^{t-t_3}\sqrt{r_\star}\|\Mtruth\|
  + 10 C_{\ref{lem:noisy-p3-aux}}\kappa^4 \uinorm{\Utruth^\T E}
  + 80\|\Mtruth\|\max_{t_3\le\tau\le t}\left(\frac{\|\Overpar_{\tau}\|}{\|\Xtruth\|}\right)^{1/2},
\end{align*}
where in the last line we use $\|\Overpar_t\|\le\|\Xtruth\|$---an implication of~\eqref{eqn:p1.5-2-var}, which holds in the generalized setting by Lemma~\ref{lem:noisy-p1.5}. 
To see this, the assumption \eqref{eqn:alpha-cond} implies that 
$\alpha\le\|\Xtruth\|$ as long as $\eta\le 1/2$ and $C_\alpha\ge 4$, 
which in turn implies $\|\Overpar_t\|\le\alpha^{2/3}\|\Xtruth\|^{1/3}\le\|\Xtruth\|$.
This completes the proof for the first part of Lemma~\ref{lem:noisy-p3} with $c_{\ref{lem:noisy-p3}}=1/(10 C_{\ref{lem:noisy-p3-aux}})$.

For the second part of Lemma~\ref{lem:noisy-p3}, 
notice that 
\[
8c_{\ref{lem:noisy-p3}}^{-1}\max_{t_3\le\tau\le\Tmax}(\|\Overpar_\tau\|/\|\Xtruth\|)^{1/2}
\le \frac12\left(\frac{\alpha}{\|\Xtruth\|}\right)^{1/3}
\]
by~\eqref{eqn:p1.5-2-var}, thus
\begin{equation*}
  \uinorm{X_t X_t^\T - \Mtruth}
  \le (1-c_{\ref{lem:p3}}\eta)^{t-t_3} \sqrt{r_\star} \|\Mtruth\| 
  + c_{\ref{lem:noisy-p3}}^{-1}\kappa^4 \uinorm{\Utruth^\T E}
  + \frac12\left(\frac{\alpha}{\|\Xtruth\|}\right)^{1/3}
\end{equation*}
for $t_3\le t\le\Tmax$. 
There exists some iteration number 
$t_4: t_3\le t_4\le t_3 + \frac{2}{c_{\ref{lem:noisy-p3}}\eta} \log(\|\Xtruth\|/\alpha) \le t_3 + \Tcvg/16$ 
such that
\[
(1 - c_{\ref{lem:p3}}\eta) ^ {t_4-t_3}
\le \left(\frac{\alpha}{\|\Xtruth\|}\right)^2
\le \frac{1}{2\sqrt{r_\star}} \left(\frac{\alpha}{\|\Xtruth\|}\right)^{1/3},
\]
where the last inequality is due to~\eqref{eqn:alpha-cond}.
It is then clear that $t_4$ has the property claimed in the lemma.

\subsubsection{Proof of Lemma \ref{lem:noisy-p3-aux}}
The idea is the same as the proof of Lemma~\ref{lem:p3-aux}.
Fix any $t\in[t_3,\Tmax]$, if~\eqref{eqn:noisy-p3-aux} were true for all $\tau\in[t_3,t]$, 
taking into account that $\|\Overpar_\tau\|/\|\Xtruth\|\le 1/10000$ 
for all $\tau\in[t_3,\Tmax]$ by~\eqref{eqn:p1.5-2-var} (which still holds in the generalized setting according to Lemma~\ref{lem:noisy-p1}), 
we could show by induction that $\|\Gamma_\tau\|\le 1/10$ for all $\tau\in[t_3,t]$. 
Thus it suffices to assume $\|\Gamma_t\|\le 1/10$ and prove \eqref{eqn:noisy-p3-aux}.

Apply Lemma~\ref{lem:noisy-local-unify} to obtain
\begin{equation}\label{eq:noisy-Gamma-local}
  \uinorm{ {\Gamma_{t+1}} }
  \le (1-\eta) \uinorm{ {\Gamma_t} }
  +\eta\frac{C_{\ref{lem:local-unify}}\kappa^4}{\|\Xtruth\|^{2}} \uinorm{ {\Utruth^\T\Delta'_t} }
  +\frac{1}{16}\eta\|\Xtruth\|^{-1} \uinorm{ {\Misalign_t\EssS_t^{-1}\SigmaTruth} }
  +\eta\left(\frac{\|\Overpar_t\|}{\|\Xtruth\|}\right)^{7/12},
\end{equation}
In addition, Lemma~\ref{lem:noisy-misalign-update} tells us that 
\begin{equation*}
  \uinorm{ \Misalign_{t+1}\EssS_{t+1}^{-1}\SigmaTruth }
  \le \left(1-\frac{\eta}{3(\|Z_t\|+\eta)}\right) \uinorm{\Misalign_t\EssS_t^{-1}\SigmaTruth}
    + \eta \frac{C_{\ref{lem:misalign-update}}\kappa^4}{c_\lambda\|\Xtruth\|}\uinorm{\Utruth^\T\Delta'_t}
    + \eta \left(\frac{\|\Overpar_t\|}{\smin(\EssS_t)}\right)^{2/3}\|\Xtruth\|,
\end{equation*}
where $Z_t = \SigmaTruth^{-1}(\EssS_t\EssS_t^\T+\lambda I)\SigmaTruth^{-1}$. 
It is easy to check that $\|Z_t\| \leq 1+\|\Gamma_t\| + c_\lambda\le 2$ as $\| \Gamma_t \| \leq 1/10$ and $c_\lambda$ is sufficiently small.
In addition, one has $\smin(\EssS_t)^2\ge(1-\|\Gamma_t\|)\smin(\Xtruth)^2$ 
and $\|\Overpar_t\|/\smin(\EssS_t)\le (2\kappa)^{-24}$. 
Combine these relationships together to arrive at
\begin{equation}\label{eqn:noisy-local-misalign-update}
  \uinorm{\Misalign_{t+1}\EssS_{t+1}^{-1}\SigmaTruth}
  \le \left(1-\frac{\eta}{8}\right) \uinorm{\Misalign_t\EssS_t^{-1}\SigmaTruth}
  + \eta \frac{C_{\ref{lem:misalign-update}}\kappa^2}{c_\lambda\|\Xtruth\|}
    \uinorm{ \Utruth^\T\Delta'_t }
  + \frac12\eta \|\Xtruth\|\left(\frac{\|\Overpar_t\|}{\|\Xtruth\|}\right)^{7/12}.
\end{equation}

Summing up~\eqref{eq:Gamma-local},~\eqref{eqn:local-misalign-update}, 
we obtain
\begin{align}
&\uinorm{ \Gamma_{t+1} }
+ \|\Xtruth\|^{-1} \uinorm{ \Misalign_{t+1}\EssS_{t+1}^{-1}\SigmaTruth }
\nonumber\\
&\le \left(1 - \frac{\eta}{8}\right) 
(\uinorm{ \Gamma_{t} } 
+ \|\Xtruth\|^{-1} \uinorm{ \Misalign_{t}\EssS_{t}^{-1}\SigmaTruth })
+ \eta\frac{2(C_{\ref{lem:misalign-update}}+C_{\ref{lem:local-unify}}c_\lambda)\kappa^6}{c_\lambda \|\Xtruth\|^2} 
\uinorm{ \Utruth^\T\Delta'_t }
+ 2\eta\left(\frac{\|\Overpar_t\|}{\|\Xtruth\|}\right)^{7/12}.
\nonumber\\
&\le \left(1 - \frac{\eta}{8}\right) 
(\uinorm{ \Gamma_{t} } 
+ \|\Xtruth\|^{-1} \uinorm{ \Misalign_{t}\EssS_{t}^{-1}\SigmaTruth })
+ \eta\frac{2(C_{\ref{lem:misalign-update}}+C_{\ref{lem:local-unify}}c_\lambda)\kappa^8}{c_\lambda \|\Xtruth\|^2} 
(\uinorm{ \Utruth^\T\Delta_t } + \uinorm{ \Utruth^\T E })
+ 2\eta\left(\frac{\|\Overpar_t\|}{\|\Xtruth\|}\right)^{7/12}.
\label{eqn:noisy-loc-with-Delta}
\end{align}
This is close to our desired conclusion, 
but we would need to eliminate $\uinorm{\Utruth^\T\Delta_t}$. 
To this end we shall need the following lemma. 
\begin{lemma}
\label{lem:noisy-ui-Delta}
If $\uinorm{\cdot}$ is taken to be the operator norm $\|\cdot\|$ or the Frobenius norm $\|\cdot\|_{\fro}$, under the same setting as Lemma~\ref{lem:noisy-p3-aux}, one has
\begin{equation}
\label{eqn:noisy-ui-Delta}
\uinorm{\Utruth^\T \Delta_t} \le 32c_\delta\kappa^{-6}\|\Xtruth\|^2\left(\uinorm{ \Gamma_t }
+\|\Xtruth\|^{-1} \uinorm{ \Misalign_t\EssS_t^{-1}\SigmaTruth }
+\left(\frac{\|\Overpar_t\|}{\|\Xtruth\|}\right)^{2/3}\right).
\end{equation}
\end{lemma}

Return to the proof of Lemma~\ref{lem:noisy-p3-aux}. The conclusion follows from applying 
the above lemma to the term $\uinorm{\Utruth^\T \Delta_t}$ in ~\eqref{eqn:noisy-loc-with-Delta}, 
noting that $c_\delta$ can be chosen sufficiently small such that
\[
  \frac{2(C_{\ref{lem:misalign-update}}+C_{\ref{lem:local-unify}}c_\lambda)}{c_\lambda}\cdot 32c_\delta < \frac{1}{16},
\] 
and that $\|\Overpar_t\|/\|\Xtruth\|$ 
is sufficiently small due to~\eqref{eqn:p1.5-2-var}, which still holds in the generalized setting in virtue of Lemma~\ref{lem:noisy-p1}.

\subsubsection{Proof of Lemma \ref{lem:noisy-ui-Delta}}
Observe that $\Utruth^\T\Delta_t$ has rank at most $r_\star$, thus
\[
  \| \Utruth^\T\Delta_t \| \le \|\Delta_t\|, \quad \|\Utruth^\T \Delta_t\|_{\fro} \le \sqrt{r_\star}\|\Delta_t\|.
\]
On the other hand, from Lemma~\ref{lem:Delta-bound}, we know 
\begin{align*}
  \|{ \Delta_t }\|
  &\le 8\delta
  \left(\|{ \EssS_t\EssS_t^\T-\SigmaTruth^2 }\|_{\fro}
    +\|\EssS_t\| \|{ \Misalign_t }\|_{\fro}
  +n\|\Overpar_t\|^2\right)
  \\
  &\le 16c_\delta r_\star^{-1/2}\kappa^{-4}\|\Xtruth\|^2\left(\|{ \Gamma_t }\|_{\fro}
    +\|\Xtruth\|^{-1} \|{ \Misalign_t\EssS_t^{-1}\SigmaTruth }\|_{\fro}
    +\left(\frac{\|\Overpar_t\|}{\|\Xtruth\|}\right)^{2/3}
  \right)
  \\
  &\le 32c_\delta\kappa^{-4}\|\Xtruth\|^2\left(\|{ \Gamma_t }\|
  +\|\Xtruth\|^{-1} \|{ \Misalign_t\EssS_t^{-1}\SigmaTruth }\|
  +\left(\frac{\|\Overpar_t\|}{\|\Xtruth\|}\right)^{2/3}
\right)
\end{align*}
where the penultimate line follows from~\eqref{eqn:delta-cond} 
and from controlling the sum inside the brackets 
in a similar way as~\eqref{eqn:pf:loss-decomp}, 
and the last line follows from $\Gamma_t=\SigmaTruth^{-1}\EssS_t\EssS_t^\T\SigmaTruth^{-1} - I$ being a matrix of rank at most $r_\star + 1$, which implies $\|\Gamma_t\|_{\fro}\le \sqrt{r_\star + 1}\|\Gamma_t\|$, and similarly $\|{ \Misalign_t\EssS_t^{-1}\SigmaTruth }\|_{\fro} \le \sqrt{r_\star} \|{ \Misalign_t\EssS_t^{-1}\SigmaTruth }\|$.
The conclusion then follows from bounding $\|\Utruth^\T\Delta_t\|$ and $\|\Utruth^\T\Delta_t\|_{\fro}$ separately. We have
\begin{align*}
\|\Utruth^\T\Delta_t\|
\le \|\Delta_t\| \le 32c_\delta\kappa^{-4}\|\Xtruth\|^2\left(\|{ \Gamma_t }\|
+\|\Xtruth\|^{-1} \|{ \Misalign_t\EssS_t^{-1}\SigmaTruth }\|
+\left(\frac{\|\Overpar_t\|}{\|\Xtruth\|}\right)^{2/3}\right),
\end{align*}
and 
\begin{align*}
\|\Utruth^\T\Delta_t\|_{\fro}
& \le \sqrt{r_\star}\|\Delta_t\|
\\
& \le \sqrt{r_\star}\cdot 16c_\delta r_\star^{-1/2}\kappa^{-4}\|\Xtruth\|^2\left(\|{ \Gamma_t }\|_{\fro}
+\|\Xtruth\|^{-1} \|{ \Misalign_t\EssS_t^{-1}\SigmaTruth }\|_{\fro}
+\left(\frac{\|\Overpar_t\|}{\|\Xtruth\|}\right)^{2/3}
\right)
\\
& = 16c_\delta\kappa^{-4}\|\Xtruth\|^2\left(\|{ \Gamma_t }\|_{\fro}
+\|\Xtruth\|^{-1} \|{ \Misalign_t\EssS_t^{-1}\SigmaTruth }\|_{\fro}
+\left(\frac{\|\Overpar_t\|}{\|\Xtruth\|}\right)^{2/3}
\right).
\end{align*}
Combining the above two inequalities together proves that~\eqref{eqn:noisy-ui-Delta} holds with $\uinorm{\cdot}$ taken to be the operator norm or the Frobenius norm.

\subsection{Proof of Theorem~\ref{lem:unified-noisy}}
Combining Lemma~\ref{lem:noisy-p1}, Lemma~\ref{lem:noisy-p2} and Lemma~\ref{lem:noisy-p3}, the final $t_4$ given by Lemma~\ref{lem:noisy-p3}
is no more than $4\times\Tcvg/16\le\Tcvg/2$, 
thus~\eqref{eqn:lem:noisy-p3-final} holds for all $t\in[\Tcvg/2,\Tmax]$, 
in particular, for some $T\le \Tcvg$. 
Plugging in~\eqref{eqn:lem:noisy-p3-final} the bound for $\uinorm{\Utruth^\T E}$ given by Lemma~\ref{lem:noise-bound} when $\uinorm{\cdot}$ is taken to be the operator norm $\|\cdot\|$ or the Frobenius norm $\|\cdot\|_{\fro}$, we obtain the conclusion as desired.

\end{document}